\documentclass{article}
\usepackage[left=1 in,right=1 in,top=1 in,bottom=1 in]{geometry}
\usepackage[pdftex]{graphicx}

\usepackage{amsfonts}
\usepackage{latexsym}
\usepackage{tabularx}
\usepackage{fancyhdr}
\usepackage{verbatim}
\usepackage{multirow}
\usepackage{framed}
\usepackage{algorithmic}
\usepackage{algorithm}
\usepackage{amsmath,amsthm,amssymb}
\usepackage[pdftex,colorlinks,citecolor=blue,linkcolor=red]{hyperref}
\usepackage{titling} % for \setlength
\usepackage{caption}
\usepackage{graphicx, subfigure}
\usepackage{enumitem}
\usepackage[table]{xcolor}
\usepackage{url}

\usepackage[]{authblk}

\definecolor{lightgray}{RGB}{215,215,215}

\allowdisplaybreaks[4] % allowing align equations break over two pages. 
\hypersetup{pdfauthor={Y.W. Park and D. Klabjan}} \hypersetup{pdftitle= Optimization via Clustering in Machine Learning}

\newcommand{\vertiii}[1]{{\left\vert\kern-0.25ex\left\vert\kern-0.25ex\left\vert #1 
    \right\vert\kern-0.25ex\right\vert\kern-0.25ex\right\vert}}

\theoremstyle{definition}

\newtheorem{lemma}{Lemma}

\newtheorem{proposition}{Proposition}
\newtheorem{corollary}{Corollary}
\theoremstyle{definition} 
\theoremstyle{definition}

\makeatletter
\newcommand{\rmnum}[1]{\romannumeral #1}
\newcommand{\Rmnum}[1]{\expandafter\@slowromancap\romannumeral #1@}
\makeatother

\title{An Aggregate and Iterative Disaggregate Algorithm with Proven Optimality in Machine Learning} 

\author[1]{Young Woong Park \thanks{Major part of the work done at Northwestern University} \thanks{ywpark@smu.edu}}
\author[2]{Diego Klabjan \thanks{d-klabjan@northwestern.edu}}
\affil[1]{Cox School of Business, Southern Methodist University, Dallas, TX, USA}
\affil[2]{Department of Industrial Engineering and Management Sciences, Northwestern University, Evanston, IL, USA}

\date{Feb 25, 2016}

\begin{document}

\maketitle

\begin{abstract}
We propose a clustering-based iterative algorithm to solve certain optimization problems in machine learning, where we start the algorithm by aggregating the original data, solving the problem on aggregated data, and then in subsequent steps gradually disaggregate the aggregated data. We apply the algorithm to common machine learning problems such as the least absolute deviation regression problem, support vector machines, and semi-supervised support vector machines. We derive model-specific data aggregation and disaggregation procedures. We also show optimality, convergence, and the optimality gap of the approximated solution in each iteration. A computational study is provided.
\end{abstract}

\section{Introduction}

In this paper, we propose a clustering-based iterative algorithm to solve certain optimization problems in machine learning when data size is large and thus it becomes impractical to use out-of-the-box algorithms. We rely on the principle of data aggregation and then subsequent disaggregations. While it is standard practice to aggregate the data and then calibrate the machine learning algorithm on aggregated data, we embed this into an iterative framework where initial aggregations are gradually disaggregated to the extent that even an optimal solution is obtainable.

Early studies in data aggregation consider transportation problems \cite{Balas:65,Evans:83}, where either demand or supply nodes are aggregated. Zipkin \cite{ZipkinPHD77} studied data aggregation for linear programming (LP) and derived error bounds of the approximate solution. There are also studies on data aggregation for 0-1 integer programming \cite{Chvatal-Hammer:77, Hallefjord-Storoy:90}. The reader is referred to Rogers \textit{et al} \cite{Rogers-etal:91} and Litvinchev and Tsurkov \cite{Litvinchev-Tsurkov-book} for comprehensive literature reviews for aggregation techniques applied for optimization problems. 

For support vector machines (SVM), there exist several works using the concept of clustering or data aggregation. Evgeniou and Pontil \cite{Evgeniou-Pontil:02} proposed a clustering algorithm that creates large size clusters for entries surrounded by the same class and small size clusters for entries in the mixed-class area. The clustering algorithm is used to preprocess the data and the clustered data is used to solve the problem. The algorithm tends to create large size clusters for entries far from the decision boundary and small size clusters for the other case. Wang \textit{et al} \cite{Wang-etal:2014} developed screening rules for SVM to discard non-support vectors that do not affect the classifier. Nath \textit{et al} \cite{Nath2006} and Doppa \textit{et al} \cite{Doppa2010} proposed a second order cone programming (SOCP) formulation for SVM based on chance constraints and clusters. The key idea of the SOCP formulations is to reduce the number of constraints (from the number of the entries to number of clusters) by defining chance constraints for clusters.

After obtaining an approximate solution by solving the optimization problem with aggregated data, a natural attempt is to use less-coarsely aggregated data, in order to obtain a finer approximation. In fact, we can do this iteratively: modify the aggregated data in each iteration based on the information at hand. This framework, which iteratively passes information between the original problem and the aggregated problem \cite{Rogers-etal:91}, is known as \textit{Iterative Aggregation Disaggregation} (IAD). The IAD framework has been applied for several optimization problems such as LP \cite{Mendelssohn:80, Shetty-Taylor:87, Vakhutinsky-etal:87} and network design \cite{Barmann:13}. In machine learning, Yu \textit{et al} \cite{Yu-etal:03,Yu-etal:05} used hierarchical micro clustering and a clustering feature tree to obtain an approximate solution for support vector machines.

In this paper, we propose a general optimization algorithm based on clustering and data aggregation, and apply it to three common machine learning problems: least absolute deviation regression (LAD), SVM, and semi-supervised support vector machines (S$^{3}$VM). The algorithm fits the IAD framework, but has additional properties shown for the selected problems in this paper. The ability to report the optimality gap and monotonic convergence to global optimum are features of our algorithm for LAD and SVM, while our algorithm guarantees optimality for S$^{3}$VM without monotonic convergence. Our work for SVM is distinguished from the work of Yu \textit{et al} \cite{Yu-etal:03,Yu-etal:05}, as we iteratively solve weighted SVM and guarantee optimality, whereas they iteratively solve the standard unweighted SVM and thus find only an approximate solution. On the other hand, it is distinguished from Evgeniou and Pontil \cite{Evgeniou-Pontil:02}, as our algorithm is iterative and guarantees global optimum, whereas they used clustering to preprocess data and obtain an approximate optimum. Nath \textit{et al} \cite{Nath2006} and Doppa \textit{et al} \cite{Doppa2010} are different because we use the typical SVM formulation within an iterative framework, whereas they propose an SOCP formulation based on chance constraints.

Our data disaggregation and cluster partitioning procedure is based on the optimality condition derived in this paper: relative location of the observations to the hyperplane (for LAD, SVM, S$^{3}$VM) and labels of the observations (for SVM, S$^{3}$VM). For example, in the SVM case, if the separating hyperplane divides a cluster, the cluster is split. The condition for S$^{3}$VM is even more involved since a single cluster can be split into four clusters. In the computational experiment, we show that our algorithm outperforms the current state-of-the-art algorithms when the data size is large. The implementation of our algorithms is based on in-memory processing, however the algorithms work also when data does not fit entirely in memory and has to be read from disk in batches. The algorithms never require the entire data set to be processed at once. Our contributions are summarized as follows.
\begin{enumerate}[noitemsep]
\item We propose a clustering-based iterative algorithm to solve certain optimization problems, where an optimality condition is derived for each problem. The proposed algorithmic framework can be applied to other problems with certain structural properties (even outside of machine learning). The algorithm is most beneficial when the time complexity of the original optimization problem is high.
\item We present model specific disaggregation and cluster partitioning procedures based on the optimality condition, which is one of the keys for achieving optimality.
\item For the selected machine learning problems, i.e., LAD and SVM, we show that the algorithm monotonically converges to a global optimum, while providing the optimality gap in each iteration. For S$^{3}$VM, we provide the optimality condition. 
\end{enumerate}

We present the algorithmic framework in Section 2 and apply it to LAD, SVM, and S$^{3}$VM in Section 3. A computational study is provided in Section 4, followed by a discussion on the characteristic of the algorithm and how to develop the algorithm for other problems in Section 5.

\section{Algorithm: Aggregate and Iterative Disaggregate (AID)}
\label{section_dab}

We start by defining a few terms. A \textit{data matrix} consists of \textit{entries} (rows) and \textit{attributes} (columns). A machine learning optimization problem needs to be solved over the data matrix. When the entries of the original data are partitioned into several sub-groups, we call the sub-groups \textit{clusters} and we require every entry of the original data to belong to exactly one cluster. Based on the clusters, an \textit{aggregated entry} is created for each cluster to represent the entries in the cluster. This aggregated entry (usually the centroid) represents one cluster, and all aggregated entries are considered in the same attribute space as the entries of the original data. The notion of the \textit{aggregated data} refers to the collection of the aggregated entries. The \textit{aggregated problem} is a similar optimization problem to the original optimization problem, based on the aggregated data instead of the original data. \textit{Declustering} is the procedure of partitioning a cluster into two or more sub-clusters.

We consider optimization problems of the type
\begin{equation}
\label{opt_problem_for_dab}
\mbox{
\begin{tabular}{lll}
$\displaystyle \min_{x, y}$ & $\displaystyle \sum_{i=1}^n f_i(x_i) + f(y)$\\
$s.t.$ & $g^1_i(x_i,y) \geq 0,$ & for every $i = 1,\cdots,n,$\\
 	 & $g^2_i(x_i) \geq 0,$ & for every $i = 1,\cdots,n,$\\
	& $g(y) \geq 0,$
\end{tabular}
}
\end{equation}
where $n$ is the number of entries of $x$, $x_i$ is $i^{\mbox{\scriptsize{th}}}$ entry of $x$, and arbitrary functions $f_i, g^1_i,$ and $g^2_i$ are defined for every $i=1,\cdots,n$. One of the common features of such problems is that the data associated with $x$ is aggregated in practice and an approximate solution can be easily obtained. Well-known problems such as LAD, SVM, and facility location fall into this category. The focus of our work is to design a computationally tractable algorithm that actually yields an optimal solution in a finite number of iterations. 

Our algorithm needs four components tailored to a particular optimization problem or a machine learning model.
\begin{enumerate}[noitemsep]
\item A \textit{definition of the aggregated data} is needed to create aggregated entries.
\item \textit{Clustering and declustering procedures (and criteria)} are needed to cluster the entries of the original data and to decluster the existing clusters.
\item An \textit{aggregated problem} (usually weighted version of the problem with the aggregated data) should be defined.
\item An \textit{optimality condition} is needed to determine whether the current solution to the aggregated problem is optimal for the original problem. 
\end{enumerate}

The overall algorithm is initialized by defining clusters of the original entries and creating aggregated data. In each iteration, the algorithm solves the aggregated problem. If the obtained solution to the aggregated problem satisfies the optimality condition, then the algorithm terminates with an optimal solution to the original problem. Otherwise, the selected clusters are declustered based on the declustering criteria and new aggregated data is created. The algorithm continues until the optimality condition is satisfied. We refer to this algorithm, which is summarized in Algorithm \ref{algo_aid}, as \textit{Aggregate and Iterative Disaggregate} (AID). Observe that the algorithm is finite as we must stop when each cluster is an entry of the original data. In the computational experiment section, we show that in practice the algorithm terminates much earlier.

\begin{algorithm}[ht]
\caption{AID (Aggregate and Iterative Disaggregate)}        
\label{algo_aid}                           
\begin{algorithmic}[1]    
\STATE Create clusters and aggregated data
\STATE \textbf{Do}
\STATE \qquad Solve aggregated problem
\STATE \qquad Check optimality condition
\STATE \qquad \textbf{if} optimality condition is violated \textbf{then} decluster the clusters and redefine aggregated data
\STATE \textbf{While} optimality condition is not satisfied
\end{algorithmic}
\end{algorithm}

In Figure \ref{fig:ex_decluster}, we illustrate the concept of the algorithm. In Figure \ref{fig:ex_decluster_1}, small circles represent the entries of the original data. They are partitioned into three clusters (large dotted circles), where the crosses represent the aggregated data (three aggregated entries). We solve the aggregated problem with the three aggregated entries in Figure \ref{fig:ex_decluster_1}. Suppose that the aggregated solution does not satisfy the optimality condition and that the declustering criteria decide to partition all three clusters. In Figure \ref{fig:ex_decluster_2}, each cluster in Figure \ref{fig:ex_decluster_1} is split into two sub-clusters. Suppose that the optimality condition is satisfied after several iterations. Then, we terminate the algorithm with guaranteed optimality. Figure \ref{fig:ex_decluster_3} represents possible final clusters after several iterations from Figure \ref{fig:ex_decluster_2}. Observe that some of the clusters in Figure \ref{fig:ex_decluster_2} remain the same in Figure \ref{fig:ex_decluster_3}, due to the fact that we selectively decluster.

\begin{figure}[ht]
     \begin{center}
        \subfigure[Initial clusters]{%
           \includegraphics[scale=0.3]{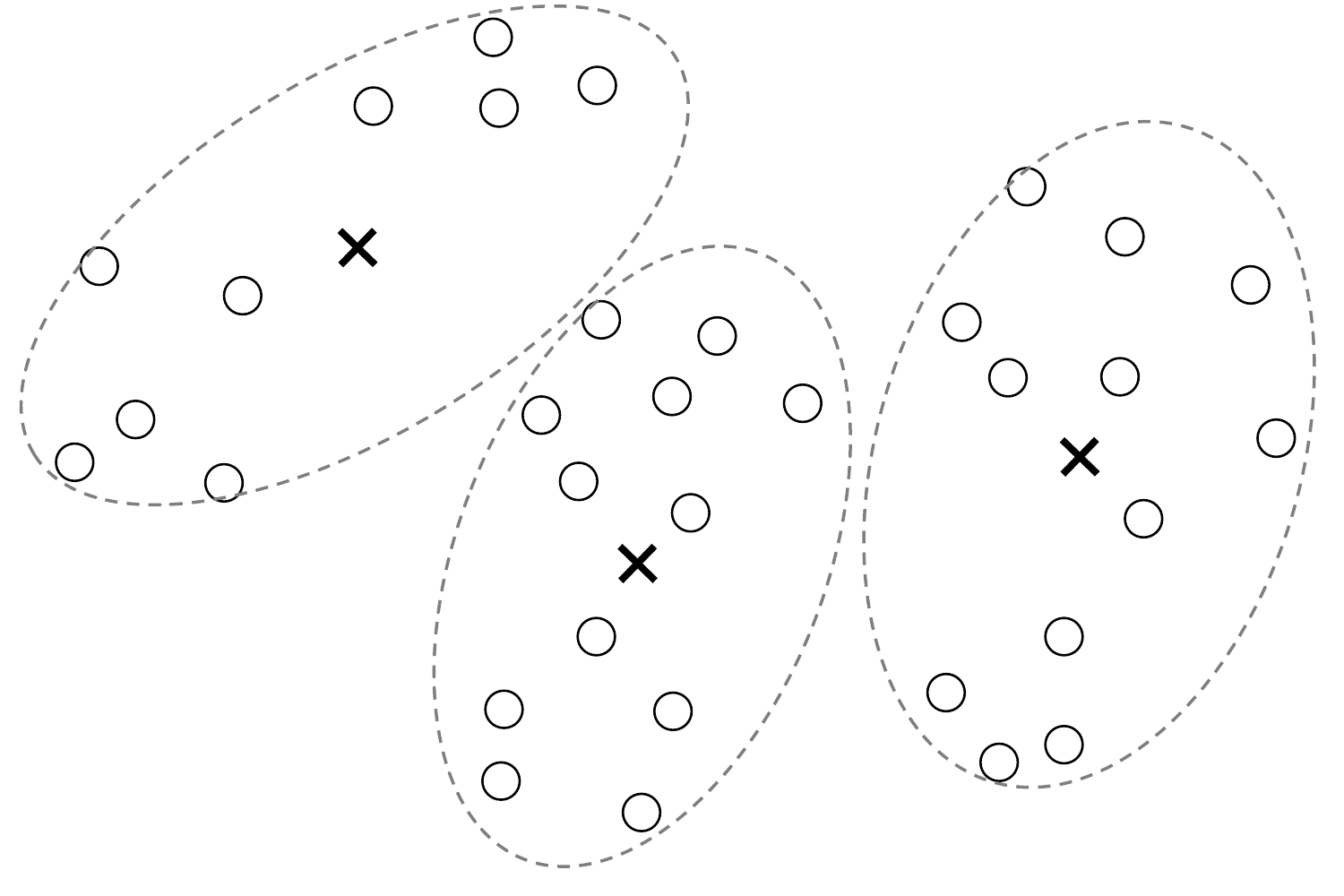} \label{fig:ex_decluster_1}
        }\qquad
        \subfigure[Declustered]{%
           \includegraphics[scale=0.3]{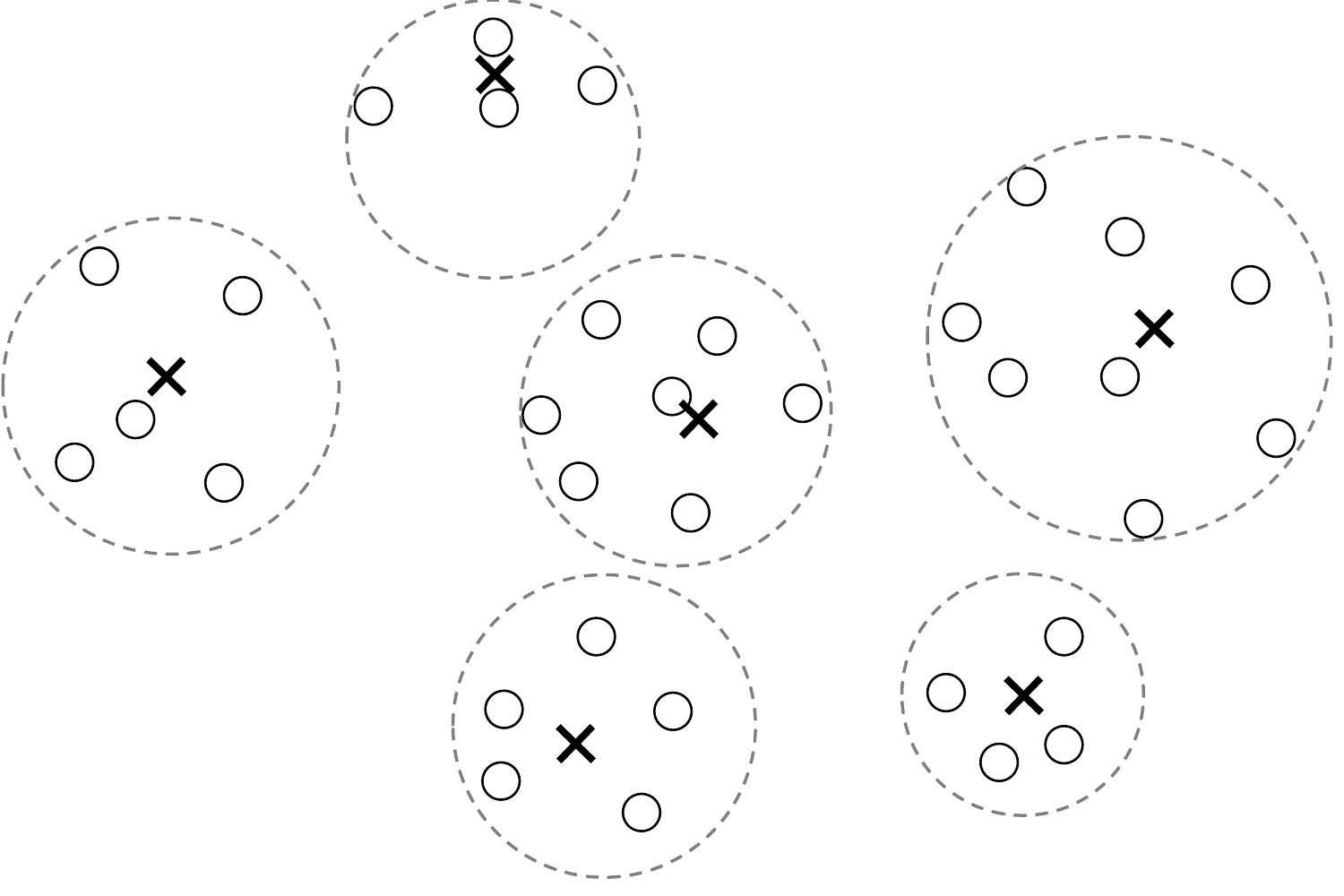} \label{fig:ex_decluster_2}
        }\qquad
        \subfigure[Final clusters]{%
           \includegraphics[scale=0.3]{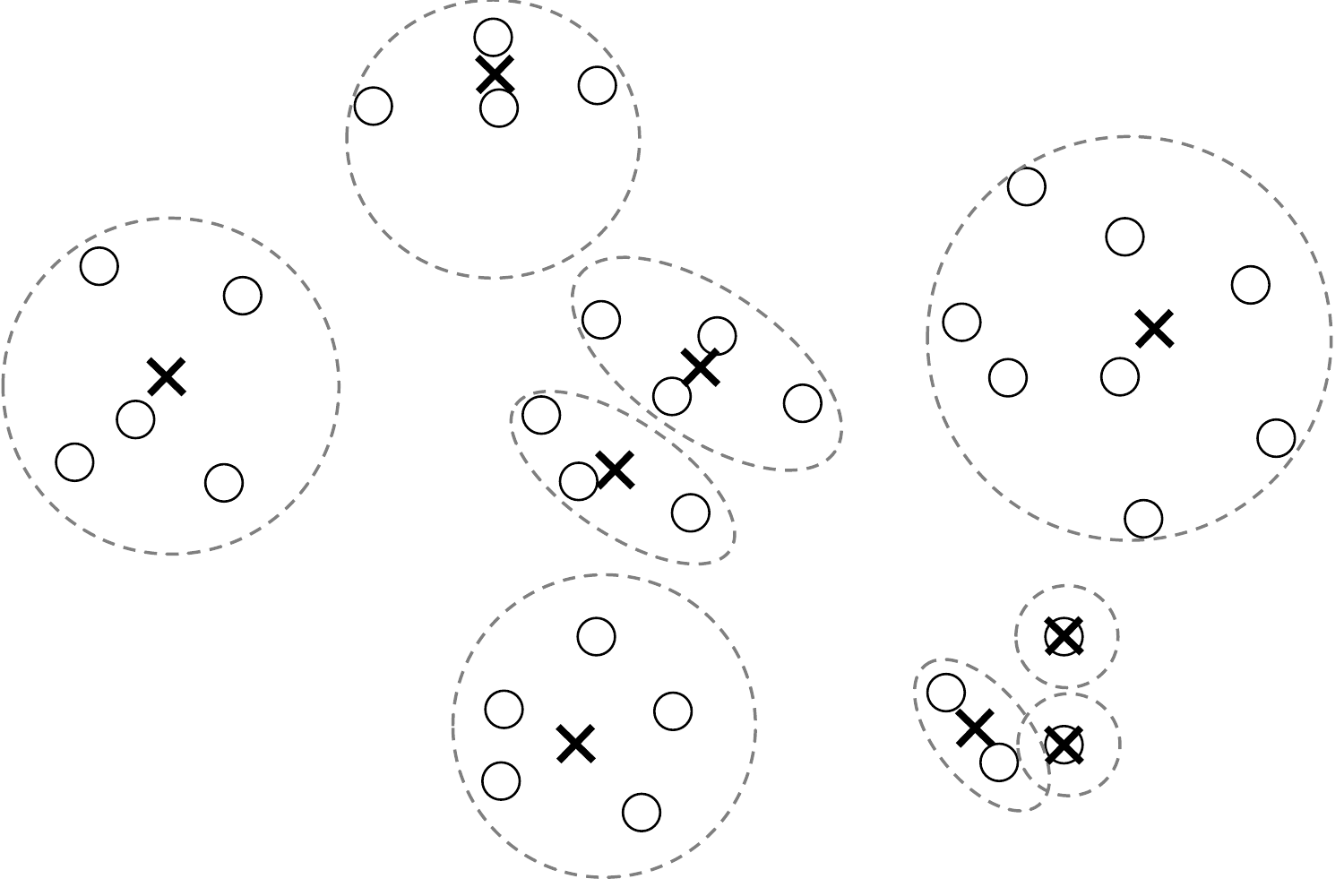} \label{fig:ex_decluster_3}
        }
    \end{center}
    \vspace{-0.5cm}
    \caption{Illustration of AID}
    \label{fig:ex_decluster}  
\end{figure}

We use the following notation in subsequent sections.
\begin{enumerate}[noitemsep]
\item[] $I = \{ 1,2,\cdots, n\}$: Index set of entries, where $n$ is the number of entries (observations)
\item[] $J = \{ 1,2,\cdots, m\}$: Index set of attributes, where $m$ is the number of attributes
\item[] $K^t = \{1,2,\cdots, |K^t| \}$: Index set of the clusters in iteration $t$
\item[] $C^t = \{C_1^t,C_2^t,\cdots,C_{|K^t|}^t\}$: Set of clusters in iteration $t$, where $C_k^t$ is a subset of $I$ for any $k$ in $K^t$
\item[] $T$: Last iteration of the algorithm when the optimality condition is satisfied
\end{enumerate}

\section{AID for Machine Learning Problems}
\label{section_dab_for_ML}

\subsection{Least Absolute Deviation Regression}
The multiple linear least absolute deviation regression problem (LAD) can be formulated as
\begin{equation}
\label{formulation_standard_regression}
 E^* = \min_{\beta \in \mathbb{R}^{m}} \sum_{i \in I} |y_i - \sum_{j \in J} x_{ij} \beta_j |,
\end{equation}
where $x = [x_{ij}] \in \mathbb{R}^{n \times m}$ is the explanatory variable data, $y = [y_i] \in \mathbb{R}^{n}$ is the response variable data, and $\beta \in \mathbb{R}^m$ is the decision variable. Since the objective function of \eqref{formulation_standard_regression} is the summation of functions over all $i$ in $I$, LAD fits \eqref{opt_problem_for_dab}, and we can use AID.

Let us first define the clustering method. Given target number of clusters $|K^0|$, any clustering algorithm can be used to partition $n$ entries into $|K^0|$ initial clusters $C^0 = \{C_1^0,C_2^0, \cdots, C_{|K^0|}^0 \}$. Given $C^t$ in iteration $t$, for each $k \in K^t$, we generate aggregated data by
\begin{center}
$\displaystyle x_{kj}^t = \frac{\sum_{i \in C_k^t} x_{ij}}{|C_k^t|}$, for all $j \in J$, and $y_k^t = \frac{\sum_{i \in C_k^t} y_i}{|C_k^t|}$,
\end{center}
where $x^t \in \mathbb{R}^{|K^t| \times m}$ and $y^t \in \mathbb{R}^{|K^t| \times 1}$. To balance the clusters with different cardinalities, we give weight $|C_k^t|$ to the absolute error associated with $C_k^t$. Hence, we solve the aggregated problem
\begin{equation}
\label{formulation_weighted_regression_iter_t}
F^t = \min_{\beta^t \in \mathbb{R}^{m}} \sum_{k \in K^t} |C_k^t| |y_k^t - \sum_{j \in J} x_{kj}^t \beta_j^t |.
\end{equation}
Observe that any feasible solution to \eqref{formulation_weighted_regression_iter_t} is a feasible solution to \eqref{formulation_standard_regression}. Let $\bar{\beta}^t$ be an optimal solution to \eqref{formulation_weighted_regression_iter_t}. Then, the objective function value of $\bar{\beta}^t$ to \eqref{formulation_standard_regression} with the original data is 
\begin{equation}
E^t = \sum_{i \in I} |y_i - \sum_{j \in J} x_{ij} \bar{\beta}_j^t |.
\end{equation}

Next, we present the declustering criteria and construction of $C^{t+1}$. Given $C^t$ and $\bar{\beta}^t$, we define the clusters for iteration $t+1$ as follows. 
\begin{enumerate}[noitemsep]
\item[] \textsf{Step 1} $C^{t+1} = \emptyset$.
\item[] \textsf{Step 2} For each  $k \in K^t$,
	\begin{enumerate}
	\item[] \textsf{Step 2(a)} If $y_i - \sum_{j \in J} x_{ij} \bar{\beta}_j^t$ for all $i \in C_k^t$ have the same sign, then $C^{t+1} \gets C^{t+1} \cup \{ C_k^t \}$.
	\item[] \textsf{Step 2(b)} Otherwise, decluster $C_k^t$ into two clusters: $C_{k+}^{t} = \{ i \in C_k^t | y_i - \sum_{j \in J} x_{ij} \bar{\beta}_j^t > 0 \}$ and $C_{k-}^{t} = \{ i \in C_k^t | y_i - \sum_{j \in J} x_{ij} \bar{\beta}_j^t \leq 0 \}$, and set $C^{t+1} \gets C^{t+1} \cup \{ C_{k+}^{t}, C_{k-}^{t} \}$.
	\end{enumerate}
\end{enumerate}

The above procedure keeps cluster $C_k^t$ if all original entries in the clusters are on the same side of the regression hyperplane. Otherwise, the procedure splits $C_k^t$ into two clusters $C_{k+}^t$ and $C_{k-}^t$, where the two clusters contain original entries on the one and the other side of the hyperplane. It is obvious that this rule implies a finite algorithm.

In Figure \ref{fig:lad}, we illustrate AID for LAD. In Figure \ref{fig:lad_1}, the small circles and crosses represent the original and aggregated entries, respectively, where the large dotted circles are the clusters associated with the aggregated entries. The straight line represents the regression line $\bar{\beta}^t$ obtained from an optimal solution to \eqref{formulation_weighted_regression_iter_t}. In Figure \ref{fig:lad_2}, the shaded and empty circles are the original entries below and above the regression line, respectively. Observe that two clusters have original entries below and above the regression line. Hence, we decluster the two clusters based on the declustering criteria and obtain new clusters and aggregated data for the next iteration in Figure \ref{fig:lad_3}.

\begin{figure}[ht]
     \begin{center}
        \subfigure[Clusters $C^t$ and $\bar{\beta}^t$]{%
           \includegraphics[scale=0.3]{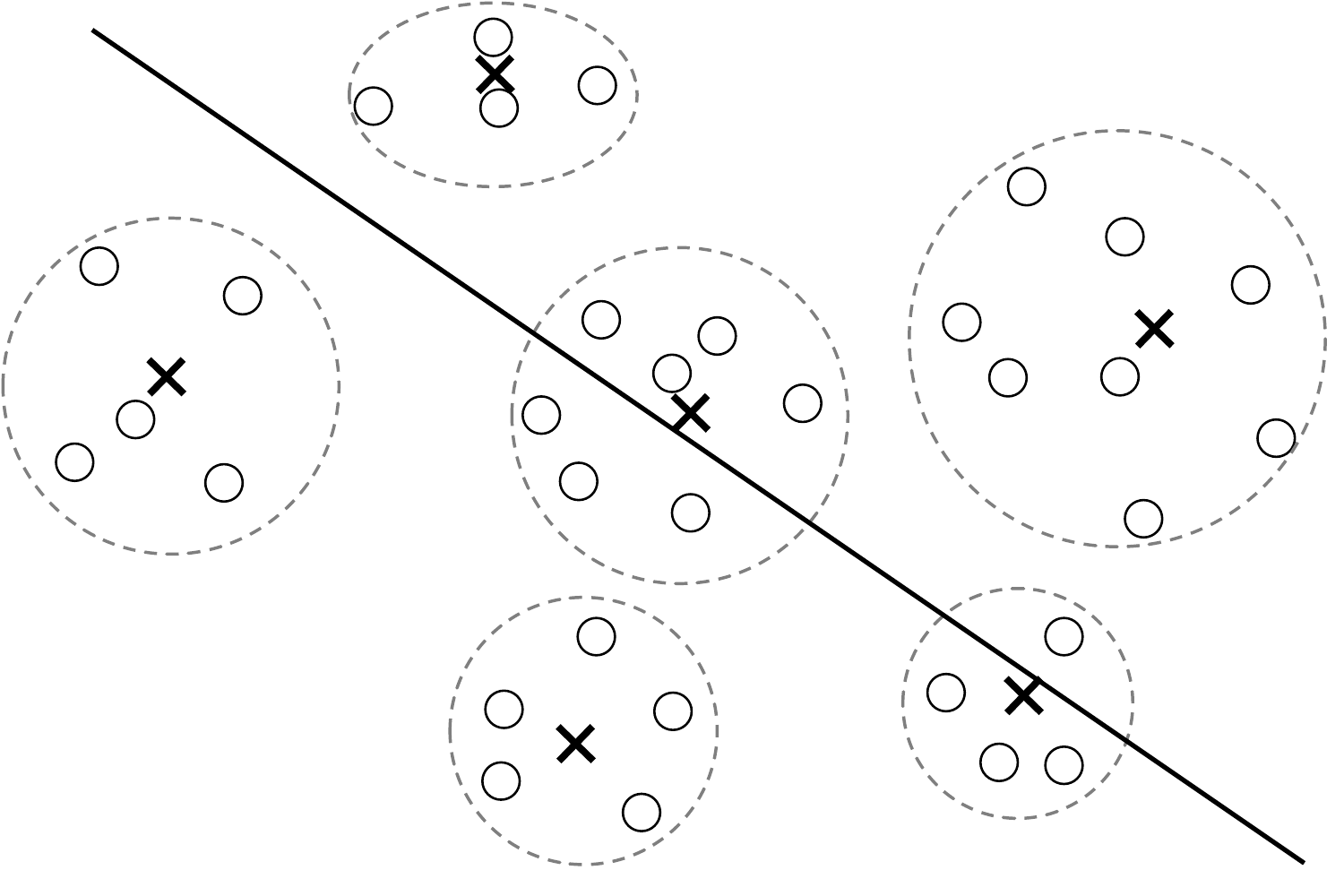} \label{fig:lad_1}
        }\qquad
        \subfigure[Declustered]{%
           \includegraphics[scale=0.3]{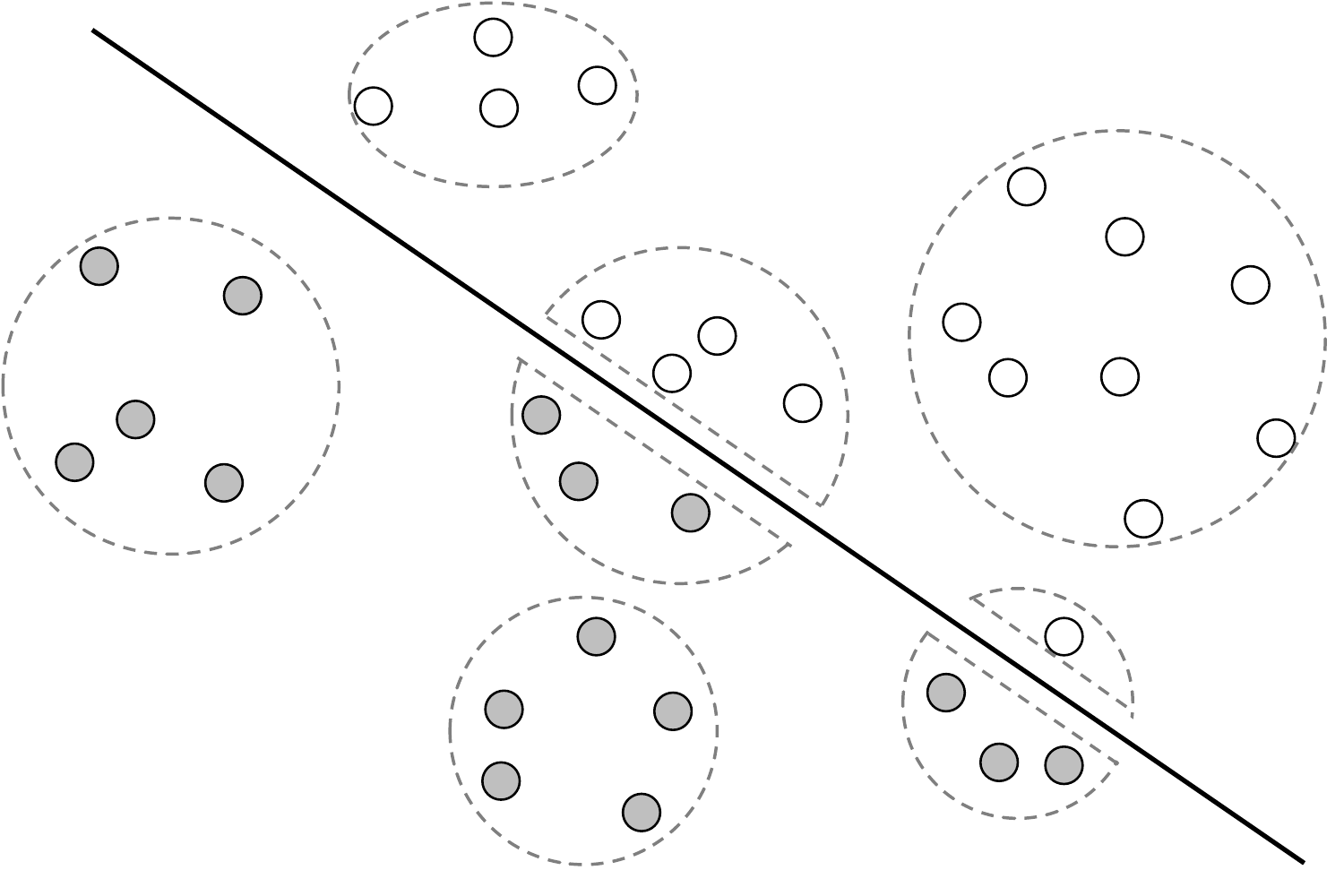} \label{fig:lad_2}
        }\qquad
        \subfigure[New clusters $C^{t+1}$]{%
           \includegraphics[scale=0.3]{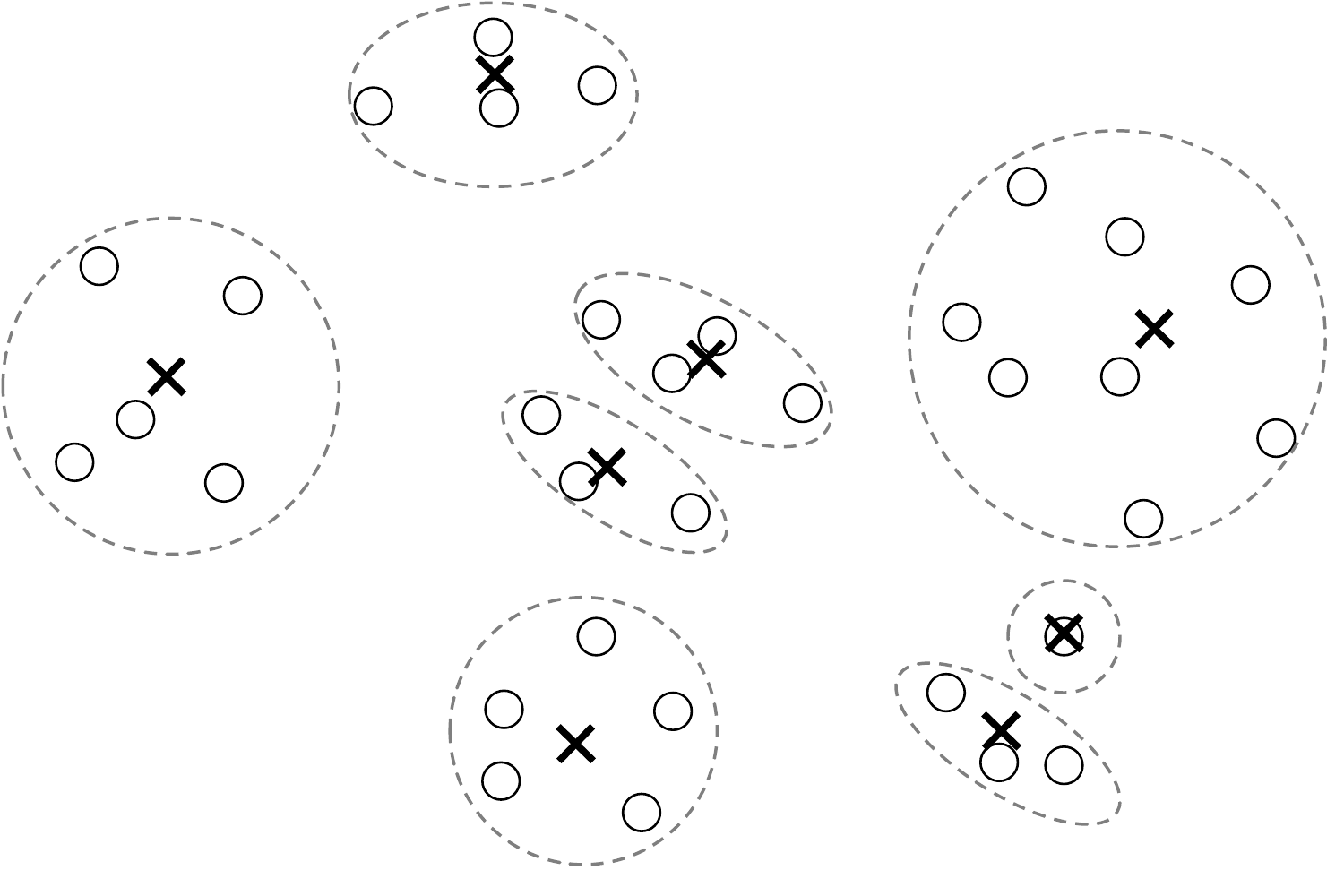} \label{fig:lad_3}
        }
    \end{center}
    \vspace{-0.5cm}
    \caption{Illustration of AID for LAD}
    \label{fig:lad}  
\end{figure}

Now we are ready to present the optimality condition and show that $\bar{\beta}^t$ is an optimal solution to \eqref{formulation_standard_regression} when the optimality condition is satisfied. The optimality condition presented in the following proposition is closely related to the clustering criteria.

\begin{proposition}
\label{proposition_F_and_E_equivalent}
If $y_i - \sum_{j \in J} x_{ij} \bar{\beta}_j^t$ for all $i \in C_k^t$ have the same sign for all $k \in K^t$, then $\bar{\beta}^t$ is an optimal solution to \eqref{formulation_standard_regression}. In other words, if all entries in $C_k^t$ are on the same side of the hyperplane defined by $\bar{\beta}^t$ for all $k \in K^t$, then $\bar{\beta}^t$ is an optimal solution to \eqref{formulation_standard_regression}. Further, $E^t = F^t$.
\end{proposition}
\begin{proof}
Let $\beta^*$ be an optimal solution to \eqref{formulation_standard_regression}. Then, we derive
\vspace{0.3cm}

\begin{tabular}{lll}
$E^*$ & $=$ & $\displaystyle \sum_{i \in I}  | y_i - \sum_{j \in J} x_{ij} \beta_j^* | = \sum_{k \in K^t}\sum_{i \in C_k^t}  | y_i - \sum_{j \in J} x_{ij} \beta_j^* |$ \\
	& $\geq$& $\displaystyle \sum_{k \in K^t} | \sum_{i \in C_k^t}  y_i - \sum_{j \in J} x_{ij} \beta_j^* | = \sum_{k \in K^t} |C_k^t| |y_k^t - \sum_{j \in J} x_{kj}^t \beta_j^* |$\\
	& $\geq$ & $\displaystyle \sum_{k \in K^t} |C_k^t| |y_k^t - \sum_{j \in J} x_{kj}^t \bar{\beta}_j^t| =  \sum_{k \in K^t} | \sum_{i \in C_k^t} y_i - \sum_{i \in C_k^t} \sum_{j \in J} x_{ij} \bar{\beta}_j^t |$\\
	& $=$ & $\displaystyle \sum_{k \in K^t} \sum_{i \in C_k^t} | y_i - \sum_{j \in J} x_{ij} \bar{\beta}_j^t | = \sum_{i \in I}  | y_i - \sum_{j \in J} x_{ij} \bar{\beta}_j^t | = E^t$,\\
\end{tabular}
\vspace{0.3cm}

\noindent where the third line holds since $\bar{\beta}^t$ is optimal to \eqref{formulation_weighted_regression_iter_t} and the fourth line is based on the condition that all observations in $C_k^t$ are on the same side of the hyperplane defined by $\bar{\beta}^t$, for all $k \in K^t$. Since $\bar{\beta}^t$ is feasible to \eqref{formulation_standard_regression}, clearly $E^* \leq E^t$, which shows $E^* = E^t$. This implies that $\bar{\beta}^t$ is an optimal solution to \eqref{formulation_standard_regression}. Observe that $\sum_{k \in K^t} |C_k^t| |y_k^t - \sum_{j \in J} x_{kj}^t \bar{\beta}_j^t|$ in the fifth line is equivalent to $F^t$. Hence, we also showed $E^t = F^t$ by the fifth to ninth lines.
\end{proof}

We also show the non-decreasing property of $F^t$ in $t$ and the convergence.
\begin{proposition}
\label{proposition_lad_improving_obj}
We have $F^{t-1} \leq F^{t}$ for $t=1,\cdots,T$. Further, $F^T = E^T = E^*$.
\end{proposition}
\begin{proof}
For simplicity, let us assume that $\{ C_1^{t-1} \} = C^{t-1} \setminus C^{t}$, $\{ C_1^{t},C_2^{t} \} = C^{t} \setminus C^{t-1}$, and $C_1^{t-1} = C_1^{t} \cup C_2^{t}$. That is, $C_1^{t-1}$ is the only cluster in $C^{t-1}$ such that the entries in $C_1^{t-1}$ have both positive and negative signs, and $C_1^{t-1}$ is partitioned into $C_1^t$ and $C_2^t$ for iteration $t$. Then, we derive
\vspace{0.3cm}

\noindent
\begin{tabular}{lll}
$F^{t-1}$ & $=$ & $\displaystyle \Big|C_1^{t-1}\Big| \Big|y_1^{t-1} - \sum_{j \in J} x_{1j}^{t-1} \bar{\beta}_j^{t-1} \Big| + \sum_{k \in K^{t-1} \setminus \{1\}} \Big|C_k^{t-1}\Big| \Big|y_k^{t-1} - \sum_{j \in J} x_{kj}^{t-1} \bar{\beta}_j^{t-1} \Big|$\\[0.2cm]
	& $\leq$ & $\displaystyle \Big|C_1^{t-1}\Big| \Big|y_1^{t-1} - \sum_{j \in J} x_{1j}^{t-1} \bar{\beta}_j^{t} \Big| + \sum_{k \in K^{t-1} \setminus \{1\}} \Big|C_k^{t-1}\Big| \Big|y_k^{t-1} - \sum_{j \in J} x_{kj}^{t-1} \bar{\beta}_j^{t} \Big|$\\[0.2cm]
	& $=$ & $\displaystyle \Big| \sum_{i \in C_1^{t-1}} y_i - \sum_{i \in C_1^{t-1}} \sum_{j \in J} x_{ij} \bar{\beta}_j^{t} \Big| + \sum_{k \in K^{t-1} \setminus \{1\}} \Big|C_k^{t-1}\Big| \Big|y_k^{t-1} - \sum_{j \in J} x_{kj}^{t-1} \bar{\beta}_j^{t} \Big|$\\[0.2cm]
	& $=$ & $\displaystyle \Big| \sum_{i \in C_1^{t}} y_i - \sum_{i \in C_1^{t}} \sum_{j \in J} x_{ij} \bar{\beta}_j^{t} + \sum_{i \in C_2^{t}} y_i - \sum_{i \in C_2^{t}} \sum_{j \in J} x_{ij} \bar{\beta}_j^{t}\Big| + \sum_{k \in K^{t-1} \setminus \{1\}} \Big|C_k^{t-1}\Big| \Big|y_k^{t-1} - \sum_{j \in J} x_{kj}^{t-1} \bar{\beta}_j^{t} \Big|$\\[0.2cm]
	& $\leq$ & $\displaystyle \Big| \sum_{i \in C_1^{t}} y_i - \sum_{i \in C_1^{t}} \sum_{j \in J} x_{ij} \bar{\beta}_j^{t}\Big| + \Big|\sum_{i \in C_2^{t}} y_i - \sum_{i \in C_2^{t}} \sum_{j \in J} x_{ij} \bar{\beta}_j^{t}\Big| + \sum_{k \in K^{t-1} \setminus \{1\}} \Big|C_k^{t-1}\Big| \Big|y_k^{t-1} - \sum_{j \in J} x_{kj}^{t-1} \bar{\beta}_j^{t} \Big| $\\[0.2cm]
	& $=$ & $\displaystyle \Big| C_1^{t} \Big| \Big| y_1^{t} - \sum_{j \in J} x_{1j}^{t} \bar{\beta}_j^{t}\Big| + \Big| C_2^{t} \Big| \Big| y_2^{t} - \sum_{j \in J} x_{2j}^{t} \bar{\beta}_j^{t}\Big| +  \sum_{k \in K^{t-1} \setminus \{1\}} \Big|C_k^{t-1}\Big| \Big|y_k^{t-1} - \sum_{j \in J} x_{kj}^{t-1} \bar{\beta}_j^{t} \Big|$\\[0.2cm]
	& $=$ & $\displaystyle  \Big| C_1^{t} \Big| \Big| y_1^{t} - \sum_{j \in J} x_{1j}^{t} \bar{\beta}_j^{t}\Big| + \Big| C_2^{t} \Big| \Big| y_2^{t} - \sum_{j \in J} x_{2j}^{t} \bar{\beta}_j^{t}\Big| + \sum_{k \in K^{t} \setminus \{1,2\}} \Big|C_k^{t}\Big| \Big|y_k^{t} - \sum_{j \in J} x_{kj}^{t} \bar{\beta}_j^{t} \Big|$\\
	& $=$ & $F^{t}$,
\end{tabular}
\vspace{0.3cm}

\noindent where the second line holds since $\bar{\beta}^{t-1}$ is an optimal solution to the aggregate problem in iteration $t-1$, and the seventh line follows from the fact that there exist $q \in K^{t-1} \setminus \{1\}$ and $k \in K^{t} \setminus \{1,2\}$ such that $C_q^{t-1} = C_k^t$. For the cases with multiple clusters in $t$ are declustered, we can use the similar technique. This completes the proof.
\end{proof}

By Proposition \ref{proposition_lad_improving_obj}, in any iteration, $F^t$ can be interpreted as a lower bound to \eqref{formulation_standard_regression}. Further, the optimality gap $\frac{E^{\mbox{\scriptsize{best}}} - F^t}{E^{\mbox{\scriptsize{best}}}}$ is non-increasing in $t$, where $\displaystyle E^{\mbox{\scriptsize{best}}} = \min_{s=1,\cdots,t} \{E^s\}$.

\subsection{Support Vector Machines}
\label{subsection_SVM}

One of the most popular forms of support vector machines (SVM) includes a kernel satisfying the Mercer's theorem \cite{Mercer:09} and soft margin. Let $\phi: x_i \in \mathbb{R}^{m} \mapsto \phi(x_i) \in \mathbb{R}^{m'}$ be the mapping function that maps from the $m$-dimensional original feature space to $m'$-dimensional new feature space. Then, the primal optimization problem for SVM is written as
\begin{equation}
\label{def_SVM_kernel}
\begin{split}
& E^* = \min_{w,b,\xi} \frac{1}{2} \|w\|^2 + M \| \xi \|_1 \\
& \qquad \quad \mbox{s.t.} \quad y_i(w \phi(x_i) + b) \geq 1 - \xi_i, \xi_i \geq 0, i \in I,
\end{split}
\end{equation}
where $x = [x_{ij}] \in \mathbb{R}^{n \times m}$ is the feature data, $y = [y_i] \in \{-1,1\}^n$ is the class (label) data, and $w \in \mathbb{R}^{m'}, b \in \mathbb{R},$ and $\xi \in \mathbb{R}_{+}^n$ are the decision variables, and the corresponding dual optimization problem is written as
\begin{equation}
\label{def_SVM_dual_kernel}
\begin{split}
& \max \sum_{i \in I} \alpha_i - \frac{1}{2} \sum_{i,j \in I} \mathcal{K}(x_i,x_j) \alpha_i \alpha_j y_i y_j \\
& \mbox{s.t.} \quad \sum_{i \in I} \alpha_i y_i = 0,\\
& \qquad 0 \leq \alpha_i \leq M, i \in I,
\end{split}
\end{equation}
where $\mathcal{K}(x_i,x_j) = \langle \phi(x_i), \phi(x_j) \rangle$ is the kernel function. In this case, $\phi(x_i)$ in \eqref{def_SVM_kernel} can be interpreted as new data in $m'$-dimensional feature space with linear kernel. Hence, without loss of generality, we derive all of our findings in this section for \eqref{def_SVM_kernel} with the linear kernel, while all of the results hold for any kernel function satisfying the Mercer's theorem. However, in Appendix \ref{section_appendix_B_aid_kernel}, we also describe AID with direct use of the kernel function.

By using the linear kernel, \eqref{def_SVM_kernel} is simplified as
\begin{equation}
\label{def_SVM}
\begin{split}
& E^* = \min_{w,b,\xi} \frac{1}{2} \|w\|^2 + M \| \xi \|_1 \\
& \qquad \quad \mbox{s.t.} \quad y_i(w x_i + b) \geq 1 - \xi_i, \xi_i \geq 0, i \in I,
\end{split}
\end{equation}
where $w \in \mathbb{R}^{m}$. Since $\|\cdot\|_1$ in the objective function of \eqref{def_SVM} is the summation of the absolute values over all $i$ in $I$ and the constraints are defined for each $i$ in $I$, SVM fits \eqref{opt_problem_for_dab}. Hence, we apply AID to solve \eqref{def_SVM}.

Let us first define the clustering method. The algorithm maintains that the observations $i_1$ and $i_2$ with different labels $(y_{i_1} \neq y_{i_2})$ cannot be in the same cluster. We first cluster all data $i$ with $y_i = 1$, and then we cluster those with $y_i = -1$. Thus we run the clustering algorithm twice. This gives initial clusters $C^0 = \{C_1^0,C_2^0, \cdots, C_{|K^0|}^0 \}$. Given $C^t$ in iteration $t$, for each $k \in K^t$, we generate aggregated data by 
\begin{center}
$\displaystyle x_{k}^t = \frac{\sum_{i \in C_k^t} x_{i}}{|C_k^t|}$ and $y_k^t = \frac{\sum_{i \in C_k^t} y_i}{|C_k^t|} \in \{-1,1\}$,
\end{center}
where $x^t \in \mathbb{R}^{|K^t| \times m}$ and $y^t \in \{-1,1\}^{|K^t|}$. Note that, since we create a cluster with observations with the same label, we have
\begin{equation}
\label{property_aggregated_y}
y_k^t = y_i \mbox{ for all } i \in C_k^t.
\end{equation}
By giving weight $|C_k^t|$ to $\xi_k^t$, we obtain
\begin{equation}
\label{def_SVM2}
\begin{split}
& F^t = \min_{w^t, b^t, \xi^t} \frac{1}{2} \|w^t \|^2 + M \sum_{k \in K^t} |C_k^t|  \xi_k^t \\
& \qquad \quad \mbox{s.t.} \quad y_k^t \big[w^t x_k^t + b^t \big] \geq 1 - \xi_k^t, \xi_k^t \geq 0, k \in K^t,
\end{split}
\end{equation}
where $w^t \in \mathbb{R}^{m},b^t \in \mathbb{R},$ and $\xi^t \in \mathbb{R}_{+}^{|K^t|}$ are the decision variables. Note that $\xi$ in \eqref{def_SVM} has size of $n$, whereas the aggregated data has $|K^t|$ entries. Note also that \eqref{def_SVM2} is weighted SVM \cite{weightedSVM}, where weight is $|C_k^t|$ for aggregated entry $k \in K^t$.

Next we present the declustering criteria and construction of $C^{t+1}$. Let $(w^*,\xi^*,b^*)$ and $(\bar{w}^t,\bar{\xi}^t,\bar{b}^t)$ be optimal solutions to \eqref{def_SVM} and \eqref{def_SVM2}, respectively. Given $C^t$ and $(\bar{w}^t,\bar{\xi}^t,\bar{b}^t)$, we define the clusters for iteration $t+1$ as follows. 
\begin{enumerate}[noitemsep]
\item[] \textsf{Step 1} $C^{t+1} \gets \emptyset$.
\item[] \textsf{Step 2} For each  $k \in K^t$,
	\begin{enumerate}[noitemsep]
	\item[] \textsf{Step 2(a)} If (\rmnum{1}) $1-y_i (\bar{w}^t x_i + \bar{b}^t) \leq 0$ for all $i \in C_k^t$ or (\rmnum{2}) $1-y_i (\bar{w}^t x_i + \bar{b}^t) > 0$ for all $i \in C_k^t$, then $C^{t+1} \gets C^{t+1} \cup \{ C_k^t \}$.
	\item[] \textsf{Step 2(b)} Otherwise, decluster $C_k^t$ into two clusters: $C_{k+}^{t} = \{ i \in C_k^t | 1-y_i (\bar{w}^t x_i + \bar{b}^t) > 0 \}$ and $C_{k-}^{t} = \{ i \in C_k^t | 1-y_i (\bar{w}^t x_i + \bar{b}^t) \leq 0 \}$, and set  $C^{t+1} \gets C^{t+1} \cup \{ C_{k+}^{t}, C_{k-}^{t} \}$.
	\end{enumerate}
\end{enumerate}

In Figure \ref{fig:svm}, we illustrate AID for SVM. In Figure \ref{fig:svm_1}, the small white circles and crosses represent the original entries with labels 1 and -1, respectively. The small black circles and crosses represent the aggregated entries, where the large circles are clusters associated with the aggregated entries. The plain line represents the separating hyperplane $(\bar{w}^t, \bar{b}^t)$ obtained from an optimal solution to \eqref{def_SVM2}, where the margins are implied by the dotted lines. The shaded large circles represent the clusters violating the optimality condition in Proposition \ref{proposition_svm_F_and_E_equivalent}. In Figure \ref{fig:svm_2}, below the bottom dotted line is the area such that observations with label 1 (circles) have zero error and above the top dotted line is the area such that observations with label -1 (crosses) have zero error. Observe that two clusters have original entries below and above the corresponding dotted lines. Based on the declustering criteria, the two clusters are declustered and we obtain new clusters in Figure \ref{fig:svm_3}.

\begin{figure}[ht]
     \begin{center}
        \subfigure[Clusters $C^t$ and $(\bar{w}^t, \bar{b}^t)$]{%
           \includegraphics[scale=0.17]{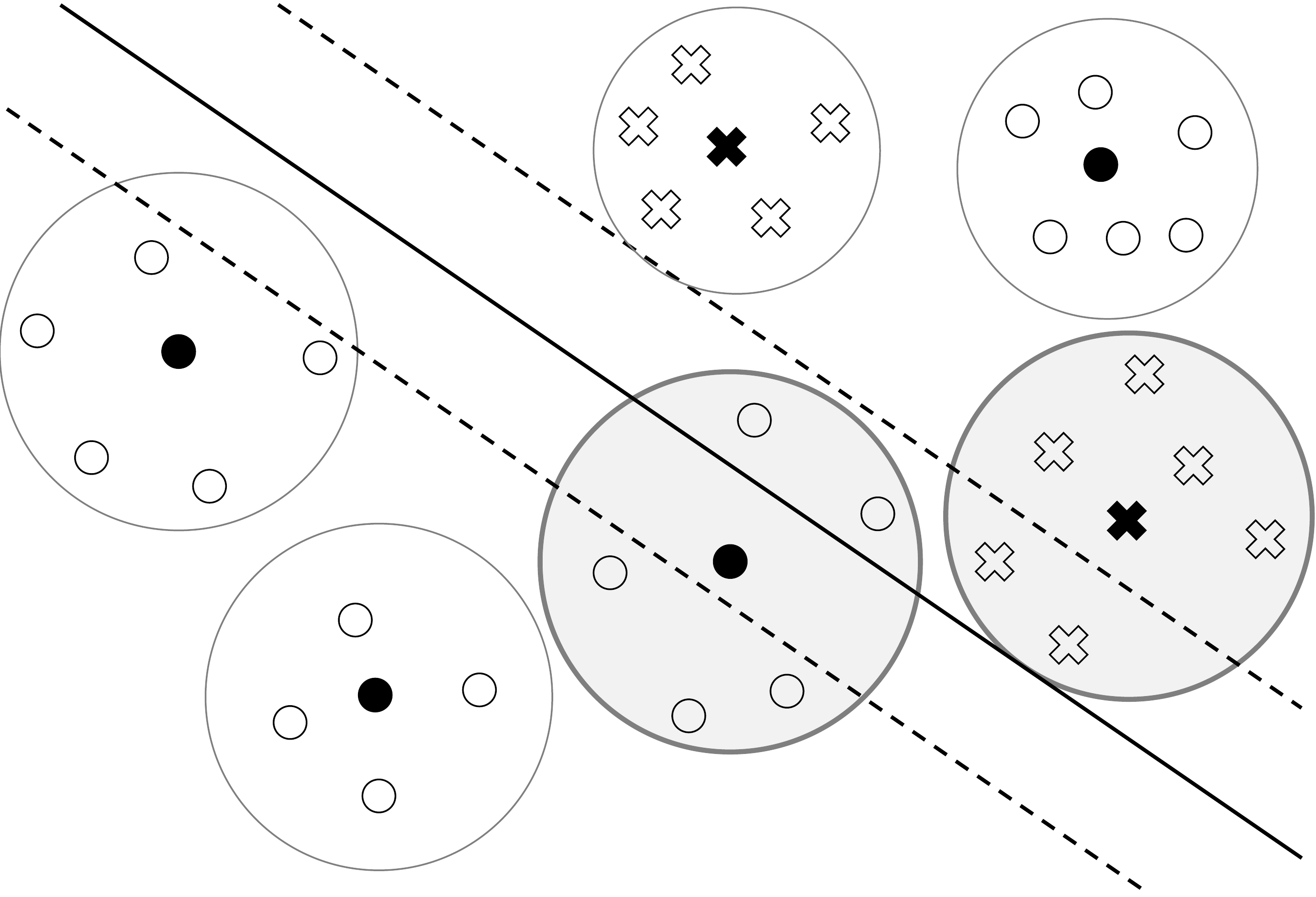} \label{fig:svm_1}
        }\qquad
        \subfigure[Declustered]{%
           \includegraphics[scale=0.17]{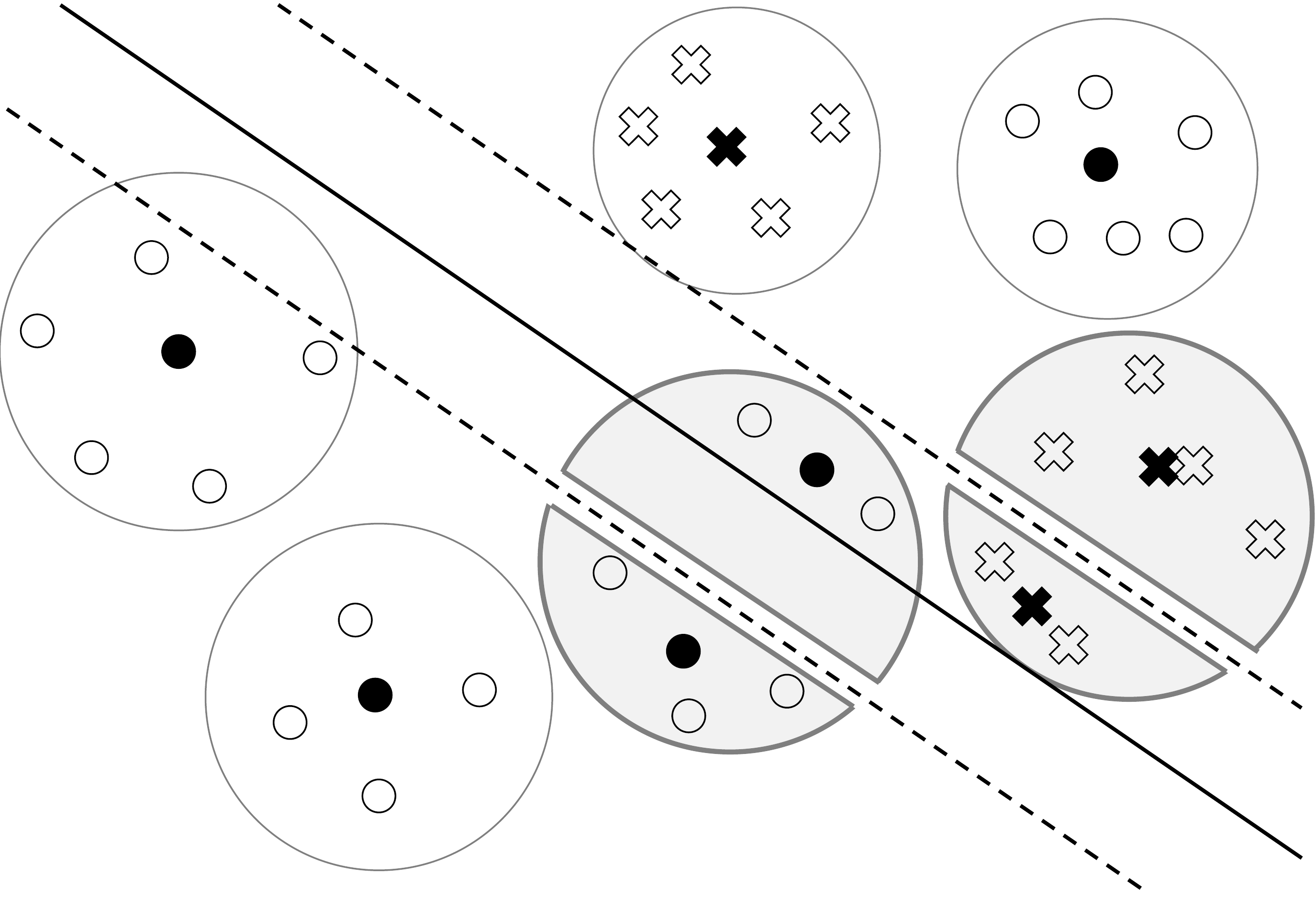} \label{fig:svm_2}
        }\qquad
        \subfigure[New clusters $C^{t+1}$]{%
           \includegraphics[scale=0.17]{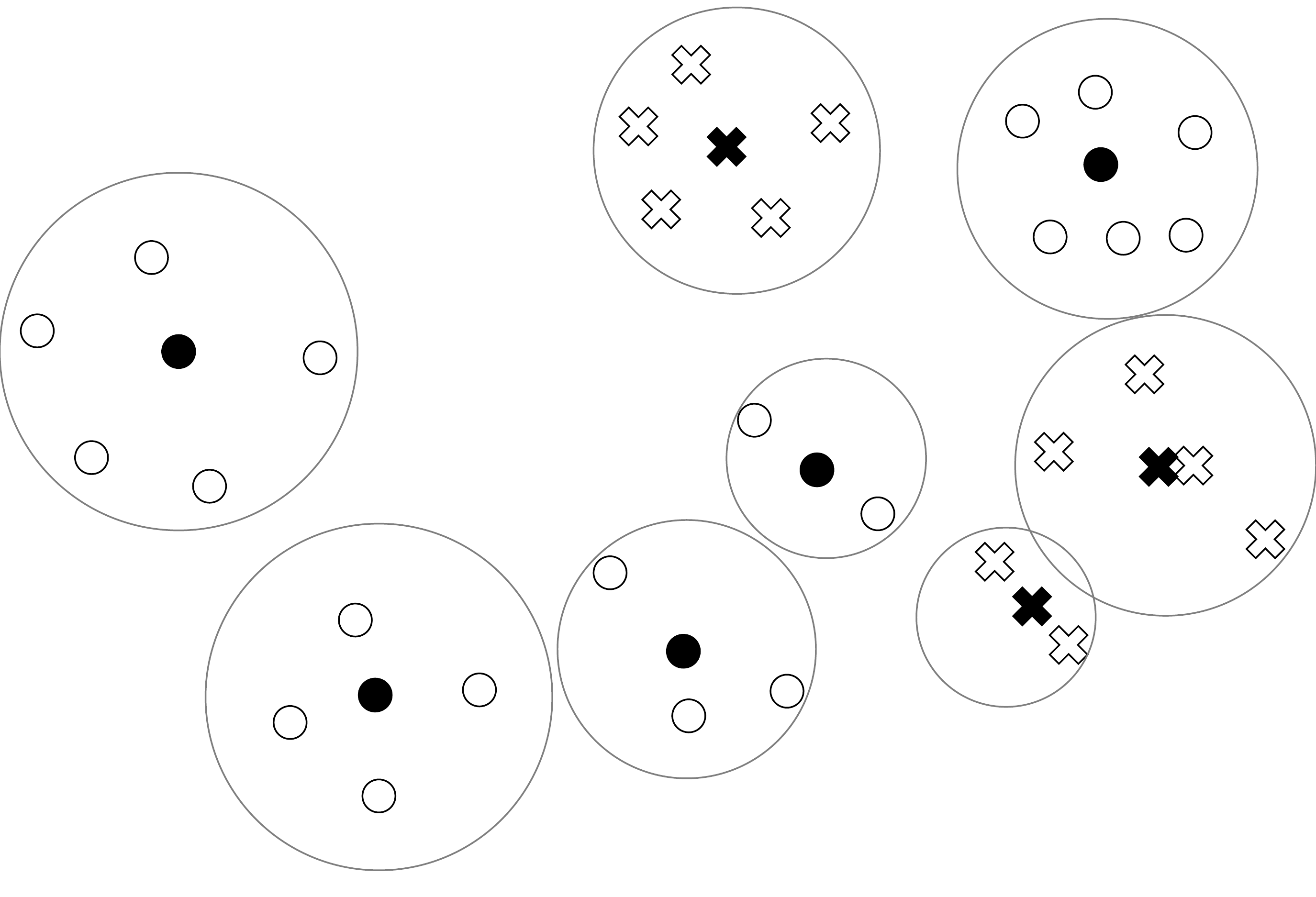} \label{fig:svm_3}
        }
    \end{center}
    \vspace{-0.5cm}
    \caption{Illustration of AID for SVM}
    \label{fig:svm}  
\end{figure}

Note that a feasible solution to \eqref{def_SVM} does not have the same dimension as a feasible solution to \eqref{def_SVM2}. In order to analyze the algorithm, we convert feasible solutions to \eqref{def_SVM} and \eqref{def_SVM2} to feasible solutions to \eqref{def_SVM2} and \eqref{def_SVM}, respectively.
\begin{enumerate}[noitemsep]
\item Conversion from \eqref{def_SVM} to \eqref{def_SVM2}\\
Given a feasible solution $(w,b,\xi) \in (\mathbb{R}^m, \mathbb{R}, \mathbb{R}_{+}^n)$ to \eqref{def_SVM}, we define a feasible solution $(w^t,b^t,\xi^t) \in (\mathbb{R}^m, \mathbb{R}, \mathbb{R}_{+}^{|K^t|})$ to \eqref{def_SVM2} as follows: $w^t := w$, $b^t := b$, and $\xi_k^t := \frac{\sum_{i \in C_k^t} \xi_i}{|C_k^t|}$ for $k \in K^t$.
\item Conversion from \eqref{def_SVM2} to \eqref{def_SVM}\\
Given a feasible solution $(w^t,b^t,\xi^t) \in (\mathbb{R}^m, \mathbb{R}, \mathbb{R}_{+}^{|K^t|})$ to \eqref{def_SVM2}, we define a feasible solution $(w,b,\xi) \in (\mathbb{R}^m, \mathbb{R}, \mathbb{R}_{+}^n)$ to \eqref{def_SVM} as follows: $w := w^t$, $b := b^t$, and $\xi_i := \max \{ 0, 1-y_i(w^t x_i + b^t) \}$ for $i \in I$.
\end{enumerate}
Given an optimal solution $(w^*,b^*,\xi^*)$ to \eqref{def_SVM}, by using the above mappings, we denote by $(\hat{w}^*,\hat{b}^*,\hat{\xi}^*)$ the corresponding feasible solution to \eqref{def_SVM2}. Likewise, given an optimal solution $(\bar{w}^t,\bar{b}^t,\bar{\xi}^t)$ to \eqref{def_SVM2}, we denote by $(\hat{w}^t,\hat{b}^t,\hat{\xi}^t)$ the corresponding feasible solution to \eqref{def_SVM}. The objective function value of $(\hat{w}^t,\hat{b}^t,\hat{\xi}^t)$ to \eqref{def_SVM} is evaluated by 
\begin{equation}
E^t = \frac{1}{2} \| \hat{w}^t \|^2 + M \| \hat{\xi}^t \|_1.
\end{equation}

In Propositions \ref{proposition_svm_F_and_E_equivalent} and \ref{proposition_svm_improving_obj}, we present the optimality condition and monotone convergence property.
\begin{proposition}
\label{proposition_svm_F_and_E_equivalent}
For all $k \in K^t$, if (\rmnum{1}) $1-y_i (\bar{w}^t x_i + \bar{b}^t) \leq 0$ for all $i \in C_k^t$ or (\rmnum{2}) $1-y_i (\bar{w}^t x_i + \bar{b}^t) > 0$ for all $i \in C_k^t$, then $(\hat{w}^t, \hat{b}^t, \hat{\xi}^t_i)$ is an optimal solution to \eqref{def_SVM}. In other words, if all entries in $C_k^t$ are on the same side of the margin-shifted hyperplane of the separating hyperplane $(\bar{w}^t, \bar{b}^t)$, then $(\hat{w}^t, \hat{b}^t, \hat{\xi}^t_i)$ is an optimal solution to \eqref{def_SVM}. Further, $E^t = F^t$.
\end{proposition}
\begin{proof}
We can derive
\vspace{0.2cm}

\begin{tabular}{lll}
$\frac{1}{2} \|w^*\|^2 + M \sum_{i \in I}  \xi_i^* $ & $=$ & $\frac{1}{2} \|w^*\|^2 + M \sum_{k \in K^t} |C_k^t| \frac{\sum_{i \in C_k^t}  \xi_i^* }{|C_k^t|} $\\
	& $=$ & $\frac{1}{2} \|\hat{w}^*\|^2 + M \sum_{k \in K^t} |C_k^t| \hat{\xi}_k^*$\\
	& $\geq$ & $\frac{1}{2} \|\bar{w}^t\|^2 + M \sum_{k \in K^t} |C_k^t| \bar{\xi}_k^t$\\
\end{tabular}

\begin{tabular}{lll}
\textcolor{white}{$\frac{1}{2} \|w^*\|^2 + M \sum_{i \in I}  \xi_i^* $}	& $=$ & $\frac{1}{2} \|\bar{w}^t\|^2 + M \sum_{k \in K^t} \max \{ 0,  |C_k^t| - |C_k^t| y_k^t ( \bar{w}^t x_k^t + \bar{b}^t) \}$\\
	& $=$ & $\frac{1}{2} \|\bar{w}^t\|^2 + M \sum_{k \in K^t} \max \{ 0,  |C_k^t| - y_k^t \bar{w}^t \sum_{i \in C_k^t} x_i - y_k^t |C_k^t| \bar{b}^t \}$\\
	& $=$ & $\frac{1}{2} \|\bar{w}^t\|^2 + M \sum_{k \in K^t} \max \Big\{ 0, \sum_{i \in C_k^t} [1 - y_i (\bar{w}^t x_i + \bar{b}^t)] \Big\}$\\
	& $=$ & $\frac{1}{2} \|\bar{w}^t\|^2 + M \sum_{k \in K^t} \sum_{i \in C_k^t} \max \{ 0,  1 - y_i (\bar{w}^t x_i + \bar{b}^t) \}$\\
	& $=$ & $\frac{1}{2} \|\hat{w}^t\|^2 + M \sum_{k \in K^t} \sum_{i \in C_k^t} \hat{\xi}_i^t$\\
	& $\geq$ & $\frac{1}{2} \|w^*\|^2 + M \sum_{i \in I} \xi_i^* $,
\end{tabular}
\vspace{0.2cm}

\noindent where the second line follows from the definition of $\hat{\xi}_k^*$, the third line holds since $(\bar{w}^t,\bar{b}^t,\bar{\xi}^t)$ is an optimal solution to \eqref{def_SVM2}, the fourth line is by the definition of $\bar{\xi}^t$, the fifth line is by the definition of $x_k^t$, the eighth line is true because of the assumption such that all observations are on the same side of the margin-shifted hyperplane of the separating hyperplane (optimality condition), and the last line holds since $(w^*,b^*,\xi^*)$ is an optimal solution to \eqref{def_SVM}. Observe that the inequalities above must hold at equality. This implies that $(\hat{w}^t, \hat{b}^t, \hat{\xi}^t_i)$ is an optimal solution to \eqref{def_SVM}.
\end{proof}

Because $(\hat{w}^t,\hat{b}^t)$ defines an optimal hyperplane, we are also able to obtain the corresponding dual optimal solution $\hat{\alpha}^t \in \mathbb{R}^{n}$ for (6). However, unlike primal optimal solutions, $\hat{\alpha}^t$ cannot be directly constructed from dual solution $\bar{\alpha}^t \in \mathbb{R}^{|K^t|}$ of $(\bar{w}^t,\bar{b}^t)$ for the aggregated problem within the current settings. Within a modified setting presented later in this section, we can explain the relationship between $\bar{\alpha}^t$ and $\hat{\alpha}^t$, modified optimality condition, declustering procedure, and the construction of $\hat{\alpha}^t$ based on $\bar{\alpha}^t$.

\begin{proposition}
\label{proposition_svm_improving_obj}
We have $F^{t-1} \leq F^t$ for $t=1,\cdots,T$. Further, $F^T = E^T = E^*$.
\end{proposition}

\begin{proof}
Recall that $(\bar{w}^t, \bar{b}^t, \bar{\xi}^t)$ is an optimal solution to \eqref{def_SVM2} with aggregated data $x^t$ and $y^t$. Let $(\tilde{w}^{t-1}, \tilde{b}^{t-1},\tilde{\xi}^{t-1})$ be a feasible solution to \eqref{def_SVM2} with aggregated data $x^{t-1}$ and $y^{t-1}$ such that $\tilde{w}^{t-1} = \bar{w}^t$, $\tilde{b}^{t-1} = \bar{b}^t$, and $\tilde{\xi}^{t-1} = \max \{ 0, 1 - y_{k}^{t-1} ( \bar{w}^t x_k^{t-1} + \bar{b}^t ) \} $ for $k \in K^{t-1}$. In other words, $(\tilde{w}^{t-1}, \tilde{b}^{t-1}, \tilde{\xi}^{t-1})$ is a feasible solution to \eqref{def_SVM2} with aggregated data $x^{t-1}$ and $y^{t-1}$, but generated based on $(\bar{w}^t,\bar{b}^t,\bar{\xi}^t)$. For simplicity, let us assume that $\{ C_1^{t-1} \} = C^{t-1} \setminus C^{t}$, $\{ C_1^{t},C_2^{t} \} = C^{t} \setminus C^{t-1}$, and $C_1^{t-1} = C_1^{t} \cup C_2^{t}$. The cases such that more than one cluster of $C^{t-1}$ are declustered in iteration $t$ can be derived using the same technique.

Observe that there exists a pair $(q,k)$, $q \in K^{t-1} \setminus \{1\}$ and $k \in K^t \setminus \{1,2\}$ such that
\begin{equation}
\label{one2one_match}
\tilde{\xi}_q^{t-1} = \bar{\xi}_k^t
\end{equation}
for all $q$ in $K^{t-1} \setminus \{1\}$ and the match between $K^{t-1} \setminus \{1\}$ and $K^t \setminus \{1,2\}$ is one-to-one. This is because the aggregated data for these clusters remains same and the hyper-plane used, $(\tilde{w}^{t-1}, \tilde{b}^{t-1})$ and $(\bar{w}^t, \bar{b}^t)$, are the same. Hence, we derive
\vspace{0.3cm}

\begin{tabular}{lll}
$F^{t-1}$ & $=$ & $\frac{1}{2} \| \bar{w}^{t-1} \|^2 + M |C_1^{t-1}| \bar{\xi}_1^{t-1}  + M \sum_{k \in K^{t-1} \setminus \{1\}} |C_k^{t-1}| \bar{\xi}_k^{t-1} $\\
	& $\leq$ & $\frac{1}{2} \| \tilde{w}^{t-1} \|^2 + M |C_1^{t-1}| \tilde{\xi}_1^{t-1}  + M \sum_{k \in K^{t-1} \setminus \{1\}} |C_k^{t-1}| \tilde{\xi}_k^{t-1} $\\
	\textcolor{white}{$F^{t-1}$}	& $\leq$ & $\frac{1}{2} \| \bar{w}^{t} \|^2 + M |C_1^{t-1}| \tilde{\xi}_1^{t-1}  +  M \sum_{k \in K^{t} \setminus \{1,2\}} |C_k^{t}| \bar{\xi}_k^{t} $\\
	& $\leq$ & $\frac{1}{2} \| \bar{w}^{t} \|^2 + M |C_1^{t}| \bar{\xi}_1^{t}  + M |C_2^{t}| \bar{\xi}_2^{t} +  M \sum_{k \in K^{t} \setminus \{1,2\}} |C_k^{t}| \bar{\xi}_k^{t} $\\
\end{tabular}

\begin{tabular}{lll}
\textcolor{white}{$F^{t-1}$}	& $=$ & $\frac{1}{2} \| \bar{w}^{t} \|^2  +  M \sum_{k \in K^{t}} |C_k^{t}| \bar{\xi}_k^{t} $\\
	& $=$ & $F^t$,
\end{tabular}
\vspace{0.3cm}

\noindent where the first inequality holds because $\bar{w}^{t-1}, \bar{b}^{t-1}, \bar{\xi}^{t-1}$ is an optimal solution to \eqref{def_SVM2} with aggregated data $x^{t-1}$ and $y^{t-1}$, the second inequality follows by the fact that $\tilde{w}^{t-1} = \bar{w}^t$ and by \eqref{one2one_match}, and the last inequality is true because
\vspace{0.3cm}

\begin{tabular}{lll}
$|C_1^{t-1}| \tilde{\xi}_1^{t-1} $ & $=$ & $|C_1^{t-1}|$ $ \max \Big\{ 0, 1 - y_1^{t-1} (\bar{w}^t x_1^{t-1} + \bar{b}^t) \Big\}$\\
 & $=$ & $|C_1^{t-1}|$ $ \max \Big\{ 0, 1 - y_1^{t-1} (\bar{w}^t x_1^{t-1} + \bar{b}^t) \Big\} $\\
& $=$ & $\max \Big\{ 0, |C_1^{t-1}| - |C_1^{t-1}| y_1^{t-1} \bar{w}^t x_1^{t-1} - |C_1^{t-1}| y_1^{t-1} \bar{b}^t \Big\}$\\
& $=$ & $\max \Big\{ 0, |C_1^{t}|+|C_2^{t}| - y_1^{t-1} \bar{w}^t (\sum_{i \in C_1^t} x_i + \sum_{i \in C_2^t} x_i) -  |C_1^{t}| y_1^t \bar{b}^t - |C_2^{t}| y_2^t \bar{b}^t \Big\}$\\
\end{tabular}

\begin{tabular}{lll}
\textcolor{white}{$|C_1^{t-1}| \tilde{\xi}_1^{t-1} $}	& $=$ & $\max \Big\{ 0, |C_1^{t}| - |C_1^{t}| y_1^{t} \bar{w}^t x_1^t - |C_1^{t}| y_1^t \bar{b}^t + |C_2^{t}| - |C_2^{t}| y_2^{t} \bar{w}^t x_2^t - |C_2^{t}| y_2^t \bar{b}^t  \Big\}$\\
	& $\leq$ & $\max \Big\{ 0, |C_1^{t}| - |C_1^{t}| y_1^{t} \bar{w}^t x_1^t - |C_1^{t}| y_1^t \bar{b}^t \Big\} + \max \Big\{ 0, |C_2^{t}| - |C_2^{t}| y_2^{t} \bar{w}^t x_2^t - |C_2^{t}| y_2^t \bar{b}^t  \Big\} $\\
	& $=$ & $|C_1^{t}| \max \{ 0, 1 - y_1^{t} ( \bar{w}^t x_1^t +\bar{b}^t) \} + |C_2^{t}| \max \{ 0,  1 - y_2^{t}( \bar{w}^t x_2^t + \bar{b}^t ) \}$\\
	& $=$ & $|C_1^{t}| \bar{\xi}_1^{t}  + |C_2^{t}| \bar{\xi}_2^{t}$.
\end{tabular}
\vspace{0.3cm}

\noindent where the fourth line holds by (\rmnum{1}) $|C_1^{t-1}| = |C_1^{t}|+|C_2^{t}| $,(\rmnum{2}) $y_1^t = y_2^t = y_1^{t-1}$ (by \eqref{property_aggregated_y}), and (\rmnum{3}) by the definition of $x_1^{t-1}$, and the fifth line holds due to the definition of $x_1^t$ and $x_2^t$. This completes the proof.
\end{proof}

So far, we have explained the algorithm based on the primal formulation of SVM. However, we can also explain the relationship between the dual of the original and aggregated problems by proposing a modified procedure. Let us divide observations in $C_k^t$ into three sets.
\begin{enumerate}[noitemsep]
\item $1-y_i (\bar{w}^t x_i + \bar{b}^t) < 0$ for $i \in C_k^t$
\item $1-y_i (\bar{w}^t x_i + \bar{b}^t) = 0$ for $i \in C_k^t$
\item $1-y_i (\bar{w}^t x_i + \bar{b}^t) > 0$ for $i \in C_k^t$
\end{enumerate}
These three sets correspond to the following three cases for the original data given hyperplane $(\bar{w}^t, \bar{b}^t)$ from an optimal solution of \eqref{def_SVM2}.
\begin{enumerate}[noitemsep]
\item Observations correctly classified: $1-y_i (\bar{w}^t x_i + \bar{b}^t) < 0$ and $\hat{\xi}_i^t = 0$ in \eqref{def_SVM} and $\hat{\alpha}_i = 0$ in the dual of \eqref{def_SVM}
\item Observations on the hyperplane: $1-y_i (\bar{w}^t x_i + \bar{b}^t) = 0$ and $\hat{\xi}_i^t = 0$ in \eqref{def_SVM} and $0 < \hat{\alpha}_i < M$ in the dual of \eqref{def_SVM}
\item Observations in the margin or misclassified: $1-y_i (\bar{w}^t x_i + \bar{b}^t) > 0$ and $\hat{\xi}_i^t > 0$ in \eqref{def_SVM} and $\hat{\alpha}_i = M$ in the dual of \eqref{def_SVM}
\end{enumerate}
Suppose we are given $\bar{\alpha}^t$, a dual optimal solution that corresponds to $(\bar{w}^t, \bar{b}^t)$. Then we can construct dual optimal solution $\hat{\alpha}^t$ for the original problem from $\bar{\alpha}^t$ by
\begin{center}
$\hat{\alpha}_i^t = \frac{\bar{\alpha}^t}{|C_k|}$ for $i \in C_k^t, k \in K^t$.
\end{center}
With this definition, now all original observations in a cluster belong to exactly one of the three categories above.

We first show that $\hat{\alpha}^t$ is a feasible solution. Let us consider cluster $k \in K^t$. We derive
\begin{center}
$\sum_{i \in C_k^t} \hat{\alpha}_i^t y_i = \sum_{i \in C_k^t} \hat{\alpha}_i^t y_k = \sum_{i \in C_k^t} \frac{\bar{\alpha}_k^t}{|C_k^t|} y_k = y_k \frac{|C_k^t|\bar{\alpha}_k^t }{|C_k^t|} = y_k \bar{\alpha}_k^t$,
\end{center}
where the first equality holds because all labels are the same for a cluster and the second equality is obtained by plugging the definition of $\hat{\alpha}_i^t$. Because $y_k \bar{\alpha}_k^t = \sum_{i \in C_k^t} \hat{\alpha}_i^t y_i$ and $\sum_{k \in K^t} y_k \bar{\alpha}_k^t = 0$, we conclude that $\hat{\alpha}^t$ is a feasible solution to \eqref{def_SVM_dual_kernel}.

In order to show optimality, we show that $\hat{\alpha}^t$ and $\bar{\alpha}^t$ give the same hyperplane. Let us consider cluster $k \in K^t$. We derive
\begin{center}
$\sum_{i \in C_k^t} \hat{\alpha}_i^t y_i x_i = \sum_{i \in C_k^t} \hat{\alpha}_i^t y_k x_i = \sum_{i \in C_k^t} \frac{\bar{\alpha}_k^t }{|C_k^t|} y_k x_i = \bar{\alpha}_k^t y_k  \frac{\sum_{i \in C_k^t}  x_i}{|C_k^t|} = \bar{\alpha}_k^t y_k x_k^t$, 
\end{center}
where the first equality holds because all labels are the same for a cluster, the second equality is obtained by plugging the definition of $\hat{\alpha}_i^t$, and the last equality is due to the definition of $x_k^t$. Because $\bar{\alpha}_k^t y_k x_k^t = \sum_{i \in C_k^t} \hat{\alpha}_i^t y_i x_i$, by summing over all clusters, we obtain $\bar{w}^t =\bar{\alpha}_k^t y_k x_k^t = \sum_{i \in C_k^t} \hat{\alpha}_i^t y_i x_i = \hat{w}^t$, which completes the proof.

\subsection{Semi-Supervised SVM}

The task of semi-supervised learning is to decide classes of unlabeled (unsupervised) observations given some labeled (supervised) observations. Semi-supervised SVM (S$^{3}$VM) is an SVM-based learning model for semi-supervised learning. In S$^{3}$VM, we need to decide classes for unlabeled observations in addition to finding a hyperplane. Let $I_l = \{1,\cdots,l\}$ and $I_u = \{l+1, \cdots,n \}$ be the index sets of labeled and unlabeled observations, respectively. The standard S$^{3}$VM with linear kernel is written as the following minimization problem over both the hyperplane parameters $(w,b)$ and the unknown label vector $d := [d_{l+1} \cdots d_n]$,
\begin{equation}
\label{definition_sssvm}
E^* = \min_{w,b,d} \frac{1}{2} \| w\|^2 + M_l \sum_{i \in I_l} \max \{ 0, 1-y_i(w x_i + b)\} + M_u \sum_{i \in I_u} \max \{ 0, 1-d_i(w x_i + b)\},
\end{equation}
where $x = [x_{ij}] \in \mathbb{R}^{n \times m}$ is the feature data, $y = [y_i] \in \{-1,1\}^{l}$ is the class (label) data, and $w \in \mathbb{R}^{m}$, $b \in \mathbb{R}$, and $d \in \{-1,1\}^{|I_u|}$ are the decision variables. By introducing error term $\xi \in \mathbb{R}_{+}^n$, \eqref{definition_sssvm} is rewritten as

\begin{equation}
\label{def_SSSVM}
\begin{array}{llll}
E^* & = & \displaystyle \min_{w,\xi,b,d}& \displaystyle  \frac{1}{2} \|w\|^2 + M_l \sum_{i \in I_l} \xi_i + M_u \sum_{i \in I_u} \xi_i\\
    &   & \displaystyle \quad \mbox{s.t.} & \displaystyle y_i(w x_i + b) \geq 1 - \xi_i, \xi_i \geq 0, i \in I_l,\\
    &   & & \displaystyle d_i(w x_i + b) \geq 1 - \xi_i, \xi_i \geq 0, i \in I_u.
\end{array}
\end{equation}
Observe that $y_i$, $i \in I_l$, is given as data, whereas $d_i$, $i \in I_u$, is unknown and decision variable. Note that \eqref{def_SSSVM} has non-convex constraints. In order to eliminate the non-convex constraints, Bennett and Demiriz \cite{Bennett:1999} proposed a mixed integer quadratic programming (MIQP) formulation
\begin{equation}
\label{def_SSSVM_alternative}
\begin{array}{llll}
E^* = & \displaystyle \min_{\substack{w,b,d\\\xi,\eta^+, \eta^-} }  &  \displaystyle \frac{1}{2} \|w\|^2 + M_l \sum_{i \in I_l} \xi_i + M_u \sum_{i \in I_u} \eta^+_i + \eta^-_i  \\ 
& \quad s.t. &    y_i(w x_i + b) \geq 1 - \xi_i, \xi_i \geq 0, & i \in I_l, \\
&	& w x_i + b + \eta^+_i + M (1-d_i) \geq 1, & i \in I_u, \\
&	& 0 \leq \eta^+_i \leq M d_i,  & i \in I_u,  \\
&	& -(w x_i + b) + \eta^-_i + M d_i \geq 1,&  i \in I_u \\
&	& 0 \leq \eta^-_i \leq M(1-d_i),  & i \in I_u,\\
&	& d_i \in \{0,1\},  & i \in I_u, 
\end{array}
\end{equation}
where $M>0$ is a large number and $w \in \mathbb{R}^{m}$, $b \in \mathbb{R}$, $d \in \{-1,1\}^{|I_u|}$, $\eta^+ \in \mathbb{R}_{+}^{|I_u|}$, $\eta^- \in \mathbb{R}_{+}^{|I_u|}$ are the decision variables. Note that $d_i$ in \eqref{def_SSSVM_alternative} is different from $d_i$ in \eqref{definition_sssvm}. In \eqref{def_SSSVM_alternative}, if $d_i = 1$ then observation $i$ is in class 1 and if $d_i = 0$ then observation $i$ is in class $-1$. Note also that, by the objective function, if $d_i = 1$ then $\eta_i^-$ becomes 0 and if $d_i = 0$ then $\eta_i^+$ becomes 0 at optimum. A Branch-and-Bound algorithm to solve \eqref{definition_sssvm} is proposed by Chapelle \textit{et al} \cite{Chapelle-etal:07} and an MIQP solver is used to solve \eqref{def_SSSVM_alternative} in \cite{Bennett:1999}. However, both of the works only solve small size problems. See Chapelle \textit{et al}  \cite{Chapelle-etal:08} for detailed survey of the literature. Observe that \eqref{def_SSSVM_alternative} fits \eqref{opt_problem_for_dab}. Hence, we use AID to solve \eqref{def_SSSVM_alternative} with larger size instances, which were not solved by the works in \cite{Bennett:1999, Chapelle-etal:07}.

Similar to the approach used to solve SVM in Section 3.1, we define clusters $C^t = \{ C_1^t, C_2^t, \cdots, C_{|K_l^t|}^t \}$ for the labeled data, where $K_l^t$ is the index set of the clusters of labeled data in iteration $t$ and each cluster contains observations with same label. In addition, we have clusters for the unlabeled data $D^t = \{ D_1^t, D_2^t, \cdots, D_{|K_u^t|}^t \}$, where $K_u^t$ is the index set of the clusters of unlabeled data in iteration $t$. In the S$^{3}$VM case, initially we need to run a clustering algorithm three times. We generate aggregated data by
\begin{itemize}[noitemsep]
\item[] $x_{k}^t = \frac{\sum_{i \in C_k^t} x_{i}}{|C_k^t|}$ and $y_k^t = \frac{\sum_{i \in C_k^t} y_i}{|C_k^t|}$ for each $k \in K_l^t$ given $C^t$,
\item[] $x_{k}^t = \frac{\sum_{i \in D_k^t} x_{i}}{|D_k^t|}$ for each $k \in K_u^t$ given $D^t$.
\end{itemize}
Using the aggregated data, we obtain the aggregated version of \eqref{def_SSSVM} as
\begin{equation}
\label{def_SSSVM_agg}
\begin{array}{llll}
F^t = & \displaystyle \min_{w^t,\xi^t,b^t,d^t}& \displaystyle \frac{1}{2} \|w\|^2 + M_l \sum_{k \in K_l^t} |C_k^t| \xi_k^t + M_u \sum_{k \in K_u^t} |D_k^t| \xi_k^t\\
       & \displaystyle \quad \mbox{s.t.} & \displaystyle y_k^t(w^t x_k^t + b^t) \geq 1 - \xi_k^t, \xi_k^t \geq 0, k \in K_l^t,\\
       & & \displaystyle d_k^t(w^t x_k^t + b^t) \geq 1 - \xi_k^t, \xi_k^t \geq 0, k \in K_u^t,
\end{array}
\end{equation}

\textcolor{white}{ywpark}

\noindent where $y_k^t$, $k \in K_l^t$, is known and $d_k^t$, $k \in K_u^t$, is unknown. Observe that \eqref{def_SSSVM_agg} can be solved optimally by the Branch-and-Bound algorithm in \cite{Chapelle-etal:07} or by an MIQP solver.

In the following lemma, we show that, given an optimal hyperplane $(w^*,b^*)$, optimal values of $\xi^* \in \mathbb{R}^{n}$ and $d \in \{0,1\}^{|I_u|}$ can be obtained.
\begin{lemma}
\label{lemma_sssvm_opt_sol_property}
Let $(w^*,b^*,\xi^*, d^*)$ be an optimal solution for \eqref{def_SSSVM}. For $i \in I_u$, if $\xi_i^* > 0$, then we must have $d_i^*(w^* x_i + b^*) \geq 0$. For $i \in I_u$, if $\xi_i^* = 0$ and $\max \{ 0,1-d_i(w^* x_i + b^*)\} = 0$ only for one of $d_i = 1$ or $d_i = -1$, then we must have $d_i^* = 1$ if $w^* x_i + b^* \geq 0$, $d_i^* = -1$ if $w^* x_i + b^* < 0$. A similar property holds for an optimal solution of the aggregated problem \eqref{def_SSSVM_agg}.
\end{lemma}
\begin{proof}
Suppose that $\xi_i^* > 0$. Hence, we have $\xi_i^* = 1 - d_i^*(w^* x_i + b^*) > 0$. If $w^* x_i + b^* < 0$, then setting $d_i^* = -1$ decreases $\xi_i^*$ most because $1 - (-1)(w^* x_i + b^*) < 1 - (1)(w^* x_i + b^*)$. Likewise, if $w^* x_i + b^* \geq 0$, then setting $d_i^* = 1$ decreases $\xi_i^*$.
\end{proof}

For the analysis, we define the following sets.
\begin{enumerate}[noitemsep]
\item For labeled observations in $I_l$, given hyperplane $(w,b)$, let us define subsets of $I_l$.
	\begin{enumerate}
	\item[] $I^{+}_{(w,b)} = \{ i \in I_l | 1-y_i(wx_i + b) > 0 \}$
	\item[] $I^{-}_{(w,b)} = \{ i \in I_l | 1-y_i(wx_i + b) \leq 0 \}$
	\end{enumerate}
\item For unlabeled observations in $I_u$, given hyperplane $(w,b)$ and labels $d$, let us define subsets of $I_u$.
	\begin{enumerate}[noitemsep]
	\item[] $I^{++}_{(w,b,d)} = \{ i \in I_u | 1-d_i(wx_i + b) > 0, wx_i + b > 0 \}$
	\item[] $I^{+-}_{(w,b,d)} = \{ i \in I_u | 1-d_i(wx_i + b) > 0, wx_i + b \leq 0 \}$
	\item[] $I^{-+}_{(w,b,d)} = \{ i \in I_u | 1-d_i(wx_i + b) \leq 0, wx_i + b > 0 \}$
	\item[] $I^{--}_{(w,b,d)} = \{ i \in I_u | 1-d_i(wx_i + b) \leq 0, wx_i + b \leq 0 \}$
	\end{enumerate}
	Note that $d_i$ that minimizes error is determined by the sign of $wx_i + b$ by Lemma \ref{lemma_sssvm_opt_sol_property}. This means that $I^{++}_{(w,b,d)}, I^{+-}_{(w,b,d)},I^{-+}_{(w,b,d)}$, and $I^{--}_{(w,b,d)}$ can be defined without $d$.
\end{enumerate}

Next we present the declustering criteria. Let $(w^*,b^*,\xi^*, d^*)$ and $(\bar{w}^t,\bar{b}^t,\bar{\xi}^t,\bar{d}^t)$ be optimal solutions to \eqref{def_SSSVM} and \eqref{def_SSSVM_agg}, respectively. Given $C^t$ and $(\bar{w}^t,\bar{b}^t,\bar{\xi}^t,\bar{d}^t)$, we define the clusters for iteration $t+1$ as follows. 
\begin{enumerate}[noitemsep]
\item[] \textsf{Step 1} $C^{t+1} \gets \emptyset$, $D^{t+1} \gets \emptyset$
\item[] \textsf{Step 2} For each  $k \in K_l^t$
	\begin{enumerate}[noitemsep]
	\item[] \textsf{Step 2(a)} If $1-y_i (\bar{w}^t x_i + \bar{b}^t) \leq 0$ for all $i \in C_k^t$, or if $1-y_i (\bar{w}^t x_i + \bar{b}^t) \geq 0$ for all $i \in C_k^t$, then $C^{t+1} \gets C^{t+1} \cup \{ C_k^t \}$ 	
	\item[] \textsf{Step 2(b)} Otherwise, first, decluster $C_k^t$ into two clusters: $C_{k+}^{t} = \{ i \in C_k^t | 1-y_i (\bar{w}^t x_i + \bar{b}^t) > 0 \}$ and $C_{k-}^{t} = \{ i \in C_k^t | 1-y_i (\bar{w}^t x_i + \bar{b}^t) \leq 0 \}$. Next, $C^{t+1} \gets C^{t+1} \cup \{ C_{k+}^{t}, C_{k-}^{t} \}$.
	\end{enumerate}
\item[] \textsf{Step 3} For each $k \in K_u^t$, 
	\begin{enumerate}[noitemsep]
	\item[] \textsf{Step 3(a)} Partition $D_k^t$ into four sub-clusters.
		\begin{enumerate}
		\item[] $D_{k++}^{t} = \{ i \in D_k^t | 1-\bar{d}_k^t (\bar{w}^t x_i + \bar{b}^t) > 0, \bar{w}^t x_i + \bar{b}^t > 0 \} = \{ i \in D_k^t \cap I^{++}_{(\bar{w}^t, \bar{b}^t, \bar{d}^t)} \}$
		\item[] $D_{k+-}^{t} = \{ i \in D_k^t | 1-\bar{d}_k^t (\bar{w}^t x_i + \bar{b}^t) > 0, \bar{w}^t x_i + \bar{b}^t \leq 0 \}= \{ i \in D_k^t \cap I^{+-}_{(\bar{w}^t, \bar{b}^t, \bar{d}^t)} \}$			
		\item[] $D_{k-+}^{t} = \{ i \in D_k^t | 1-\bar{d}_k^t (\bar{w}^t x_i + \bar{b}^t) \leq 0, \bar{w}^t x_i + \bar{b}^t > 0 \} = \{ i \in D_k^t \cap I^{-+}_{(\bar{w}^t, \bar{b}^t, \bar{d}^t)} \}$
		\item[] $D_{k--}^{t} = \{ i \in D_k^t | 1-\bar{d}_k^t (\bar{w}^t x_i + \bar{b}^t) \leq 0, \bar{w}^t x_i + \bar{b}^t \leq 0 \} = \{ i \in D_k^t \cap I^{--}_{(\bar{w}^t, \bar{b}^t, \bar{d}^t)} \}$
		\end{enumerate}
	\item[] \textsf{Step 3(b)} If one of $D_{k++}^{t}$, $D_{k+-}^{t}$, $D_{k-+}$, and $D_{k--}^{t}$ equals to $D_k^t$, then $D^{t+1} \gets D^{t+1} \cup \{ D_{k}^{t} \}$. Otherwise, we set $D^{t+1} \gets D^{t+1} \cup \{ D_{k++}^{t}, D_{k+-}^{t},$ $D_{k-+},D_{k--}^{t} \}$. Note that any of $D_{k++}^{t}$, $D_{k+-}^{t}$, $D_{k-+}^{t}$, or $D_{k--}^{t}$ can be empty.
	\end{enumerate}
\end{enumerate}

In Figure \ref{fig:s3vm}, we illustrate AID for S$^{3}$VM. As labeled observations follow the illustration in Figure \ref{fig:svm}, we only illustrate unlabeled observations. In Figure \ref{fig:s3vm_1}, the small white circles are the original entries and the black circles are the aggregated entries. The plain line represents the separating hyperplane $(\bar{w}^t,\bar{b}^t)$ obtained from an optimal solution to \eqref{def_SSSVM_agg}, where the margins are implied by the dotted lines. The original and aggregated observations have been assigned to either + or $-$ (1 or -1, respectively): the labels of aggregated entries are from the optimal solution of the aggregated problem, the labels of the original entries are based on $(\bar{w}^t,\bar{b}^t)$ and Lemma \ref{lemma_sssvm_opt_sol_property}. Observe that two clusters (gray large circles) violate the optimality conditions. In Figure \ref{fig:s3vm_2}, one of the two violating clusters is partitioned into four subclusters: (\rmnum{1}) entries with + labels and under the zero error boundary, (\rmnum{2}) entries with $-$ labels and under the zero error boundary, (\rmnum{3}) entries with + labels and above the zero error boundary, and (\rmnum{4}) entries with $-$ labels and above the zero error boundary. The other cluster is partitioned into two subclusters. Based on the declustering criteria, the two clusters are declustered and we obtain new clusters in Figure \ref{fig:s3vm_3}. Note that new labels will be decided after solving \eqref{def_SSSVM_agg} with the new aggregated data.

\begin{figure}[ht]
     \begin{center}
        \subfigure[Clusters $D^t$ and $(\bar{w}^t,\bar{b}^t)$]{%
           \includegraphics[scale=0.17]{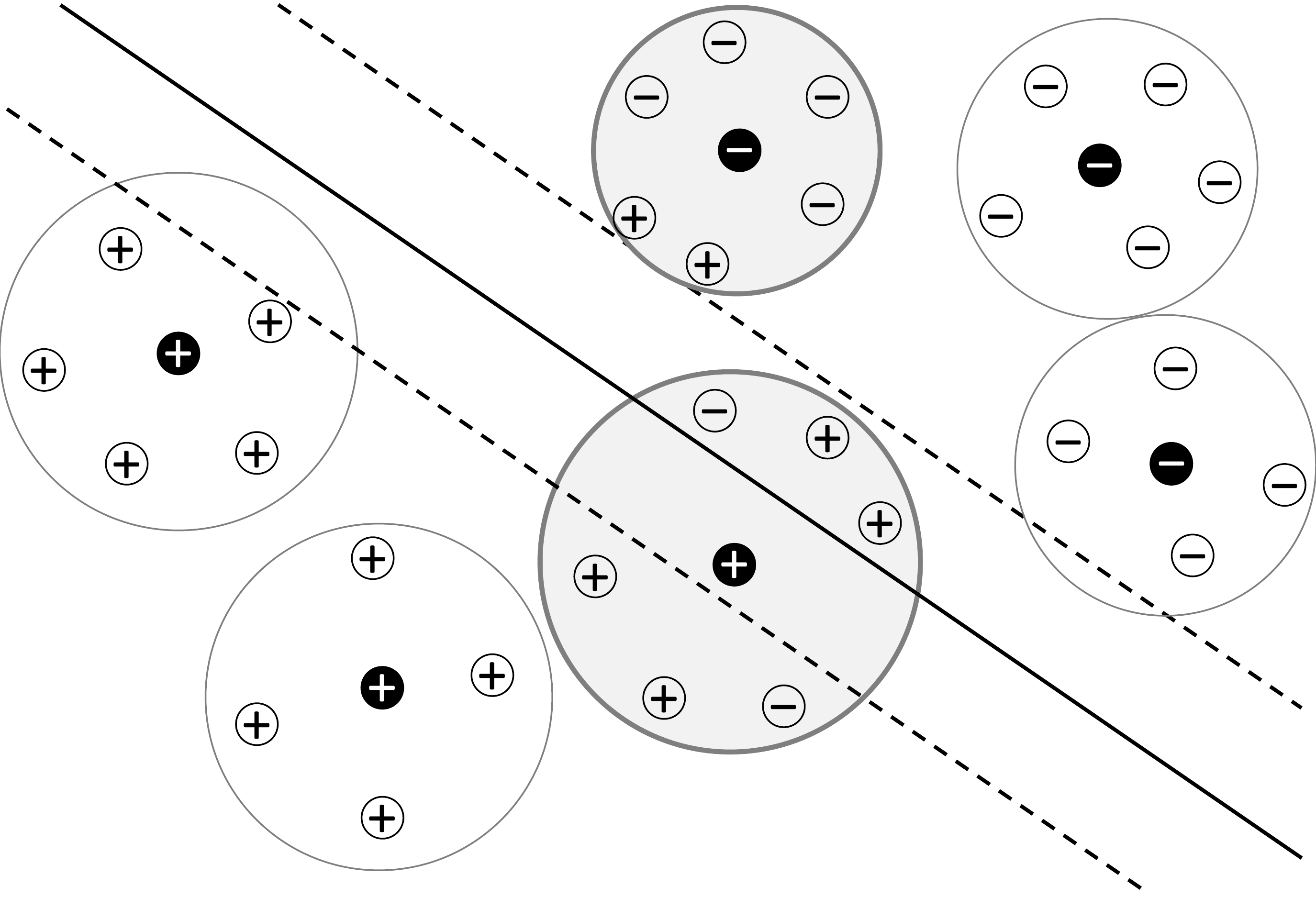} \label{fig:s3vm_1}
        }\qquad
        \subfigure[Declustered]{%
           \includegraphics[scale=0.17]{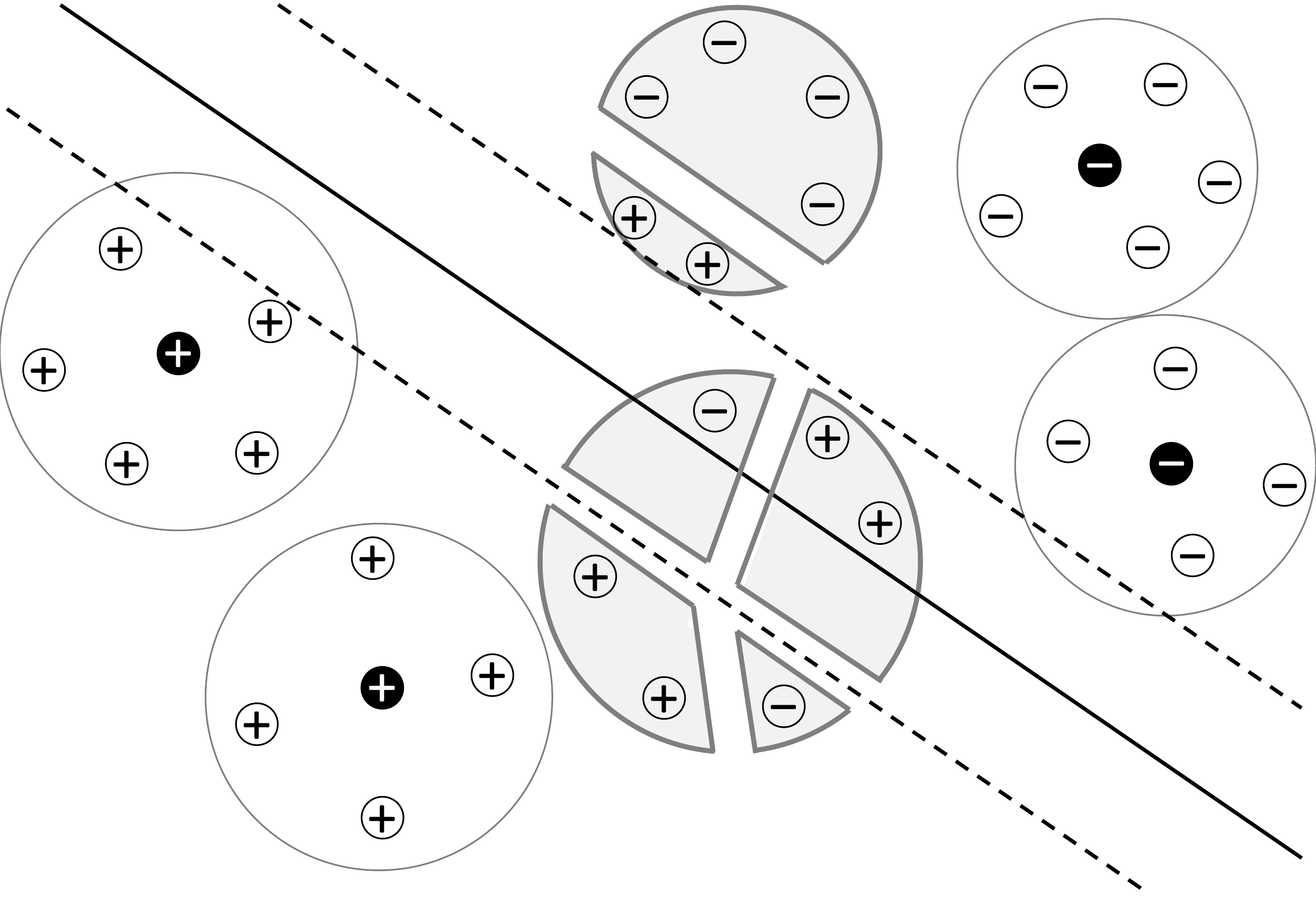} \label{fig:s3vm_2}
        }\qquad
        \subfigure[New clusters $D^{t+1}$]{%
           \includegraphics[scale=0.17]{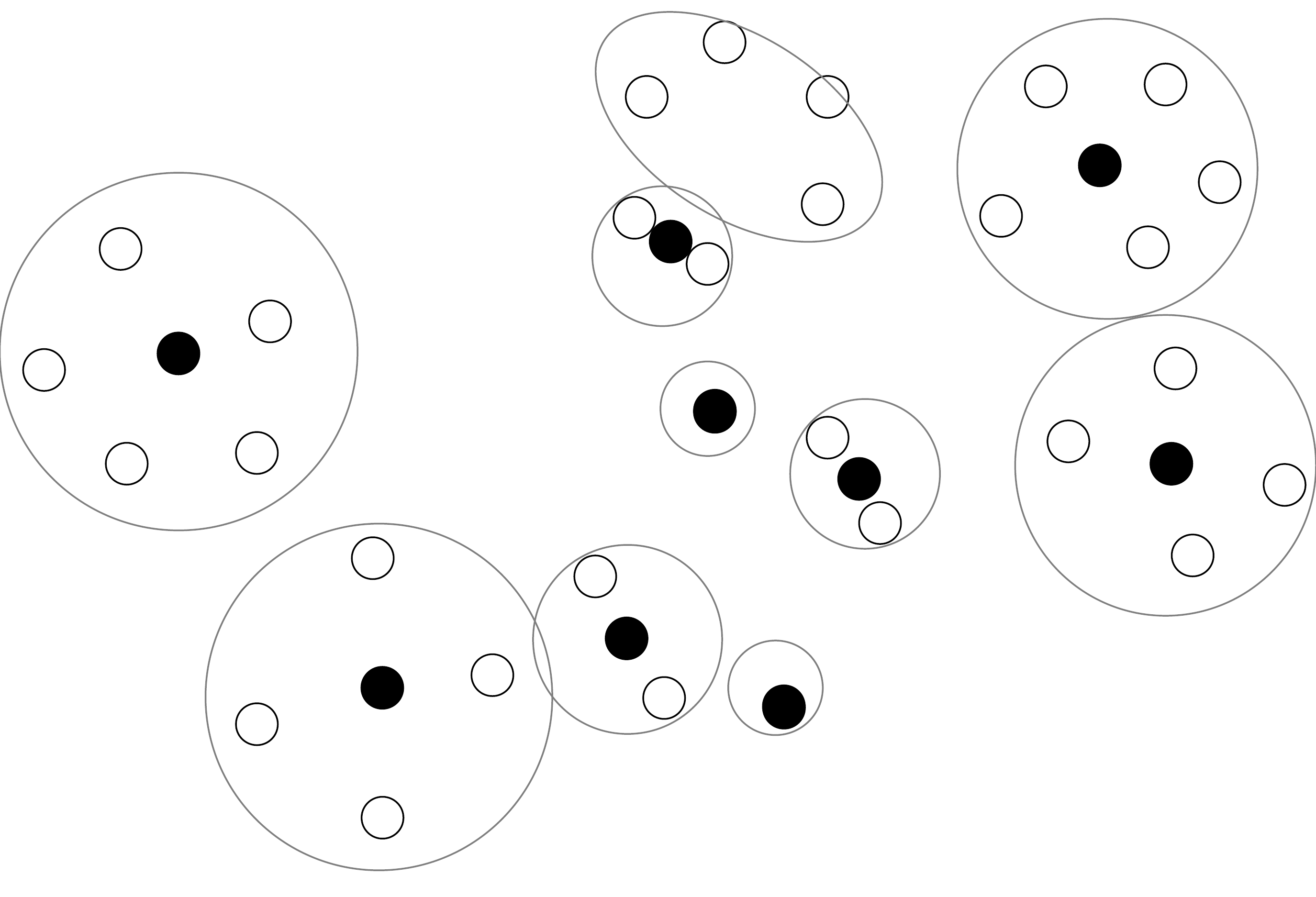} \label{fig:s3vm_3}
        }
    \end{center}
    \vspace{-0.5cm}
    \caption{Illustration of AID for S$^{3}$VM}
    \label{fig:s3vm}  
\end{figure}

Note that a feasible solution to \eqref{def_SSSVM} does not have the same dimension as a feasible solution to \eqref{def_SSSVM_agg}. In order to analyze the algorithm, we convert a feasible solution to \eqref{def_SSSVM_agg} to a feasible solution to \eqref{def_SSSVM}. Given a feasible solution $(w^t,b^t,\xi^t,d^t) \in (\mathbb{R}^m, \mathbb{R}, \mathbb{R}^{|K_l^t|+|K_u^t|}, \mathbb{R}^{|K_u^t}|)$ to \eqref{def_SSSVM_agg}, we define a feasible solution $(w,b,\xi,d) \in (\mathbb{R}^m, \mathbb{R}, \mathbb{R}^n, \mathbb{R}^{|I_u|})$ to \eqref{def_SSSVM} as follows.
\begin{enumerate}[noitemsep]
	\item[] $w := w^t$, $b := b^t$
	\item[] $d_i = \left \{
	\begin{array}{ll}
		1,  &  \mbox{if } w^tx_i + b^t < 0,\\ 
		-1, &  \mbox{if } w^tx_i + b^t \geq 0,
	\end{array}
\right. \qquad \mbox{for } i \in D_k^t$ and $k \in K_u^t,$
	\item[] $\xi_i := \max \{ 0, 1-y_i(w^t x_i + b^t) \}$ for $i \in I_l$
	\item[] $\xi_i := \max \{ 0, 1-d_i^t(w^t x_i + b^t) \}$ for $i \in I_u$
	\end{enumerate}

Using the above procedure, we map an optimal solution $(\bar{w}^t,\bar{b}^t,\bar{\xi}^t,\bar{d}^t)$ to \eqref{def_SSSVM_agg} to a feasible solution $(\hat{w}^t,\hat{b}^t,\hat{\xi}^t,\hat{d}^t)$ to \eqref{def_SSSVM}. The objective function value of $(\hat{w}^t,\hat{b}^t,\hat{\xi}^t,\hat{d}^t)$ is evaluated by
\begin{equation}
E^t = \frac{1}{2} \|\hat{w}^t\|^2 + M_l \sum_{i \in I_l} \hat{\xi}^t_i + M_u \sum_{i \in I_u} \hat{\xi}^t_i.
\end{equation}

We next explain the optimality condition and show its correctness. Let $C^g = \{ C_1^g,C_2^g, \cdots, C_{|K_l^g|}^g \}$ and $D^g = \{ D_1^g,D_2^g, \cdots, D_{|K_u^g|}^g \}$ be arbitrary clusters of labeled and unlabeled data where $K_l^g = \{1,2,\cdots,|K_l^g| \}$ and $K_u^g = \{1,2,\cdots,|K_u^g| \}$ are the associated index sets of clusters, respectively. Let us consider the following optimization problem.

\vspace{-0.5cm}
\begin{equation}
\label{def_sssvm_ext}
\begin{array}{llll}
\displaystyle G^* = & \displaystyle \min_{\substack{w,b,d,\\C^g,D^g}}& \multicolumn{2}{l}{\displaystyle \frac{1}{2} \| w\|^2 + M_l \sum_{i \in I_l} \max \{ 0, 1-y_i(w x_i + b)\} + M_u \sum_{i \in I_u} \max \{ 0, 1-d_i(w x_i + b)\}} \\
& \quad s.t.  &  C_k^g \subseteq I^{+}_{(w,b)} \mbox{ or } C_k^g \subseteq I^{-}_{(w,b)}, &   k \in K_l^g, \\
	& & D_k^g \subseteq I^{++}_{(w,b,d)} \mbox{ or } D_k^g \subseteq I^{+-}_{(w,b,d)} \mbox{ or } D_k^g \subseteq I^{-+}_{(w,b,d)} \mbox{ or } D_k^g \subseteq I^{--}_{(w,b,d)}, & k \in K_u^g,
\end{array}
\end{equation}

\noindent where $w \in \mathbb{R}^{m}$, $b \in \mathbb{R}$, $d \in \{-1,1\}^{|I_u|}$, and $C^g$ and $D^g$ are the cluster decision sets. In fact, compare to \eqref{definition_sssvm}, \eqref{def_sssvm_ext} has additional constraints and clustering decision to make. Observe that given an optimal solution to \eqref{definition_sssvm}, we can easily find $C^g$ and $D^g$ satisfying the constraints in \eqref{def_sssvm_ext} by simply classifying each observation. Hence, it is trivial to see that 
\begin{equation}
\label{eqn_estart_equals_gstar}
E^* = G^*.
\end{equation}
For the analysis, we will use $G^*$ and \eqref{def_sssvm_ext} instead of $E^*$ and \eqref{definition_sssvm}, respectively.

Next, let us consider the following aggregated problem.

\vspace{-0.5cm}
\begin{equation}
\label{def_sssvm_ext_agg}
\begin{array}{llll}
H^*  =  & \displaystyle \min_{\substack{w,b,d,\\C^g,D^g}} & \multicolumn{2}{l}{\displaystyle \frac{1}{2} \| w\|^2 + M_l \sum_{k \in K_l^g} |C_k^g| \max \{ 0, 1-y_k(w x_k + b)\} + M_u \sum_{k \in K_u^g} |D_k^g| \max \{ 0, 1-d_k(w x_k + b)\}} \\
 & \quad s.t. &  C_k^g \subseteq I^{+}_{(w,b)} \mbox{ or } C_k^g \subseteq I^{-}_{(w,b)}, & k \in K_l^g,  \\
	& & D_k^g \subseteq I^{++}_{(w,b,d)} \mbox{ or } D_k^g \subseteq I^{+-}_{(w,b,d)} \mbox{ or } D_k^g \subseteq I^{-+}_{(w,b,d)} \mbox{ or } D_k^g \subseteq I^{--}_{(w,b,d)}, & k \in K_u^g, \\
& & x_k = \frac{\sum_{i \in C_k^g} x_i}{|C_k^g|}, y_k = \frac{\sum_{i \in C_k^g} y_i}{|C_k^g|}, &  k \in K_l^g, \\
& & x_k = \frac{\sum_{i \in D_k^g} x_i}{|D_k^g|}, & k \in K_u^g,
\end{array}
\end{equation}

\noindent where $w \in \mathbb{R}^{m}$, $b \in \mathbb{R}$, and $d \in \{-1,1\}^{|K_u^g|}$, and $C^g$ and $D^g$ are the cluster decision sets. Note that $x_k$ and $y_k$ are now decision variables that depend on $C_k^g$ and $D_k^g$. Note that due to characteristic of the subsets $I^{++}_{(w,b,d)}, I^{+-}_{(w,b,d)},I^{-+}_{(w,b,d)}$, and $I^{--}_{(w,b,d)}$, we can replace $d_i$ by $d_k$ in the definition of the subsets.

\begin{lemma}
\label{lemma_equiv_sssvm_ext_agg}
There is a one-to-one correspondence between feasible solutions of \eqref{def_sssvm_ext} and \eqref{def_sssvm_ext_agg} which preserves the objective function value.
\end{lemma}
\begin{proof}
For $k \in K_l^g$, we derive
\vspace{0.2cm}

\begin{tabular}{lll}
$\sum_{i \in C_k^g} \max \{ 0, 1-y_i(wx_i + b) \}$ & $=$ & $\sum_{i \in C_k^g} \max \{ 0, 1-y_k(wx_i + b) \}$\\
	& $=$ & $ \max \{ 0, \sum_{i \in C_k^g} \big( 1-y_k(wx_i + b) \big) \}$\\
\end{tabular}

\begin{tabular}{lll}
\textcolor{white}{$\sum_{i \in C_k^g} \max \{ 0, 1-y_i(wx_i + b) \}$}	& $=$ & $ \max \{ 0, |C_k^g| -|C_k^g| y_k w \Big( \frac{\sum_{i \in C_k^g} x_i}{|C_k^g|} \Big)  - y_k |C_k^g| b  \}$\\
	& $=$ & $|C_k^g|  \max \{ 0,  1 - y_k (w x_k + b)  \}$,
\end{tabular}
\vspace{0.3cm}

\noindent where the first line holds since all $i$ in $C_k^g$ have the same label by the initial clustering, the second line holds since $C_k^g \subseteq I^{+}_{(w,b)}$ or $ C_k^g \subseteq I^{-}_{(w,b)}$ for any $k \in K_l^g$, and the fourth line follows from the constraint in \eqref{def_sssvm_ext_agg}.

For $k \in K_u^g$, we derive
\vspace{0.2cm}

\begin{tabular}{lll}
$\sum_{i \in D_k^g} \max \{ 0, 1-d_i(wx_i + b) \}$ & $=$ & $\sum_{i \in D_k^g} \max \{ 0, 1-d_k(wx_i + b) \}$\\
	& $=$ & $ \max \{ 0, \sum_{i \in D_k^g} \big( 1-d_k(wx_i + b) \big) \}$\\
	& $=$ & $ \max \{ 0, |D_k^g| -|D_k^g| d_k w \Big( \frac{\sum_{i \in D_k^g} x_i}{|D_k^g|} \Big)  - d_k |D_k^g| b  \}$\\
	& $=$ & $|D_k^g|  \max \{ 0,  1 - d_k (w x_k + b)  \}$,
\end{tabular}
\vspace{0.3cm}

\noindent where the first and second lines hold since $D_k^g \subseteq I^{++}_{(w,b,d)}$ or $D_k^g \subseteq I^{+-}_{(w,b,d)}$ or $D_k^g \subseteq I^{-+}_{(w,b,d)}$ or $D_k^g \subseteq I^{--}_{(w,b,d)}$ for any $k \in K_u^g$.

Observe that the above two results can be shown in the reverse order. Hence, it is easy to see that there is a one-to-one correspondence between feasible solutions of \eqref{def_sssvm_ext} and \eqref{def_sssvm_ext_agg}, and the objective function values of the corresponding feasible solutions are the same. 
\end{proof}

Note that Lemma \ref{lemma_equiv_sssvm_ext_agg} implies that, for an optimal solution of \eqref{def_sssvm_ext_agg}, the corresponding solution for \eqref{def_sssvm_ext} is an optimal solution for \eqref{def_sssvm_ext}. This gives the following corollary.
\begin{corollary}
We have $G^* = H^*$.
\end{corollary}

In Proposition \ref{proposition_s3vm_optimality}, we present the optimality condition.

\begin{proposition}
\label{proposition_s3vm_optimality}
Let us assume that
\begin{enumerate}
\item for all $k \in K_l^t$, (\rmnum{1}) $1-y_i (\bar{w}^t x_i + \bar{b}^t) \leq 0$ for all $i \in C_k^t$ or (\rmnum{2}) $1-y_i (\bar{w}^t x_i + \bar{b}^t) \geq 0$ for all $i \in C_k^t$
\item for all $k \in K_u^t$, exactly one of the following holds.
	\begin{enumerate}[noitemsep]
	\item[] (\rmnum{1}) $1-\hat{d}_i^t (\bar{w}^t x_i + \bar{b}^t) \leq 0$ and $\bar{w}^t x_i + \bar{b}^t \leq 0$ for all $i \in D_k^t$
	\item[] (\rmnum{2}) $1-\hat{d}_i^t (\bar{w}^t x_i + \bar{b}^t) \leq 0$ and $\bar{w}^t x_i + \bar{b}^t > 0$ for all $i \in D_k^t$
	\item[] (\rmnum{3}) $1-\hat{d}_i^t (\bar{w}^t x_i + \bar{b}^t) \geq 0$ and $\bar{w}^t x_i + \bar{b}^t > 0$ for all $i \in D_k^t$
	\item[] (\rmnum{4}) $1-\hat{d}_i^t (\bar{w}^t x_i + \bar{b}^t) \geq 0$ and $\bar{w}^t x_i + \bar{b}^t \leq 0$ for all $i \in D_k^t$
	\end{enumerate} 
\end{enumerate}
Then, $(\hat{w}^t,\hat{b}^t,\hat{\xi}^t,\hat{d}^t)$ is an optimal solution to \eqref{def_SSSVM}. In other words, if (\rmnum{1}) all observations in $C_k^t$ and $D_k^t$ are on the same side of the margin-shifted hyperplane of the separating hyperplane $(\bar{w}^t,\bar{b}^t)$ and (\rmnum{2}) all observations in $D_k^t$ have the same label, then $(\hat{w}^t,\hat{b}^t,\hat{\xi}^t,\hat{d}^t)$ is an optimal solution to \eqref{def_SSSVM}.
\end{proposition}

\begin{proof}
Observe that the conditions stated match with the definition of $I^{+}_{(w,b)}$, $I^{-}_{(w,b)}$, $I^{++}_{(w,b,d)}$, $I^{+-}_{(w,b,d)}$, $I^{-+}_{(w,b,d)}$, and $I^{--}_{(w,b,d)}$. Hence, $C^t$ and $D^t$ satisfy the constraints of \eqref{def_sssvm_ext_agg}, which implies that $(\bar{w}^t,\bar{b}^t,\bar{\xi}^t,\bar{d}^t)$ is an optimal solution to \eqref{def_sssvm_ext_agg}. By Lemma \ref{lemma_equiv_sssvm_ext_agg}, $(\hat{w}^t,\hat{b}^t,\hat{\xi}^t,\hat{d}^t)$ is an optimal solution to \eqref{def_sssvm_ext}. Finally, since $E^* = G^*$ by \eqref{eqn_estart_equals_gstar}, we conclude that $(\hat{w}^t,\hat{b}^t,\hat{\xi}^t,\hat{d}^t)$ is an optimal solution to \eqref{def_SSSVM}.
\end{proof}

Observe that, unlike LAD and SVM, we do not have the non-decreasing property of $F_t$. Due to binary variable $d_i$, the non-decreasing property of $F_t$ no longer holds.

\section{Computational Experiments}
\label{section_compuatation}
All experiments were performed on Intel Xeon X5660 2.80 GHz dual core server with 32 GB RAM, running Windows Server 2008 64 bit. We implemented AID for LAD and SVM in scripts of R statistics \cite{Rstat} and Python, respectively, and AID for S$^{3}$VM is implemented in C\# with CPLEX.

For LAD, R statistics package \textsf{quantreg} \cite{R-quantreg} is used to solve \eqref{formulation_standard_regression} and \eqref{formulation_weighted_regression_iter_t}. In detail, function \textsf{rq()} is used with the Frisch-Newton interior point method (\textsf{fn}) option. Due to the absence of large-scale real world instances, we randomly generate three sets of LAD instances.
\begin{enumerate}[noitemsep]
\item[] Set A: $n \in \{2, 4, 8 , 16  \} \times 10^5$ and $m \in \{ 10,100,500,800\}$, where $(n,m)=(16 \times 10^5, 800)$ is excluded due to memory issues
\item[] Set B: $n=10^6$ and $m \in \{50,100,150,200,250,300,350,400,450,500\}$
\item[] Set C: $n \in \{4,6,8,10,12,14,16\} \times 10^5$ and $m \in \{50,500\}$
\end{enumerate}
Set A is used for the experiment checking the performance of AID over various pairs of $n$ and $m$, whereas Sets B and C are used for checking the performance of AID over fixed $n$ and $m$, respectively.

For SVM, Python package \textsf{scikit-learn} \cite{scikit-learn} is used to solve \eqref{def_SVM} and \eqref{def_SVM2}. In detail, functions \textsf{svc()} and \textsf{linearSVC()} are used, where the implementations are based on \textsf{libsvm} \cite{CC01a} and \textsf{liblinear} \cite{liblinear}, respectively. We use two benchmark algorithms because \textsf{libsvm} is one of the most popular and widely used implementation, while \textsf{liblinear} is known to be faster for SVM with the linear kernel. For SVM, we generate two sets of instances by sampling from two large data sets: (\rmnum{1}) 1.3 million observations and 342 attributes obtained from a real world application provided by IBM (\rmnum{2}) rcv1.binary data set with 677,399 observations and 47,236 attributes from \cite{CC01a}. We denote them as IBM and RCV, respectively.
\begin{enumerate}[noitemsep]
\item[] Set 1: $n \in \{3 , 5 , 10 , 15 \} \times 10^4$ and $m \in \{ 10,30,50,70,90\}$ from IBM and RCV
\item[] Set 2: $n \in \{2,4,6,8 \} \times 10^5$ and $m \in \{ 10,30,50,70,90\}$ from IBM
\end{enumerate}
This generation procedure enables us to analyze performances of the algorithms for data set with similar characteristics and various sizes. For each $(n,m)$ pair for SVM, we generate ten instances and present the average performance of the ten instances for each $(n,m)$ pair. Note that Set 1 instances are smaller than Set 2 instances. We use Set 1 to test AID against \textsf{libsvm} and Set 2 to test AID against \textsf{liblinear}, because \textsf{liblinear} is faster and capable of solving larger problems for SVM with the linear kernel. Recall that AID is capable of using any solver for the aggregated problem. Hence, when testing against \textsf{libsvm}, AID uses \textsf{libsvm} for solving the aggregated problems. Similarly, when testing against \textsf{liblinear}, AID uses \textsf{liblinear}.

For S$^{3}$VM, aggregated problems \eqref{def_SSSVM_agg} are solved by CPLEX. We set the 1,800 seconds time limit for the entire algorithm. We consider eight semi-supervised learning benchmark data sets from \cite{Chapelle-book}. Table \ref{tab:s3vm_book_data} lists the characteristics of the data sets. For each data set, two sets of twelve data splits are given: one with 10 and the other with 100 labeled observations for each split set for the first seven data sets, whereas 1,000 and 10,000 labeled observations are used for each split set for \textit{SecStr}.

\begin{table}[htbp]
\scriptsize
  \centering
    \begin{tabular}{|c|r|r|c|}
    \hline
    Data  & $n$ (entries)     & $m$ (attributes)     & $l$ (labeled) \\ \hline
    Digit1 & 1,500  & 241   & 10 and 100 \\
    USPS  & 1,500  & 241   & 10 and 100 \\
    BCI   & 400   & 117   & 10 and 100 \\
    g241c & 1,500  & 241   & 10 and 100 \\
    g241d & 1,500  & 241   & 10 and 100 \\
    Text  & 1,500  & 11,960 & 10 and 100 \\ 
    COIL2 & 1,500 & 241 & 10 and 100\\
    SecStr & 83,679 & 315 & 1,000 and 10,000\\ \hline
    \end{tabular}%
      \caption{S$^{3}$VM Data from Chapelle \cite{Chapelle-book}}
  \label{tab:s3vm_book_data}%
\end{table}%

In all experiments, AID iterates until the corresponding optimality condition is satisfied or the optimality gap is small (we set $10^{-3}$ and $10^{-4}$ for LAD and SVM, respectively). The tolerance for \textsf{fn} and \textsf{libsvm} are set to $10^{-3}$ according to the definition of the corresponding packages. The default tolerance for \textsf{liblinear} is $10^{-3}$ with the maximum of 1,000 iterations. Because we observed early terminations and inconsistencies with the default setting, we use the maximum of 100,000 iterations for \textsf{liblinear}. For S$^{3}$VM, we run AID for one and five iterations and terminate before we reach optimum. See Section \ref{section_exp_s3vm} for detail. 

For the initial clustering methods for LAD, SVM, and S$^{3}$VM, we do not fully rely on standard clustering algorithms, as it takes extensive amount of time to optimize a clustering objective function due to large data size. Instead, for the LAD initial clustering method, we first sample a small number of original entries, build regression model, and obtain a solution $\beta^{\mbox{\scriptsize{init}}}$. Let $r \in \mathbb{R}^n$ be the residual of the original data defined by $\beta^{\mbox{\scriptsize{init}}}$. We use one iteration of k-means based on data $(r,y) \in \mathbb{R}^{n \times 2}$ to create $C^0$. For the SVM initial clustering method, we sample a small number of original entries and find a hyperplane $(w^{init},b^{init})$. Then we use one iteration of k-means based on data $d \in \mathbb{R}^{n}$, where $d \in \mathbb{R}^n$ is the distance of the original entries to $(w^{init},b^{init})$. For the S$^{3}$VM initial clustering methods, we use one iteration of k-means based on the original data to create $C^0$. The initial clustering and aggregated data generation times are small for all methods for LAD, SVM, and S$^{3}$VM.

For AID, we need to specify the number of initial clusters, measured as the initial aggregation rate. User parameter $r^0$ $=$ $\frac{|K^0|}{n}$ is the initial aggregation rate given as a parameter. It defines the number of initial clusters. We generalize the notion of the aggregation rate with $r^t$, the aggregation rate in iteration t. The number of clusters is important because too many clusters lead to a large problem and too few clusters lead to a meaningless aggregated problem. Note that $r^0$ also should be problem specific. For example, we must have $|K^0| > m$ for LAD. Hence, we set $r^0 = \max\{\frac{3m}{n},0.0005\}$ if $m \times n > 5 \times 10^8$, $r^0 = \max\{\frac{2m}{n},0.0005\}$ otherwise. This means that $|K^0|$ is at least two or three times larger than $m$ and $|K^0|$ is at least 0.05\% of $n$. For SVM, we set $r^0 = \max\{\frac{1.1m}{n},0.0001\}$ for all instances. However, since S$^3$VM instances are not extremely large in our experiment, we fix it to some constant. For the S$^3$VM instances, we test both of $r^0 = 0.01$ and $0.05$ if $n \leq 10,000$, and both of $r^0 = 0.0001$ and $0.0005$ otherwise, where the number of clusters must be at least 10 to avoid a meaningless aggregated problem.

In order to compare the performance, execution times (in seconds) $\mathcal{T}^{\mbox{\scriptsize{AID}}}$, $\mathcal{T}^{\mbox{\scriptsize{\textsf{fn}}}}$, $\mathcal{T}^{\mbox{\scriptsize{\textsf{libsvm}}}}$, and $\mathcal{T}^{\mbox{\scriptsize{\textsf{liblinear}}}}$ for AID, \textsf{fn}, \textsf{libsvm}, and \textsf{liblinear}, respectively, are considered. For SVM, standard deviations $\sigma(\mathcal{T}^{\mbox{\scriptsize{AID}}})$, $\sigma(\mathcal{T}^{\mbox{\scriptsize{\textsf{libsvm}}}})$, and $\sigma(\mathcal{T}^{\mbox{\scriptsize{\textsf{liblinear}}}})$ are used to describe the stability of the algorithms. In order to further describe the performance of AID, we use the following measures.
\begin{center}
\begin{tabular}{llp{12cm}}
$r^T$ & $=$ & $\frac{|K^T|}{n}$ is the final aggregation rate at the termination\\[0.2cm]
$\rho$ & $=$ & $\frac{\mathcal{T}^{\mbox{\tiny{AID}}}}{\mathcal{T}^{\mbox{\scriptsize{\textsf{fn}}}}}$ or $\frac{\mathcal{T}^{\mbox{\tiny{AID}}}}{\mathcal{T}^{\mbox{\scriptsize{\textsf{libsvm}}}}}$ or $\frac{\mathcal{T}^{\mbox{\tiny{AID}}}}{\mathcal{T}^{\mbox{\scriptsize{\textsf{liblinear}}}}}$\\
$\Delta$ & $=$ & $\frac{E^{\mbox{\tiny{AID}}} - E^{\mbox{\scriptsize{\textsf{fn}}}}}{E^{\mbox{\scriptsize{\textsf{fn}}}}}$ or $\frac{E^{\mbox{\tiny{AID}}} - E^{\mbox{\scriptsize{\textsf{libsvm}}}}}{E^{\mbox{\scriptsize{\textsf{libsvm}}}}}$ or $\frac{E^{\mbox{\tiny{AID}}} - E^{\mbox{\scriptsize{\textsf{liblinear}}}}}{E^{\mbox{\scriptsize{\textsf{liblinear}}}}}$\\
$\Gamma$ & $=$ & training classification rate of AID $-$ training classification rate of benchmark
\end{tabular}
\end{center}

\noindent For example, we set $r^0 = \frac{3}{33} = \frac{1}{11}$ in Figure \ref{fig:ex_decluster_1} and terminate the algorithm with $r^T = \frac{9}{33} = \frac{3}{11}$ in Figure \ref{fig:ex_decluster_3}. Note that $\rho < 1$ indicates that AID is faster. We use $\Delta$ to check the relative difference of objective function values. For LAD, $\Delta$ is also used to check if the solution qualities are the same. For SVM, $\Gamma$ is used to measure the solution quality differences.

\subsection{Performance for LAD}

In Table \ref{tab:lad_result}, the computational performance of AID for LAD is compared against the benchmark \textsf{fn} for Set A. For many LAD instances in Table \ref{tab:lad_result}, \textsf{fn} is faster. For the smallest instance, \textsf{fn} is 34 times faster. However, ratio $\rho$ decreases in general as $n$ and $m$ increase. This implies that AID is competitive for larger size data. In fact, AID is five times faster than \textsf{fn} for the two largest LAD instances considered. The values of $\Delta$ indicate that \textsf{fn} and AID give the same quality solutions within numerical error bounds, as $\Delta$ measures the relative difference in the sum of the absolute error. The final aggregation rate $r^T$ also depends on problem size. As $n$ increases, $r^T$ decreases because original entries can be grouped into larger size clusters. As $m$ increases, $r^T$ increases because it is more difficult to cluster the original entries into larger size clusters. Further discussion on how $r^t$ changes over iterations is presented in Section \ref{section_discussion}.

\begin{table}[htbp]
\scriptsize
\centering
    \begin{tabular}{|rr|rrrr|r|r|r|}
        \hline
    \multicolumn{2}{|c|}{Instance} & \multicolumn{4}{c|}{AID}                       & \multicolumn{1}{c|}{\textsf{fn}} & \multicolumn{2}{c|}{Comparison} \\ \hline
    \multicolumn{1}{|c}{$n$} & \multicolumn{1}{c|}{$m$} & \multicolumn{1}{c}{$r^0$} & \multicolumn{1}{c}{$r^T$} & \multicolumn{1}{c}{$T$} & \multicolumn{1}{c|}{$\mathcal{T}^{\mbox{\tiny{AID}}}$} & \multicolumn{1}{c|}{$\mathcal{T}^{\mbox{\textsf{fn}}}$} & \multicolumn{1}{c|}{$\Delta$} & \multicolumn{1}{c|}{$\rho$} \\ \hline
    200,000 & 10    & 0.05\% & 2.40\% & 8     & 50    & 1      & 0.01\% & 34.59\\
    200,000 & 100   & 0.10\% & 10.00\% & 9     & 84    & 22     & 0.01\%  & 3.83\\
    200,000 & 500   & 0.50\% & 16.50\% & 7     & 451   & 393     & 0.02\% & 1.15\\
    200,000 & 800   & 0.80\% & 22.40\% & 7     & 1,347  & 1,062    & 0.01\% & 1.27\\ \hline
    400,000 & 10    & 0.05\% & 2.20\% & 8     & 99    & 3      & 0.00\% & 29.84\\
    400,000 & 100   & 0.05\% & 5.80\% & 9     & 163   & 44      & 0.01\% & 3.68 \\
    400,000 & 500   & 0.25\% & 13.30\% & 8     & 904   & 904       & 0.01\% & 1 \\
    400,000 & 800   & 0.40\% & 21.10\% & 8     & 2,689  & 2,125    & 0.01\% & 1.27 \\ \hline
    800,000 & 10    & 0.05\% & 0.90\% & 5     & 139   & 9      & 0.03\% & 15.46\\
    800,000 & 100   & 0.05\% & 5.00\% & 9     & 336   & 96      & 0.01\% & 3.51\\
    800,000 & 500   & 0.13\% & 10.30\% & 9     & 1,788  & 1,851    & 0.01\%  & \textbf{0.97} \\
    800,000 & 800   & 0.30\% & 10.30\% & 7     & 2,992  & 15,215    & 0.03\% & \textbf{0.2}\\ \hline
    1,600,000 & 10    & 0.05\% & 0.40\% & 4     & 235   & 18     & 0.06\% & 13.29\\
    1,600,000 & 100   & 0.05\% & 3.20\% & 8     & 612   & 196     & 0.01\% & 3.12\\
    1,600,000 & 500   & 0.09\% & 5.35\% & 8     & 2,460  & 12,164    & 0.01\% & \textbf{0.2} \\ \hline
    \end{tabular}%
  \caption{Performance of AID for LAD for Set A}
  \label{tab:lad_result}%
\end{table}%

In Figure \ref{fg_exp_lad}, comparisons of execution times of AID and \textsf{fn} are presented for Sets B and C. In Figure \ref{fg_exp_lad_fixed_n_time}, Set B is considered to check the performances over fixed $n = 10^6$. With fixed $n$, AID is slower when $m$ is small, but AID starts to outperform at $m = 400$. The corresponding $\rho$ values are constantly decreasing from 6.6 (when $m=50$) to 0.7 (when $m=500$). This observation also supports the results presented in Figure \ref{fg_exp_lad_fixed_m50_time} and \ref{fg_exp_lad_fixed_m500_time}, where the comparisons of Set C are presented. When $m$ is fixed to 50, the execution time of AID increases faster than \textsf{fn}. However, when $m$ is fixed to 500, AID is faster than \textsf{fn} and the execution time of AID grows slower than \textsf{fn}. Therefore, we conclude that AID for LAD is faster than \textsf{fn} when $n$ and $m$ are large and is especially beneficial when $m$ is large enough. For Figures \ref{fg_exp_lad_fixed_n_time}, \ref{fg_exp_lad_fixed_m50_time}, and \ref{fg_exp_lad_fixed_m500_time}, the corresponding number of iterations ($T$) are randomly spread over $9 \sim 11$, $10 \sim 12$, and $8 \sim 10$, respectively; we did not find a trend.

\begin{figure}[ht]
     \begin{center}
        \subfigure[Set B (fixed $n=10^6$)]{%
           \includegraphics[scale=0.2]{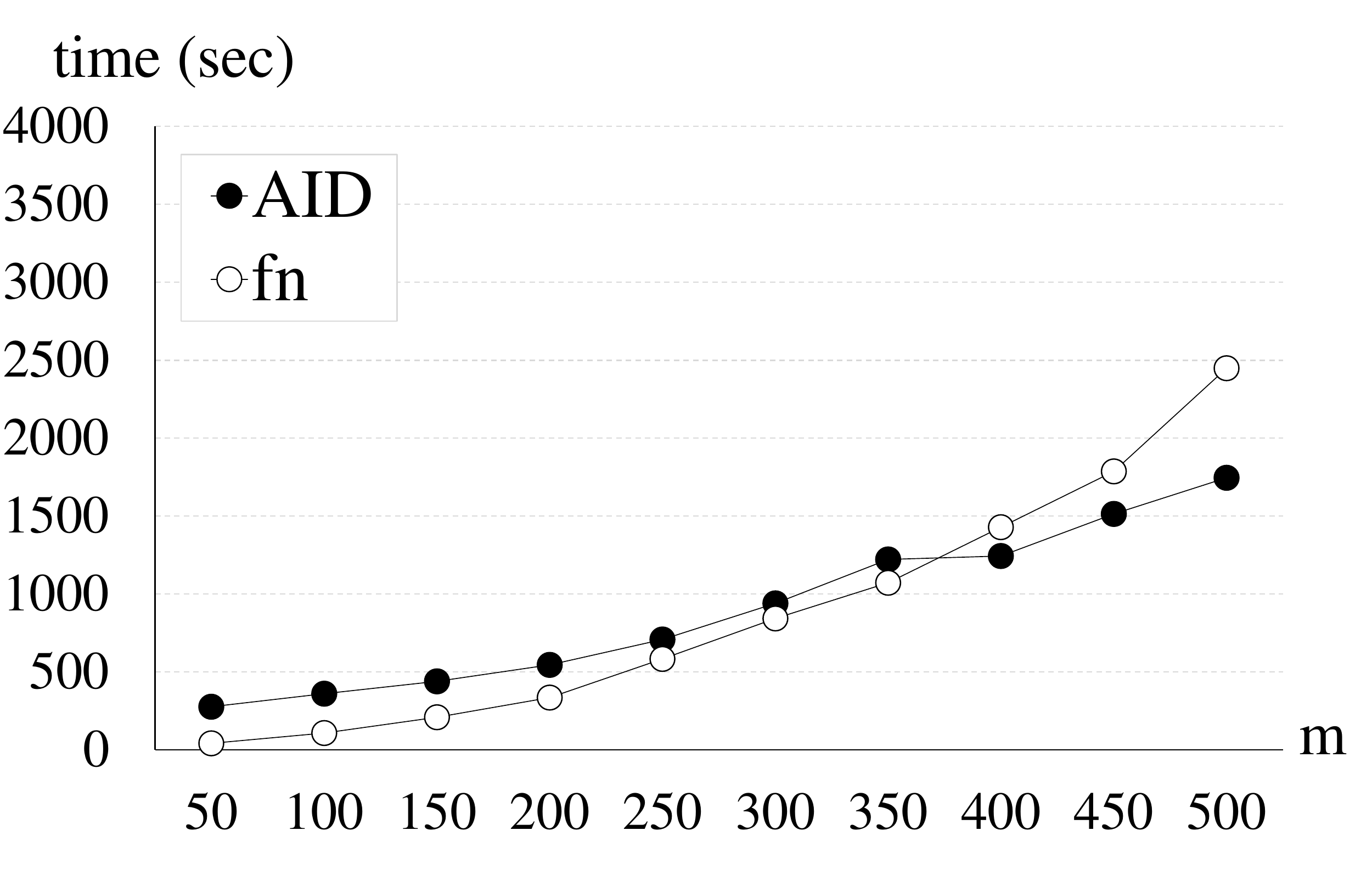} \label{fg_exp_lad_fixed_n_time}
        }
        \subfigure[Set C (fixed $m=50$)]{%
           \includegraphics[scale=0.2]{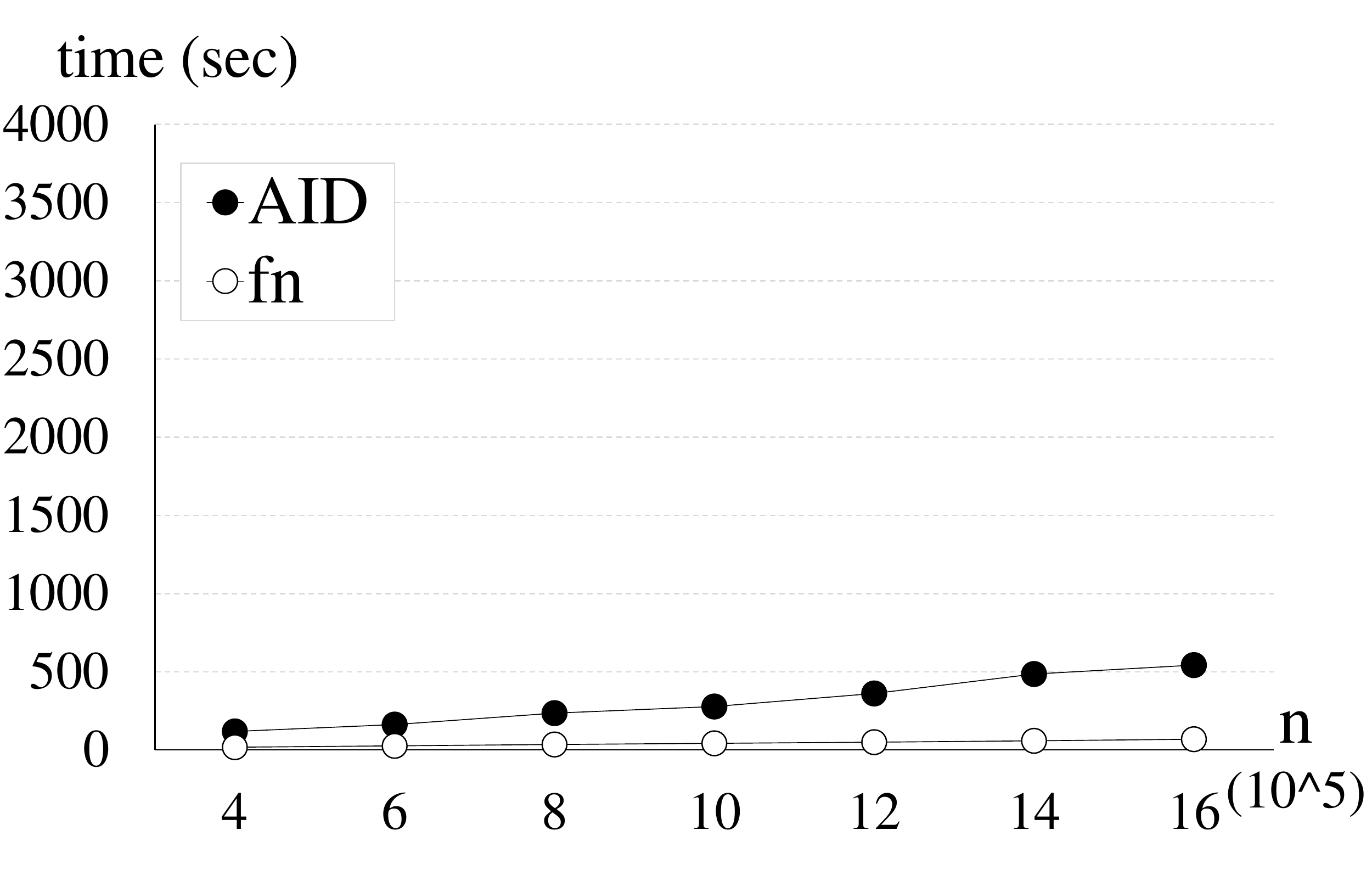} \label{fg_exp_lad_fixed_m50_time}
        }
        \subfigure[Set C (fixed $m=500$)]{%
           \includegraphics[scale=0.2]{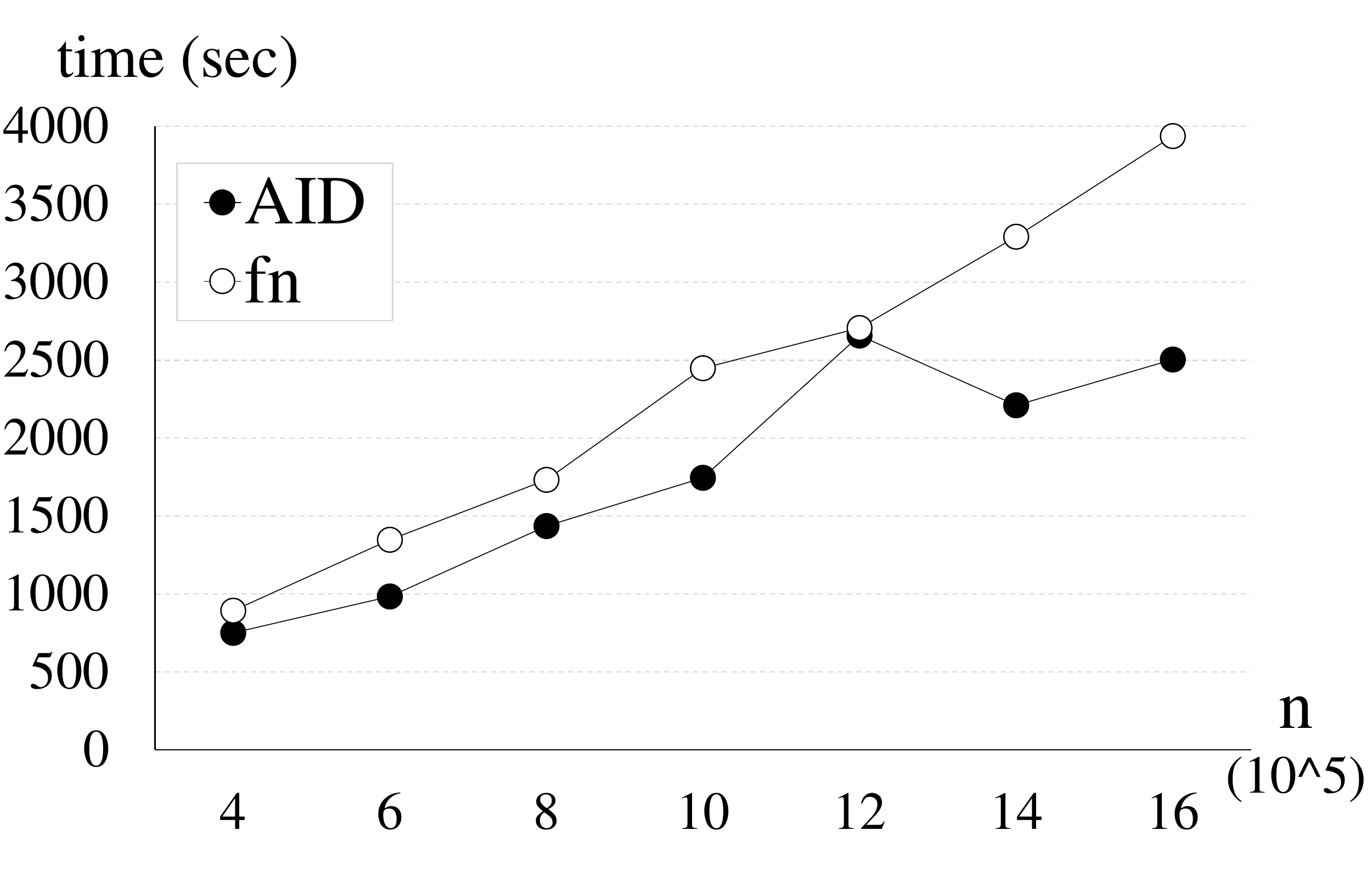} \label{fg_exp_lad_fixed_m500_time}
        }
    \end{center}
    \vspace{-0.5cm}
    \caption{Execution times of AID and \textsf{fn} for LAD}
    \label{fg_exp_lad}
\end{figure}

\subsection{Performance for SVM}

In Tables \ref{tab:svm_result_ibm} and \ref{tab:svm_result_rcv}, the computational performance of AID for SVM is compared against the benchmark \textsf{libsvm} for Set 1. In Table \ref{tab:svm_result_ibm_large}, comparison of AID and \textsf{liblinear} for Set 2 is presented. In all experiments, we fix penalty constant at $M = 0.1$. 

In Table \ref{tab:svm_result_ibm}, the result for Set 1 IBM data is presented. Observe that AID is faster than \textsf{libsvm} for all cases, as $\rho$ values are strictly less than 1 for all cases. Observe that $\rho$ values tend to decrease in $n$ and $m$. This is highly related to final aggregation rate $r^T$. Observe that, similar to the LAD result, $r^T$ decreases in $n$ and increase in $m$. As $n$ increases, it is more likely to have clusters with more original entries. This decreases $r^T$ and number of iterations $T$. It implies that we solve fewer aggregated problems and the sizes of the aggregated problems are smaller. On the other hand, as $m$ increases, $r^T$ also increases. This increases $T$ and aggregated problem sizes. However, since the complexity of \textsf{svmlib} increases faster than AID in increasing $m$, $\rho$ decreases in $m$.  Due to possibility of critical numerical errors, we also check the objective function value differences ($\Delta$) and training classification rate differences ($\Gamma$). The solution qualities of AID and \textsf{libsvm} are almost equivalent in terms of the objective function values and training classification rates.

% Table generated by Excel2LaTeX from sheet 'ibm libsvm'
\begin{table}[htbp]
  \centering
  \scriptsize
    \begin{tabular}{|rr|rrrrr|rr|rrr|}
    \hline
    \multicolumn{2}{|c|}{size} & \multicolumn{5}{c|}{AID}  & \multicolumn{2}{c|}{$\textsf{libsvm}$} & \multicolumn{3}{c|}{Comparison}\\ \hline
$n$     & $m$     & $r^0$    & $r^T$ & $T$ & $\mathcal{T}^{\mbox{\tiny{AID}}}$ & $\sigma(\mathcal{T}^{\mbox{\tiny{AID}}})$& $\mathcal{T}^{\mbox{\textsf{libsvm}}}$  & $\sigma(\mathcal{T}^{\mbox{\textsf{libsvm}}})$ & $\Gamma$ & $\Delta$ & $\rho$\\ \hline
30,000 & 10    & 0.08\% & 1.0\% & 6.1   & 3     & 0     & 10    & 2     & 0.00\% & 0.00\% & \textbf{0.26} \\
          & 30    & 0.23\% & 6.6\% & 8.5   & 4     & 0     & 12    & 1     & 0.00\% & 0.00\% & \textbf{0.34} \\
          & 50    & 0.37\% & 10.4\% & 8.4   & 6     & 1     & 19    & 3     & 0.00\% & 0.00\% & \textbf{0.30} \\
          & 70    & 0.52\% & 14.2\% & 8.2   & 8     & 1     & 26    & 1     & 0.00\% & 0.00\% & \textbf{0.32} \\
          & 90    & 0.67\% & 15.1\% & 8     & 10    & 1     & 32    & 3     & 0.00\% & 0.00\% & \textbf{0.31} \\
    50,000 & 10    & 0.05\% & 0.5\% & 6     & 5     & 1     & 15    & 2     & 0.00\% & -0.01\% & \textbf{0.30} \\
          & 30    & 0.14\% & 4.0\% & 8.1   & 8     & 3     & 31    & 6     & 0.00\% & 0.00\% & \textbf{0.25} \\
          & 50    & 0.22\% & 7.3\% & 8.9   & 9     & 1     & 44    & 5     & 0.00\% & 0.00\% & \textbf{0.21} \\
          & 70    & 0.31\% & 10.0\% & 9     & 14    & 1     & 57    & 4     & 0.00\% & 0.00\% & \textbf{0.24} \\
          & 90    & 0.40\% & 11.0\% & 8.5   & 16    & 2     & 66    & 3     & 0.00\% & 0.00\% & \textbf{0.24} \\
    100,000 & 10    & 0.02\% & 0.4\% & 5.7   & 11    & 6     & 53    & 24    & 0.00\% & -0.03\% & \textbf{0.21} \\
          & 30    & 0.07\% & 2.7\% & 9.1   & 14    & 2     & 75    & 7     & 0.00\% & 0.00\% & \textbf{0.19} \\
          & 50    & 0.11\% & 4.7\% & 9.3   & 18    & 2     & 108   & 12    & 0.00\% & 0.00\% & \textbf{0.17} \\
          & 70    & 0.16\% & 6.0\% & 9     & 23    & 2     & 133   & 8     & 0.00\% & 0.00\% & \textbf{0.17} \\
          & 90    & 0.20\% & 7.3\% & 9     & 30    & 2     & 168   & 9     & 0.00\% & 0.00\% & \textbf{0.18} \\
    150,000 & 10    & 0.02\% & 0.1\% & 3.8   & 14    & 7     & 86    & 96    & 0.00\% & -0.09\% & \textbf{0.16} \\
          & 30    & 0.05\% & 1.9\% & 9     & 23    & 3     & 157   & 81    & 0.00\% & 0.00\% & \textbf{0.15} \\
          & 50    & 0.07\% & 3.6\% & 9.6   & 27    & 2     & 174   & 14    & 0.00\% & 0.00\% & \textbf{0.16} \\
          & 70    & 0.10\% & 5.0\% & 9.5   & 36    & 5     & 234   & 15    & 0.00\% & 0.00\% & \textbf{0.15} \\
          & 90    & 0.13\% & 5.3\% & 9.1   & 39    & 3     & 280   & 11    & 0.00\% & 0.00\% & \textbf{0.14} \\ \hline
    \end{tabular}%
  \caption{Average performance of AID for SVM against \textsf{libsvm} (Set 1 IBM data)} 
  \label{tab:svm_result_ibm}
\end{table}%

The result for Set 1 RCV data is presented in Table \ref{tab:svm_result_rcv}. AID is again faster than \textsf{libsvm} for all cases. The values of $\rho$ are much smaller than the values from Table \ref{tab:svm_result_ibm}. This can be explained by the smaller values of $r^T$ and $T$. Because AID converges faster for RCV data, it terminates early and takes much less time. Recall that larger $T$ and $r^T$ imply that more aggregated problems with larger sizes are additionally solved. Similarly to the result in Table \ref{tab:svm_result_ibm}, $\rho$ values tend to decrease in $n$ and $m$, where the trend is much clearer for RCV data. By checking $\Delta$ and $\Gamma$, we observe that the solution of AID and \textsf{libsvm} are equivalent. One interesting observation is that the trend of the number of iterations ($T$) is different from Table \ref{tab:svm_result_ibm}. For Set 1 RCV data, $T$ tends to increase in $n$ and decrease in $m$. This is exactly opposite from the result for Set 1 IBM data. This can be explained by very small values of $r^T$ compared to the values in Table \ref{tab:svm_result_ibm}.

% Table generated by Excel2LaTeX from sheet 'ibm libsvm'
\begin{table}[htbp]
  \centering
  \scriptsize
    \begin{tabular}{|rr|rrrrr|rr|rrr|}
    \hline
    \multicolumn{2}{|c|}{size} & \multicolumn{5}{c|}{AID}  & \multicolumn{2}{c|}{$\textsf{libsvm}$} & \multicolumn{3}{c|}{Comparison}\\ \hline
$n$     & $m$     & $r^0$    & $r^T$ & $T$ & $\mathcal{T}^{\mbox{\tiny{AID}}}$ & $\sigma(\mathcal{T}^{\mbox{\tiny{AID}}})$& $\mathcal{T}^{\mbox{\textsf{libsvm}}}$  & $\sigma(\mathcal{T}^{\mbox{\textsf{libsvm}}})$ & $\Gamma$ & $\Delta$ & $\rho$\\ \hline
    30,000 & 10    & 0.08\% & 0.2\% & 4.1   & 2     & 0     & 25    & 1     & 0.00\% & 0.00\% & \textbf{0.064} \\
          & 30    & 0.23\% & 0.7\% & 3.7   & 2     & 0     & 41    & 1     & 0.00\% & 0.00\% & \textbf{0.038} \\
          & 50    & 0.37\% & 1.1\% & 3.3   & 2     & 0     & 71    & 1     & 0.00\% & 0.00\% & \textbf{0.022} \\
          & 70    & 0.52\% & 1.6\% & 3.6   & 2     & 0     & 113   & 2     & -0.01\% & 0.00\% & \textbf{0.016} \\
          & 90    & 0.67\% & 2.0\% & 3.6   & 2     & 0     & 146   & 2     & 0.03\% & 0.00\% & \textbf{0.013} \\
    50,000 & 10    & 0.05\% & 0.2\% & 4.1   & 3     & 0     & 67    & 2     & -0.05\% & 0.00\% & \textbf{0.039} \\
          & 30    & 0.14\% & 0.5\% & 3.8   & 3     & 0     & 147   & 4     & 0.01\% & 0.00\% & \textbf{0.018} \\
          & 50    & 0.22\% & 0.9\% & 4     & 3     & 0     & 254   & 5     & -0.01\% & 0.00\% & \textbf{0.012} \\
          & 70    & 0.31\% & 1.2\% & 3.9   & 3     & 0     & 349   & 6     & 0.01\% & 0.00\% & \textbf{0.009} \\
          & 90    & 0.40\% & 1.6\% & 3.9   & 3     & 0     & 422   & 6     & -0.01\% & 0.00\% & \textbf{0.008} \\
    100,000 & 10    & 0.02\% & 0.1\% & 4.6   & 6     & 1     & 367   & 29    & 0.01\% & 0.00\% & \textbf{0.016} \\
          & 30    & 0.07\% & 0.4\% & 4.9   & 7     & 0     & 856   & 82    & 0.00\% & 0.00\% & \textbf{0.008} \\
          & 50    & 0.11\% & 0.7\% & 4.9   & 7     & 0     & 1,312  & 167   & 0.00\% & 0.00\% & \textbf{0.005} \\
          & 70    & 0.16\% & 1.0\% & 5     & 8     & 2     & 1,524  & 163   & 0.00\% & 0.00\% & \textbf{0.005} \\
          & 90    & 0.20\% & 1.3\% & 5     & 12    & 6     & 1,918  & 285   & 0.00\% & 0.00\% & \textbf{0.006} \\
    150,000 & 10    & 0.02\% & 0.1\% & 5.7   & 11    & 1     & 1,120  & 70    & 0.02\% & 0.00\% & \textbf{0.010} \\
          & 30    & 0.05\% & 0.4\% & 5.1   & 11    & 1     & 2,503  & 424   & 0.00\% & 0.00\% & \textbf{0.004} \\
          & 50    & 0.07\% & 0.6\% & 5.1   & 11    & 1     & 3,469  & 655   & 0.00\% & 0.00\% & \textbf{0.003} \\
          & 70    & 0.10\% & 0.8\% & 5.1   & 12    & 1     & 3,963  & 820   & 0.02\% & 0.00\% & \textbf{0.003} \\
          & 90    & 0.13\% & 1.1\% & 5.2   & 13    & 1     & 4,402  & 663   & 0.00\% & 0.00\% & \textbf{0.003} \\ \hline
    \end{tabular}%
  \caption{Average performance of AID for SVM against \textsf{libsvm} (Set 1 RCV data)} 
  \label{tab:svm_result_rcv}
\end{table}%

In order to visually compare the performances, in Figure \ref{fg_exp_svm}, we plot $\rho$ values and computation times of AID and \textsf{libsvm} for IBM and RCV data. Figures \ref{fg_exp_svm_ibm_time} and \ref{fg_exp_svm_rcv_time} assert that AID is scalable, while its relative performance keeps improving with respect to \textsf{libsvm} as shown in Figures \ref{fg_exp_svm_ibm_rho} and \ref{fg_exp_svm_rcv_rho}. For both data sets, the computation times of AID grow slower than \textsf{fn} and AID saves more computation time as $n$ and $m$ increase.

\begin{figure}[ht]
     \begin{center}
        \subfigure[Execution times (Set 1 IBM data)]{%
           \includegraphics[scale=0.3]{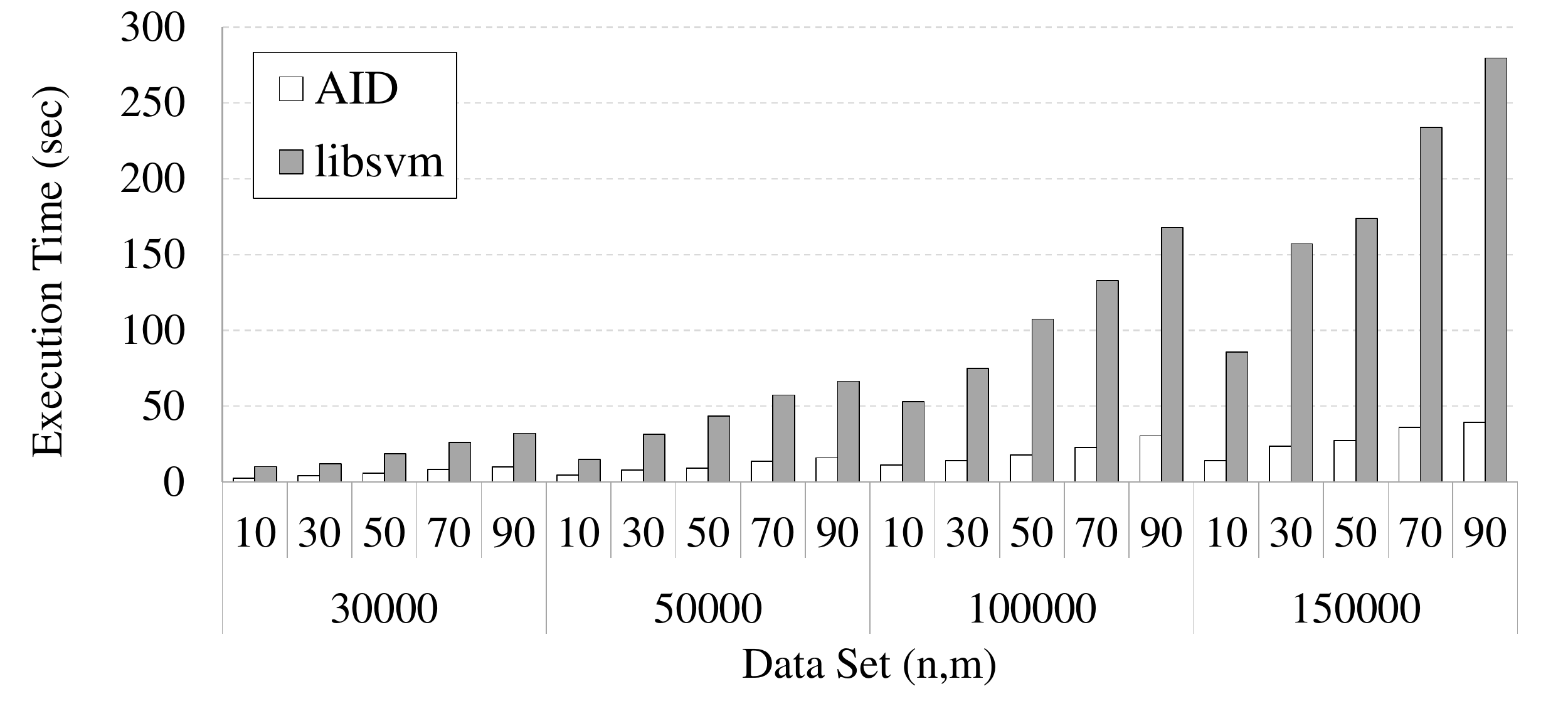} \label{fg_exp_svm_ibm_time}
        }\qquad
        \subfigure[$\rho$ (Set 1 IBM data)]{%
           \includegraphics[scale=0.3]{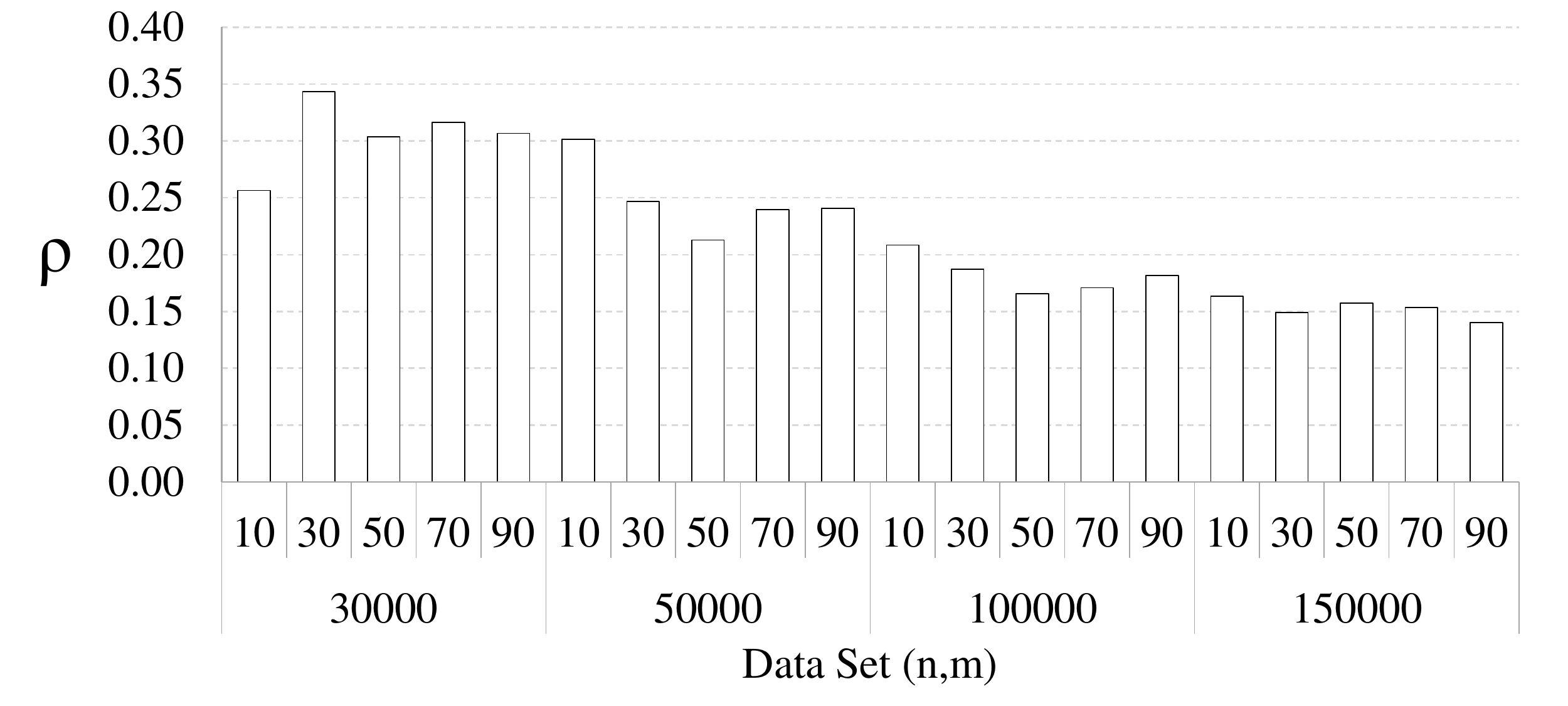} \label{fg_exp_svm_ibm_rho}
        }
        \subfigure[Execution times (Set 1 RCV data)]{%
           \includegraphics[scale=0.3]{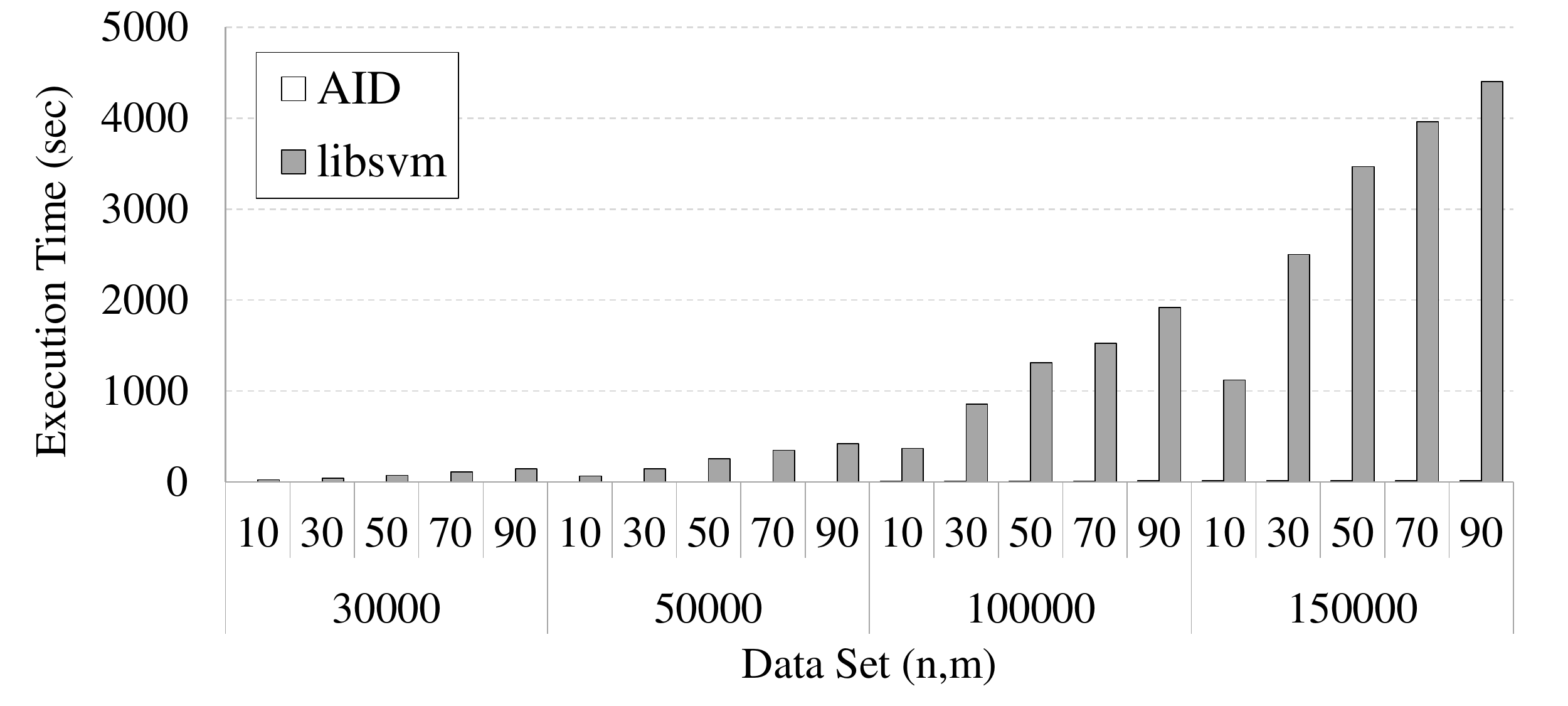} \label{fg_exp_svm_rcv_time}
        }\qquad
        \subfigure[$\rho$ (Set 1 RCV data)]{%
           \includegraphics[scale=0.3]{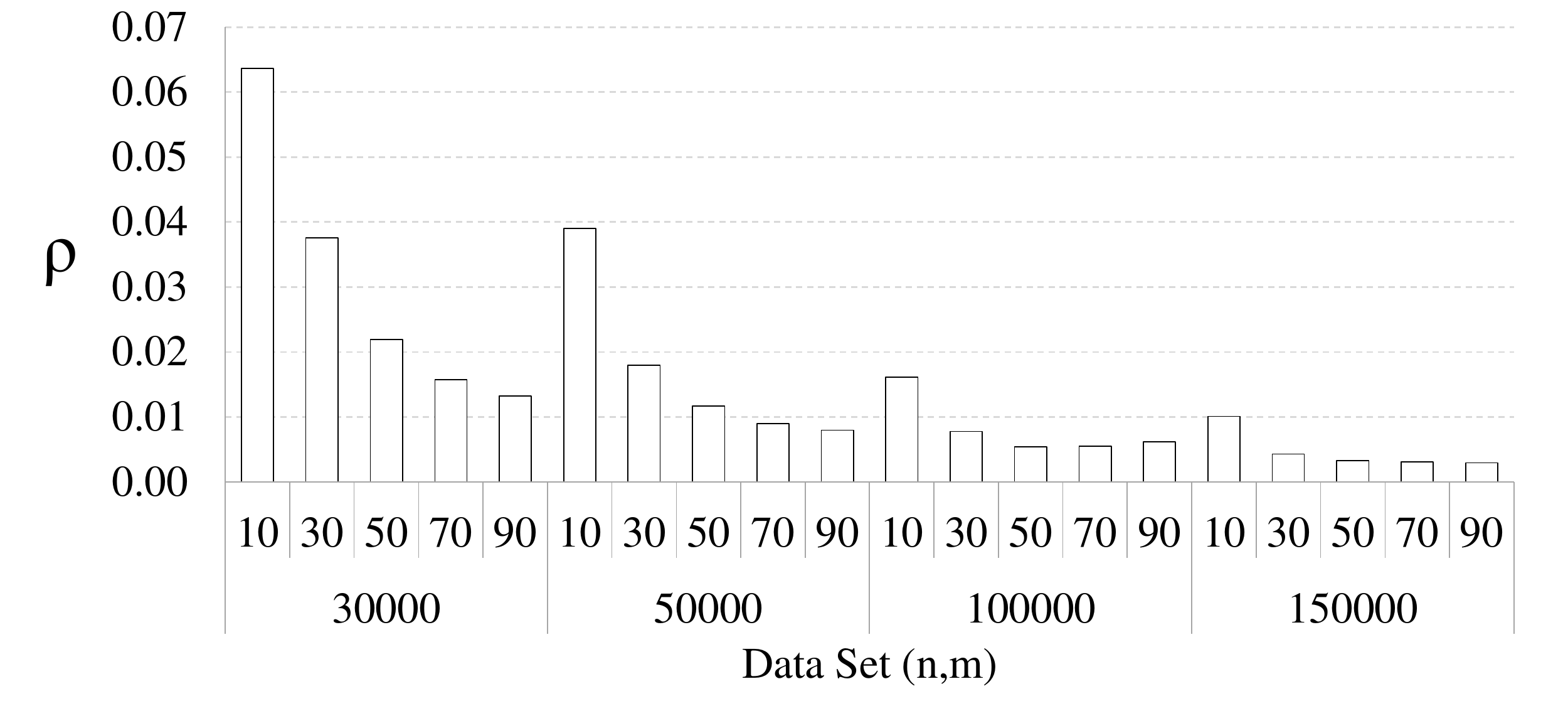} \label{fg_exp_svm_rcv_rho}
        }
    \end{center}
    \vspace{-0.5cm}
    \caption{Plots for performance of AID for SVM against \textsf{libsvm}}
    \label{fg_exp_svm}
\end{figure}

Although we present AID for SVM with kernels in Appendix \ref{section_appendix_B_aid_kernel}, we only show result for linear SVM in this experiment. For SVM with linear kernel, Liblinear \cite{liblinear} is known to be one of the fastest algorithms. Preliminary experiments showed that Set 1 instances are too small to obtain benefits from AID, and \textsf{liblinear} is faster for all cases. Hence, for comparison against \textsf{liblinear}, we consider Set 2 (larger instances sampled from IBM data). The result is shown in Table \ref{tab:svm_result_ibm_large}. Because \textsf{liblinear} is a faster solver, for some cases \textsf{liblinear} is faster than AID, especially when $n$ and $m$ are small. Among 20 cases ($n$-$m$ pairs), \textsf{liblinear} wins 45\% with 10.52 times faster than AID at maximum, and AID wins 55\% with 27.62 times faster than \textsf{liblinear}. However, the solution time of \textsf{liblinear} has very large variation. This is because \textsf{liblinear} struggles to terminate for some instances. On the other hand, AID has relatively small variations in solution time. Therefore, even though AID is not outperforming for all cases, we conclude AID is more stable and competitive. Note that objective function value difference $\Delta$ is large for some cases. With extremely large clusters (giving large weights in aggregated problems), we observe that \textsf{liblinear} does not give an accurate and stable result for aggregated problems of AID for Set 2.

% Table generated by Excel2LaTeX from sheet 'ibm libsvm'
\begin{table}[htbp]
  \centering
  \scriptsize
    \begin{tabular}{|rr|rrrrr|rr|rrr|}
    \hline
    \multicolumn{2}{|c|}{size} & \multicolumn{5}{c|}{AID}  & \multicolumn{2}{c|}{$\textsf{liblinear}$} & \multicolumn{3}{c|}{Comparison}\\ \hline
$n$     & $m$     & $r^0$    & $r^T$ & $T$ & $\mathcal{T}^{\mbox{\tiny{AID}}}$ & $\sigma(\mathcal{T}^{\mbox{\tiny{AID}}})$& $\mathcal{T}^{\mbox{\textsf{liblinear}}}$  & $\sigma(\mathcal{T}^{\mbox{\textsf{liblinear}}})$ & $\Gamma$ & $\Delta$ & $\rho$\\ \hline
    200,000 & 10    & 0.02\% & 0.1\% & 3.1   & 10    & 4     & 4     & 9     & 0.00\% & 0.12\% & 2.40 \\
          & 30    & 0.03\% & 0.8\% & 5.9   & 20    & 5     & 205   & 379   & -0.04\% & 3.56\% & \textbf{0.10} \\
          & 50    & 0.06\% & 3.0\% & 9.3   & 35    & 8     & 6     & 2     & 0.00\% & 0.23\% & 5.99 \\
          & 70    & 0.08\% & 3.8\% & 9.7   & 38    & 3     & 12    & 5     & 0.00\% & 0.00\% & 3.32 \\
          & 90    & 0.10\% & 4.3\% & 9.5   & 43    & 3     & 12    & 3     & 0.00\% & 0.00\% & 3.58 \\
    400,000 & 10    & 0.01\% & 0.0\% & 1.9   & 15    & 4     & 98    & 236   & 0.00\% & 0.79\% & \textbf{0.15} \\
          & 30    & 0.02\% & 0.3\% & 4.6   & 32    & 6     & 844   & 1,624  & 0.00\% & 14.04\% & \textbf{0.04} \\
          & 50    & 0.03\% & 1.1\% & 6.1   & 44    & 11    & 141   & 180   & 0.00\% & 14.85\% & \textbf{0.31} \\
          & 70    & 0.04\% & 2.7\% & 9.9   & 92    & 46    & 38    & 30    & 0.00\% & 0.01\% & 2.43 \\
          & 90    & 0.05\% & 3.0\% & 9.4   & 78    & 13    & 30    & 9     & 0.00\% & 0.28\% & 2.63 \\
    600,000 & 10    & 0.01\% & 0.0\% & 1.9   & 25    & 6     & 2     & 1     & 0.00\% & 0.16\% & 10.52 \\
          & 30    & 0.02\% & 0.3\% & 4.9   & 55    & 10    & 485   & 1,210  & 0.00\% & 1.51\% & \textbf{0.11} \\
          & 50    & 0.02\% & 1.2\% & 6.8   & 102   & 34    & 836   & 1,619  & 0.00\% & 0.11\% & \textbf{0.12} \\
          & 70    & 0.03\% & 1.9\% & 7.7   & 125   & 54    & 1,154  & 1,370  & -0.02\% & 7.86\% & \textbf{0.11} \\
          & 90    & 0.03\% & 2.3\% & 8.6   & 125   & 30    & 715   & 1,193  & -0.03\% & 0.44\% & \textbf{0.17} \\
    800,000 & 10    & 0.01\% & 0.0\% & 1.4   & 30    & 8     & 23    & 61    & 0.00\% & 0.42\% & 1.32 \\
          & 30    & 0.01\% & 0.4\% & 5.4   & 83    & 12    & 29    & 28    & 0.00\% & 3.01\% & 2.91 \\
          & 50    & 0.02\% & 1.1\% & 7     & 125   & 34    & 3,443  & 4,011  & 0.00\% & 3.16\% & \textbf{0.04} \\
          & 70    & 0.02\% & 1.6\% & 7.4   & 135   & 32    & 1,023  & 2,228  & 0.00\% & 0.31\% & \textbf{0.13} \\
          & 90    & 0.03\% & 1.9\% & 7.6   & 173   & 57    & 942   & 888   & 0.00\% & 5.93\% & \textbf{0.18} \\ \hline
    \end{tabular}%
  \caption{Average performance of AID for SVM against \textsf{liblinear} (Set 2 IBM data)} 
  \label{tab:svm_result_ibm_large}
\end{table}%

In Figure \ref{fg_exp_svm_ibm_large}, we plot computation times and $\rho$ values of AID and \textsf{liblinear} for Set 2. Because $\rho$ values do not scale well, we instead present $\log_{10} \rho$ in Figure \ref{fg_exp_svm_ibm_large_rho}. From Table \ref{tab:svm_result_ibm_large} and Figure \ref{fg_exp_svm_ibm_large} we observe that AID outperforms for larger instances. The number of negative $\log_{10} \rho$ values (implying AID is faster) tend to increase as $n$ and $m$ increase.

\begin{figure}[ht]
     \begin{center}
        \subfigure[Execution times]{%
           \includegraphics[scale=0.3]{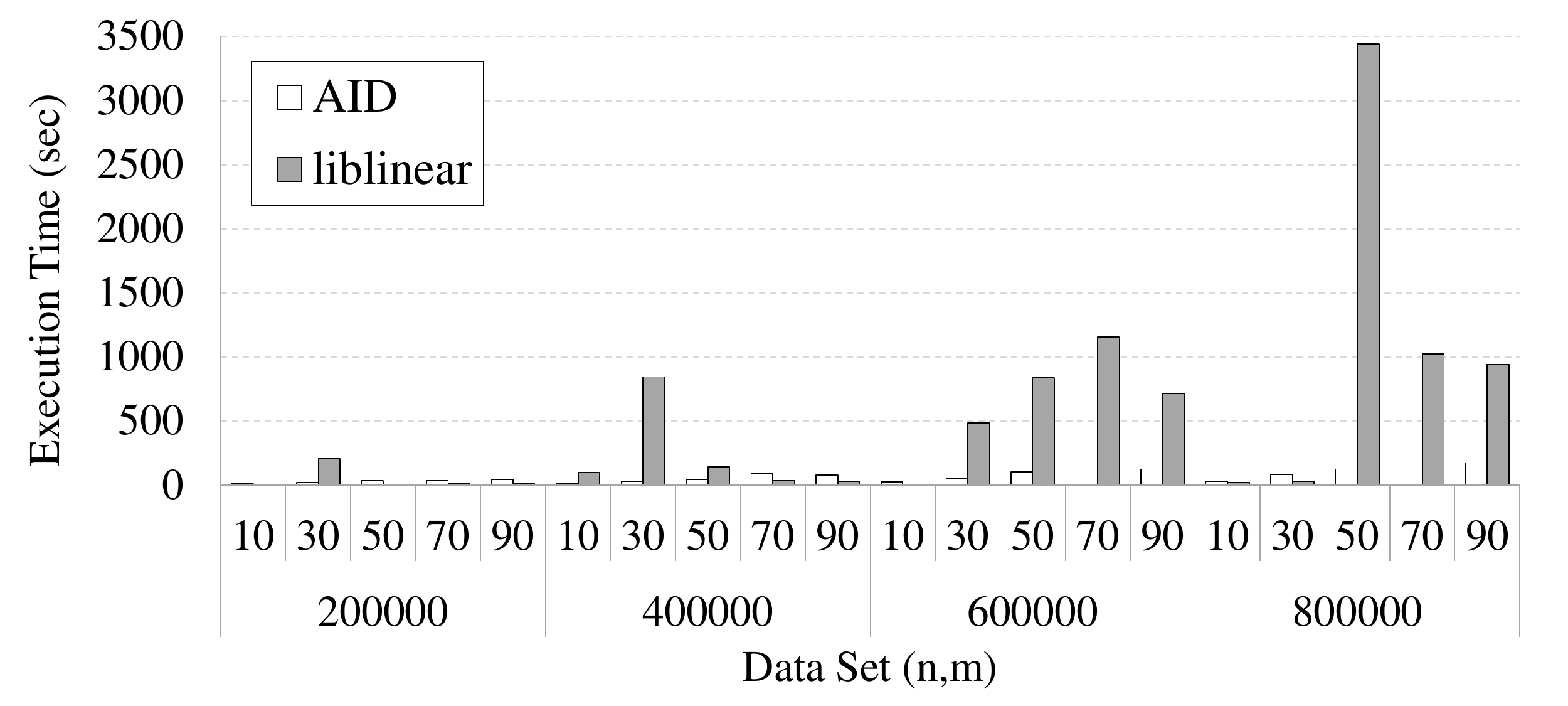} \label{fg_exp_svm_ibm_large_time}
        }\qquad
        \subfigure[$\rho$]{%
           \includegraphics[scale=0.3]{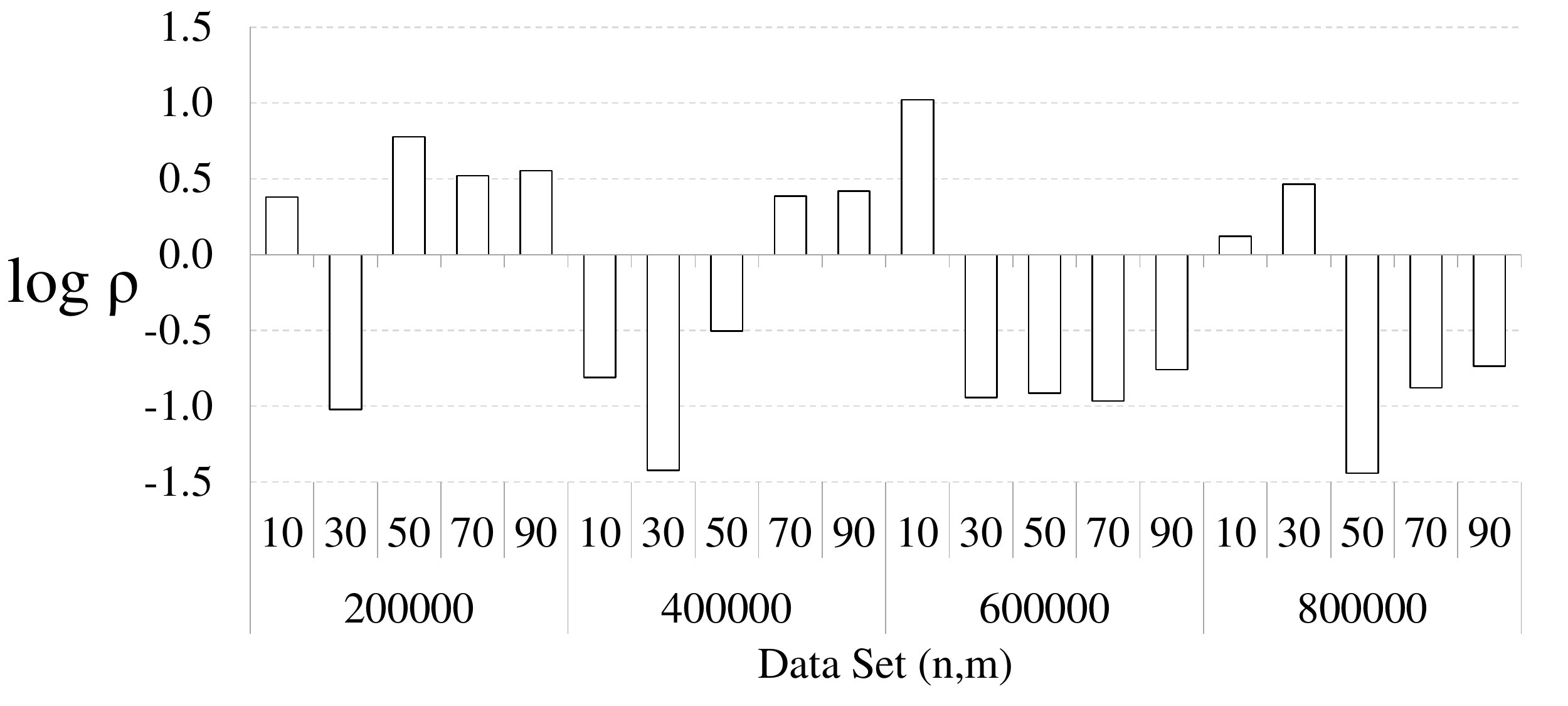} \label{fg_exp_svm_ibm_large_rho}
        }
    \end{center}
    \vspace{-0.5cm}
    \caption{Performance of AID for SVM against \textsf{liblinear} (Set 2 IBM data)}
    \label{fg_exp_svm_ibm_large}
\end{figure}

\subsection{Performance for \texorpdfstring{S{$^3$}VM}{S3VM}}
\label{section_exp_s3vm}

As pointed out in \cite{Chapelle-book}, it is difficult to solve MIQP model \eqref{def_SSSVM_alternative} optimally for large size data. In a pilot study, we observed that AID can optimally solve \eqref{def_SSSVM_alternative} for some data sets with hundreds of observations and several attributes. However, for the data sets in Table \ref{tab:s3vm_book_data}, AID was not able to terminate within a few hours. Also, no previous work in the literature provides computational result for the data sets by solving \eqref{def_SSSVM_alternative} directly. Therefore, in this experiment for S$^3$VM, we do not compare the execution times of AID and benchmark algorithms. Instead, we compare classification rates, which is the fraction of unlabeled observations that are correctly labeled by the algorithm. The comparison is only with algorithms for S$^3$VM from \cite{Chapelle-book, Li-Zhou:14}. See \cite{Chapelle-book} for comprehensive comparisons of other semi-supervised learning models. Because AID is not executed until optimality, we terminate after one and five iterations, which are denoted as AID1 and AID5, respectively.

In the computational experiment, we consider penalty parameters $(M_l,M_u) \in (5,1) \times \{10^0, 10^{-1}, 10^{-2},$ $ 10^{-3},10^{-4}\}$ and initial aggregation rate $r^0 \in \{0.01, 0.05 \}$. Also, we consider the following two techniques for unbalanced data.
\begin{enumerate}
\item \textit{balance constraint}: $\frac{\sum_{i \in I_u} (w x_i + b_i)}{|I_u|} = \frac{\sum_{i \in I_l} y_i}{|I_l|}$
\item \textit{balance cost}: Let $I_{l}^+$ and $I_l^-$ be the set of labeled observations with labels 1 and -1, respectively. In order to give larger weights for the minority class, we multiply $M_l$ by $\max \big\{1, \frac{|I_l^-|}{|I_l^+|} \big\}$ for $i \in I_l^-$ and $M_l$ by $\max \big\{1, \frac{|I_l^+|}{|I_l^-|} \big\}$ for $i \in I_l^+$.
\end{enumerate}

We first enumerate all possible combinations of the parameters and unbalanced data techniques and report the best classification rates of AID1 and AID5 with linear kernel. In Table \ref{tab:s3vm_exp_compare}, the results of AID1 and AID5 are compared against the benchmark algorithms. Recall that $l$ is the number of labeled observations. In Table \ref{tab:s3vm_exp_compare}, we present the average classification rates of AID1, AID5, and the benchmark algorithms in \cite{Chapelle-book} and \cite{Li-Zhou:14}. In the second row, TSVM with linear and RBF kernels are from \cite{Chapelle-book} and S4VM with linear and RBF kernels are from \cite{Li-Zhou:14}. The bold faced numbers represent that the corresponding algorithm gives the best classification rates. For data sets \textit{COIL2} and \textit{SecStr}, we only report the result for AID1 and AID5, as other algorithms do not provide results. From Table \ref{tab:s3vm_exp_compare}, we conclude that the classification rates of AID1 and AID5 are similar, while the execution times of AID1 are significantly smaller. Therefore, we conclude that AID1 is more efficient and we focus on AID1 for the remaining experiments.

% Table generated by Excel2LaTeX from sheet 'summary (2)'
\begin{table*}[htbp]
\scriptsize
  \centering
  \setlength{\tabcolsep}{3pt}
    \begin{tabular}{|r|rrrrrr|rrrrrr|}
    \hline
       & \multicolumn{6}{c|}{l=10}                      & \multicolumn{6}{c|}{l=100} \\ \hline
          & AID5  & AID1  & TSVM  & TSVM  & S4VM  & S4VM  & AID5  & AID1  & TSVM  & TSVM  & S4VM  & S4VM \\
    data  & linear & linear & linear & RBF   & linear & RBF   & linear & linear & linear & RBF   & linear & RBF \\ \hline
    Digit1 & \textbf{86.6\%} & \textbf{86.6\%} & 79.4\% & 82.2\% & 76.0\% & 63.6\% & 92.6\% & 92.0\% & 82.0\% & 93.9\% & 91.5\% & \textbf{94.9\%} \\
    USPS  & 80.1\% & \textbf{80.4\%} & 69.3\% & 74.8\% & 78.7\% & 80.1\% & 86.5\% & 87.1\% & 78.9\% & 90.2\% & 87.7\% & \textbf{91.0\%} \\
    BCI   & \textbf{52.9\%} & 52.2\% & 50.0\% & 50.9\% & 51.8\% & 51.3\% & \textbf{71.7\%} & 69.1\% & 57.3\% & 66.8\% & 70.5\% & 66.1\% \\
    g241c & \textbf{79.7\%} & \textbf{79.7\%} & 79.1\% & 75.3\% & 54.6\% & 52.8\% & \textbf{82.6\%} & \textbf{82.6\%} & 81.8\% & 81.5\% & 75.3\% & 74.8\% \\
    g241d & \textbf{59.5\%} & 49.9\% & 53.7\% & 49.2\% & 56.3\% & 52.7\% & 72.6\% & 59.3\% & 76.2\% & \textbf{77.6\%} & 72.2\% & 60.9\% \\
    Text  & 66.4\% & 66.4\% & \textbf{71.40\%} & 68.79\% & 52.1\% & 52.6\% & 75.77\% & 75.8\% & \textbf{77.69\%} & 75.48\% & 69.9\% & 54.1\% \\
    COIL2 & 90.9\% & 90.2\% & NA    & NA    & NA    & NA    & 87.0\% & 88.5\% & NA    & NA    & NA    & NA \\
    SecStr & 64.3\% & 62.4\% & NA    & NA    & NA    & NA    & 69.70\% & 65.8\% & NA    & NA    & NA    & NA \\ \hline
    \end{tabular}%
      \caption{Average classification rates of AID (best parameters) and the benchmark algorithms}
  \label{tab:s3vm_exp_compare}%
\end{table*}%

Next, we compare AID1 with TSVM with the RBF kernel from Table \ref{tab:s3vm_exp_compare}, as TSVM with the RBF kernel is the best among the benchmark algorithms from \cite{Chapelle-book} and \cite{Li-Zhou:14}. In Table \ref{tab:best_dab1_Vs_TSVM}, we observe that AID1 performs better than TSVM when $l = 10$, whereas the two algorithms tie when $l = 100$. Note that \textit{COIL2} and \textit{SecStr} are excluded from the comparison as results from the benchmark algorithms are not available.

\begin{table}[htbp]
\scriptsize
  \centering
    \begin{tabular}{|r|rr|rr|}
    \hline
      & \multicolumn{2}{c|}{l=10} & \multicolumn{2}{c|}{l=100} \\ \hline
          & AID1  & TSVM  & AID1  & TSVM \\
    data  & Linear & RBF   & Linear & RBF \\ \hline
    Digit1 & \textbf{86.6\%} & 82.2\% & 92.0\% & \textbf{93.9\%} \\
    USPS  & \textbf{80.4\%} & 74.8\% & 87.1\% & \textbf{90.2\%} \\
    BCI   & \textbf{52.2\%} & 50.9\% & \textbf{69.1\%} & 66.8\% \\
    g241c & \textbf{79.7\%} & 75.3\% & \textbf{82.6\%} & 81.5\% \\
    g241d & \textbf{49.9\%} & 49.2\% & 59.3\% & \textbf{77.6\%} \\
    Text  & 66.4\% & \textbf{68.8\%} & \textbf{75.8\%} & 75.5\% \\ \hline
    \# wins & 5     & 1     & 3     & 3 \\ \hline
    \end{tabular}%
  \caption{Average performances of AID (best parameters) and TSVM-RBF}
  \label{tab:best_dab1_Vs_TSVM}
\end{table}%

Recall that we report the best result by enumerating all parameters in Tables \ref{tab:s3vm_exp_compare} and \ref{tab:best_dab1_Vs_TSVM}. In the second experiment, we fix parameters $M_l  =5$, $M_u = 1$, $r^0 = 0.01$ for AID1. Since unbalanced data techniques significantly affect the result, we select balance cost for data sets \textit{USPS} and \textit{BCI} and balance constraint for data sets \textit{Digit1, g241c, g241d} and \textit{Text}. The result is compared against TSVM with the RBF kernel in Table \ref{tab:fixed_dab1_Vs_TSVM}. As we fix parameters, the classification rates of AID1 are worse than the rates in Table \ref{tab:best_dab1_Vs_TSVM}. However, AID1 is still competitive as AID1 and TSVM have the same number of wins.

% Table generated by Excel2LaTeX from sheet 'AID vs best'
\begin{table}[htbp]
\scriptsize
  \centering
    \begin{tabular}{|r|rr|rr|}
    \hline
      & \multicolumn{2}{c|}{l=10} & \multicolumn{2}{c|}{l=100} \\ \hline
          & AID1  & TSVM  & AID1  & TSVM \\
    data  & Linear & RBF   & Linear & RBF \\ \hline
    Digit1 & \textbf{83.4\%} & 82.2\% & 90.4\% & \textbf{93.9\%} \\
    USPS  & \textbf{80.4\%} & 74.8\% & 87.1\% & \textbf{90.2\%} \\
    BCI   & \textbf{51.2\%} & 50.9\% & \textbf{69.1\%} & 66.8\% \\
    g241c & 74.3\% & \textbf{75.3\%} & 76.0\% & \textbf{81.5\%} \\
    g241d & \textbf{49.9\%} & 49.2\% & 53.6\% & \textbf{77.6\%} \\
    Text  & 63.0\% & \textbf{68.8\%} & \textbf{75.8\%} & 75.5\% \\ \hline
    \# wins & 4     & 2     & 2     & 4  \\ \hline
    \end{tabular}%
  \caption{Average performances of AID (fixed parameters) and TSVM-RBF}
  \label{tab:fixed_dab1_Vs_TSVM}
\end{table}%

\section{Guidelines for Applying AID to Other Problems}
\label{section_discussion}

AID is designed to solve problems with a large number of entries that can be clustered well. From the experiments in Section \ref{section_compuatation}, we observe that AID is beneficial when data size is large. AID especially outperforms alternatives when the time complexity of the alternative algorithm is high. Recall that AID is applicable for problems following the form of \eqref{opt_problem_for_dab}. In this section, we discuss how AID can be applied for other optimization problems. We also discuss the behavior of AID as observed from additional computational experiments for LAD and SVM.

\subsection{Designing AID}
\subsubsection*{Aggregated data}
The main principle in generating aggregated data is to create each aggregated entry to represent the original entries in the corresponding cluster. However, the most important factor to consider when defining aggregated data is the interaction with the aggregated problem and optimality condition. The definition of aggregated entries plays a key role in deriving optimality and other important properties. In the proof of Proposition 1, $x$ is converted to $x^t$ in the second line of the equations. In the proof of Proposition 2, $x^{t-1}$ is converted to $x$ in the third line of the equations, which is subsequently converted into $x$ in the sixth line. Although any aggregated data definition representing the original data is acceptable, we found that the centroids work well for all of the three problems studied in this paper.

\subsubsection*{Aggregated problem}
The aggregated problem is usually the weighted version of the original problem, where weights are obtained as a function of cardinalities of the clusters. In this paper, we directly use the cardinalities as weights. We emphasize that weights are used to give priority to larger clusters (or associated aggregated entries), and defining the aggregated problem without weights is not recommended. Recall that defining the aggregated problem is closely related to the aggregated data definition and optimality condition. When the optimality condition is satisfied, the aggregated problem should give an optimal solution to the original problem. This can be proved by showing equivalent objective function values of aggregated and original problems at optimum. Hence, matching objective function values of the two problems should also be considered when designing the aggregated problem.

\subsubsection*{Optimality condition}
The optimality condition is the first step to develop when designing AID, because the optimality condition affects the aggregated data definition and problem. However, developing the optimality condition is not trivial. Properties at optimum for the given problem should be carefully considered. For example, the optimality condition for LAD is based on the fact that the residuals have the same sign if the observations are on the same side of the hyperplane. This allows us to separate the terms in the absolute value function when proving optimality. For SVM, we additionally use label information, because the errors also depend on the label. For S$^3$VM, we use even more information: the classification decision of unlabeled entries. Designing an optimality condition and proving optimality become non-trivial when constraints and variables are more complex. The proofs become more complex in the order of LAD, SVM, and S$^3$VM.

\subsection{Defining Initial Clusters}

\subsubsection*{Initial clustering algorithm}
From pilot computational experiments with various settings, we observed that the initial clustering accuracy is not the most important factor contributing to the performance of AID. This can be explained by the declustering procedure in early iterations. In the early iterations of AID, the number of clusters rapidly increases as most clusters violate the optimality condition. These new clusters are better than the k-means algorithm output using the same number of clusters in the sense that the declustered clusters are more likely to satisfy the optimality condition. Because the first few aggregated problems are usually small and can be solved quickly, the main concern in selecting an initial clustering algorithm is the computational time. Therefore, we recommend to use a very fast clustering algorithm to cluster the original entries approximately. For LAD and SVM, we use one iteration of k-means with two and one dimensional data, respectively. If one iteration of k-means is not precise enough then BIRCH \cite{Birch-Zhang:1996}, which has complexity of O($n$), may be considered.

\subsubsection*{Initial aggregation rate}
Avoiding a trivial aggregated problem is very important when deciding the initial aggregation rate. Depending on the optimization problem, this can restrict the minimum number of clusters or the number of aggregated entries. For example, we must have at least $m$ aggregate observations to have a nonzero-SSE model for LAD. Similar restrictions exist for SVM and S$^3$VM. 

We recommend to pick the smallest aggregation rate among all aggregation rates preventing trivial aggregated problems mentioned above because we can obtain better clusters (more likely to satisfy optimality condition) by solving smaller aggregated problems and by declustering. With the one iteration k-means setting, the number of clusters also affects the initial clustering time, as the time complexity is $O(|K^0|mn)$, where $|K^0|$ is the number of clusters at the beginning. In Figure \ref{fg_impact_of_r0}, we plot the solution time of AID for SVM for the IBM and RCV data sets with $n = 150,000$ and $m=10$. In each plot, the horizontal axis represents the initial aggregation rate $r^0$, and the left and right vertical axes are for the execution time and number of iterations, respectively. The stacked bars show the total time of AID, where each bar is split into the initialization time (clustering) and loop time (declustering and aggregated problem solving). The series of black circles represent the number of iterations of AID. As $r^0$ increases, we have a larger number of initial clusters. Hence, with the current initial clustering setting (one iteration of k-means), the initialization time (white bars) increases as $r^0$ increases. Although the number of iterations decreases in $r^0$, the loop time is larger when $r^0$ is large because the size of the aggregated problems is larger.

\begin{figure}[ht]
     \begin{center}
        \subfigure[IBM]{%
           \includegraphics[scale=0.3]{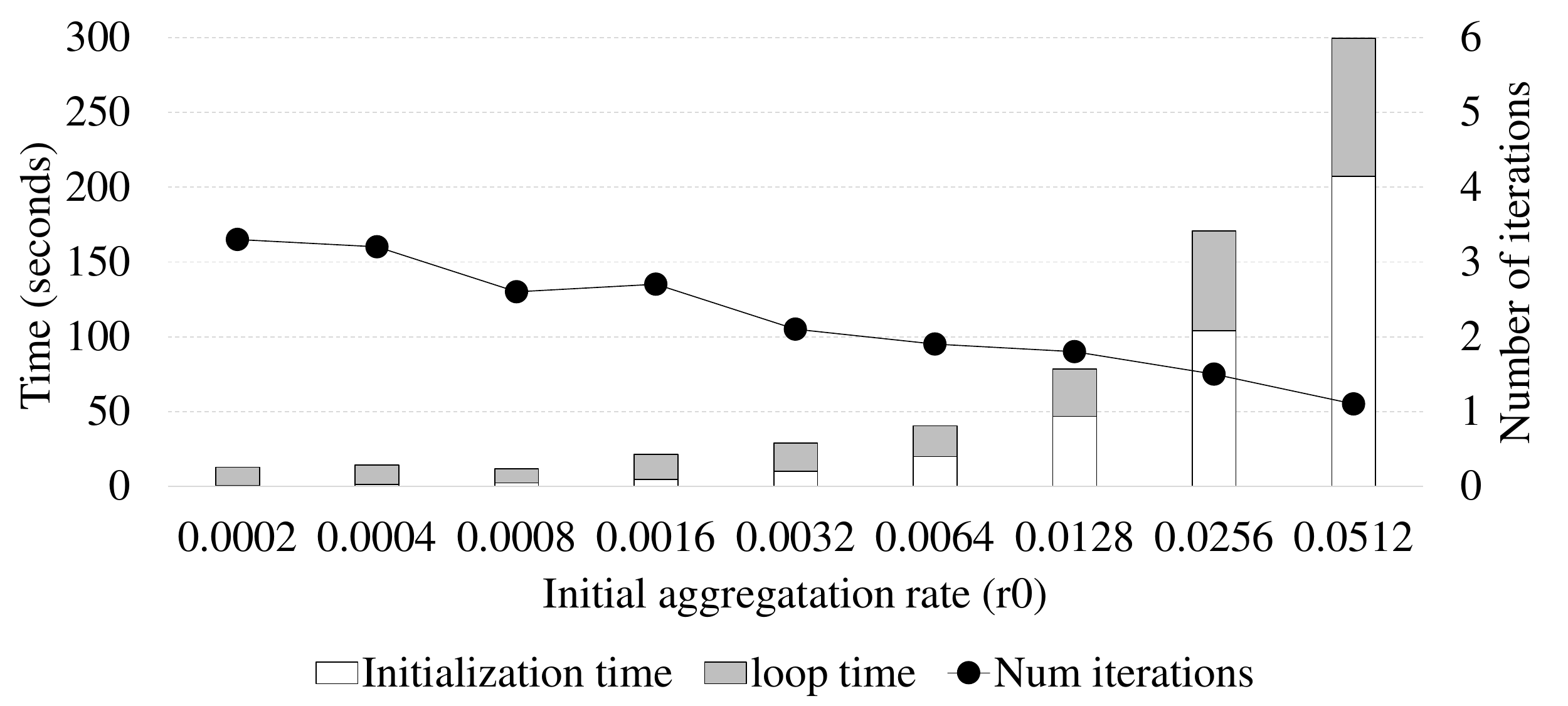} \label{fg_impact_of_r0_ibm}
        }\qquad
        \subfigure[RCV]{%
           \includegraphics[scale=0.3]{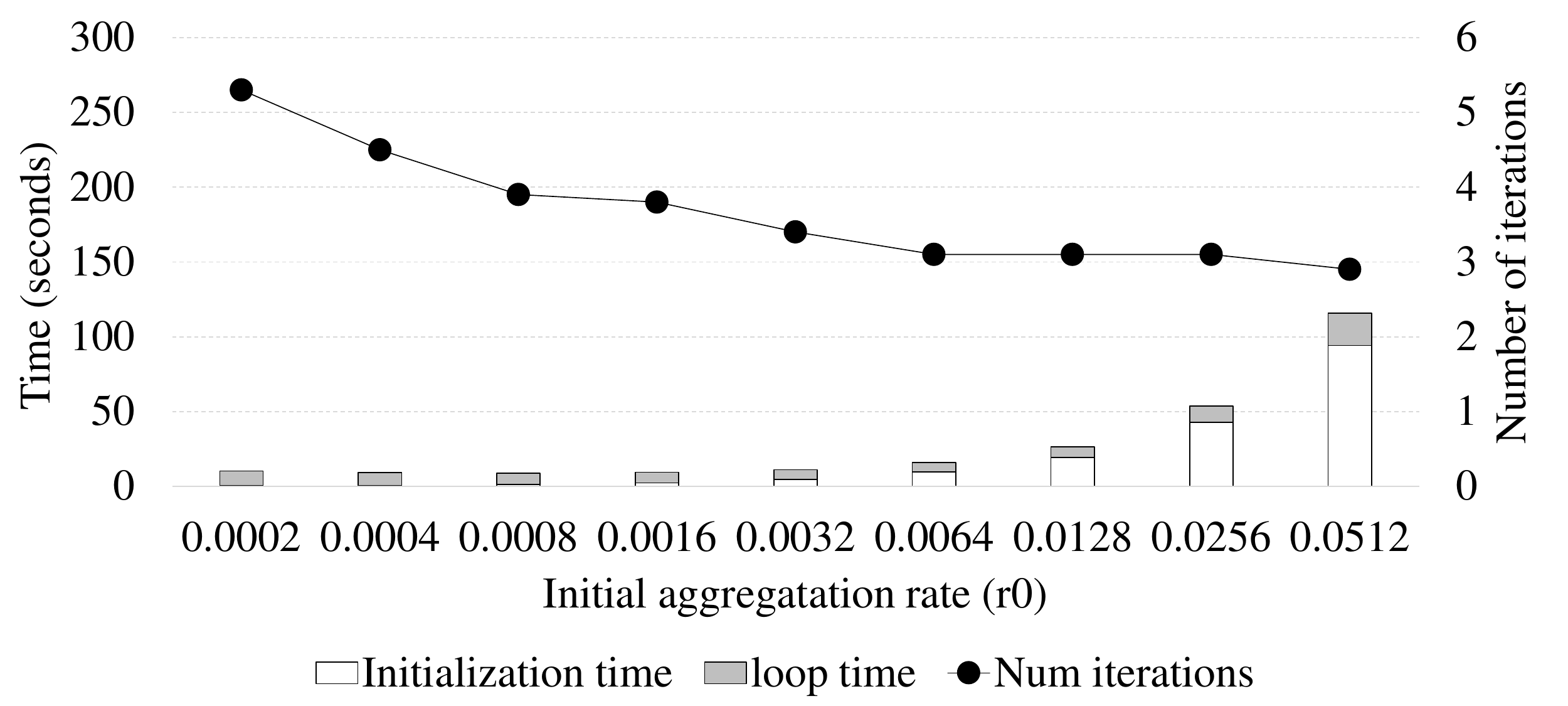} \label{fg_impact_of_r0_rcv}
        }
    \end{center}
    \vspace{-0.5cm}
    \caption{Impact of $r^0$ for Set 1 IBM and RCV data}
    \label{fg_impact_of_r0}
\end{figure}

\subsection{Aggregation Rates and Relative Location of Clusters}
Evgeniou and Pontil \cite{Evgeniou-Pontil:02} mention that clusters that are far from the SVM hyperplane tend to have large size, while the size of clusters that are near the SVM hyperplane is small. This also holds for AID for LAD, SVM, and S$^3$VM. We demonstrate this property with LAD. In order to check the relationship between the aggregation rate and the residual, we check 
\begin{enumerate}[noitemsep]
\item aggregation rates of entire clusters,
\item aggregation rates of clusters that are near the hyperplane (with residual less than median), and
\item aggregation rates of clusters that are far from the hyperplane (with residual greater than median).
\end{enumerate}
In Figure \ref{fg_exp_lad_matrix}, we plot the aggregation rate of entire clusters (series with $+$ markers), with residual less than the median (series with black circles), and with residuals greater than the median (series with empty circles). We plot the result for 9 instances in a 3 by 3 grid (3 values of $n$ and 3 values of $m$), where each sub-plot's horizontal and vertical axes are for iteration ($t$) and aggregation rates ($r^t$). We can observe that the aggregation rate of clusters that are near the hyperplane increases rapidly, while far clusters' aggregation rates stabilize after a few iterations. The aggregation rates of near clusters increase as $m$ increases.

\begin{figure}[ht]
\centering

    \includegraphics[scale=0.35]{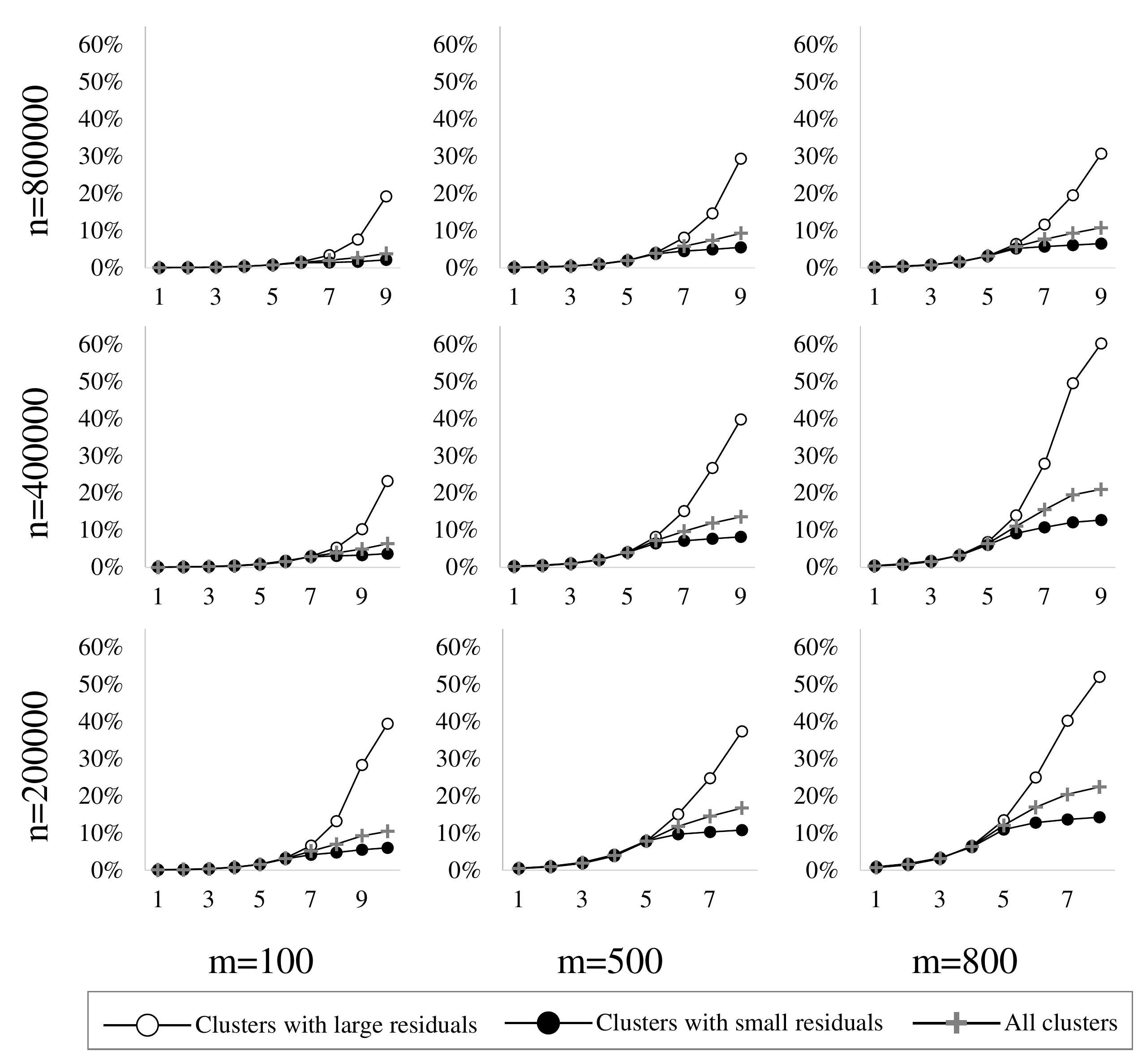} 
    \vspace{-0.2cm}
    \caption{Aggregation rate $r^t$ over iterations for LAD instances with various $n$ and $m$} \label{fg_exp_lad_matrix}
\end{figure}

\section{Conclusion}
We propose a clustering-based iterative algorithm and apply it to common machine learning problems such as LAD, SVM, and S$^3$VM. We show that the proposed algorithm AID monotonically converges to the global optimum (for LAD and SVM) and outperforms the current state-of-the-art algorithms when data size is large. The algorithm is most beneficial when the time complexity of the optimization problem is high, so that solving smaller problems many times is affordable.

\section{Acknowledgment}
We appreciate the referees for their helpful comments that strengthen the paper.

%\bibliography{aid}  % sigproc.bib is the name of the Bibliography in this case
 \bibliographystyle{abbrv}

\appendix

\section{Execution Time of AID for S3VM}

In this section, we present the execution times of AID1 and AID5 in Table \ref{tab:s3vm_time}. Although we used two techniques for unbalanced data, we did not find a big difference between the execution times of the two. However, we observed the execution times change in $M_l$. Hence, we present the average execution time of AID1 and AID5 over $M_l$ values. Since we had a time limit of 1,800 seconds, all values do not exceed the limit. This implies that AID5 may have terminated before 5 iterations due to the time limit. From the table, we observe that the first iteration of AID is very quick (except \textit{Text} data) while the remaining four iterations take longer time. The execution time of \textit{Text} is high due to large $m$, whereas the execution time of \textit{SecStr} is high due to large $n$.
% Table generated by Excel2LaTeX from sheet 'time table 151015'
\begin{table}[htbp]
\scriptsize
  \centering
    \begin{tabular}{|r|r|rrrrr|rrrrr|}
    \hline
          &       & \multicolumn{5}{c|}{l=10}              & \multicolumn{5}{c|}{l=100} \\ \hline
    algo  & name  & 1     & 0.1   & 0.01  & 0.001 & 0.0001 & 1     & 0.1   & 0.01  & 0.001 & 0.0001 \\ \hline
    AID5  & Digit1 & 776   & 633   & 545   & 15    & 12    & 935   & 962   & 770   & 9     & 28 \\
          & USPS  & 555   & 542   & 503   & 471   & 8     & 929   & 875   & 947   & 508   & 8 \\
          & BCI   & 393   & 394   & 456   & 246   & 2     & 605   & 776   & 753   & 20    & 2 \\
          & g241c & 511   & 578   & 492   & 526   & 25    & 916   & 979   & 946   & 864   & 42 \\
          & g241d & 357   & 372   & 351   & 281   & 50    & 877   & 987   & 894   & 813   & 87 \\
          & Text  & 649   & 628   & 342   & 178   & 163   & 892   & 753   & 244   & 181   & 167 \\
          & COIL2 & 882   & 889   & 895   & 870   & 862   & 1,042  & 1,085  & 1,090  & 1,043  & 959 \\
          & SecStr & 1,047  & 1,190  & 848   & 603   & 196   & 1,338  & 996   & 861   & 1,031  & 49 \\ \hline
    AID1  & Digit1 & 3     & 2     & 2     & 6     & 12    & 1     & 2     & 4     & 8     & 28 \\
          & USPS  & 2     & 2     & 2     & 2     & 3     & 3     & 3     & 3     & 4     & 5 \\
          & BCI   & 0     & 1     & 0     & 0     & 1     & 1     & 1     & 2     & 1     & 2 \\
          & g241c & 3     & 3     & 3     & 3     & 6     & 8     & 7     & 7     & 8     & 9 \\
          & g241d & 2     & 2     & 2     & 2     & 3     & 5     & 7     & 7     & 8     & 5 \\
          & Text  & 212   & 186   & 187   & 175   & 163   & 290   & 273   & 186   & 181   & 167 \\
          & COIL2 & 2     & 2     & 2     & 2     & 2     & 4     & 4     & 4     & 4     & 5 \\
          & SecStr & 9     & 10    & 9     & 9     & 10    & 14    & 14    & 14    & 12    & 14 \\ \hline
    \end{tabular}%
    \caption{Average execution times (in seconds) of AID for S$^3$VM} \label{tab:s3vm_time}%
\end{table}%

\section{AID for SVM with Direct Use of Kernels}
\label{section_appendix_B_aid_kernel}
In this section, we consider \eqref{def_SVM_kernel} with an arbitrary kernel function $\mathcal{K}(x_i,x_j) = \langle \phi(x_i), \phi(x_j) \rangle$, and we develop a procedure based on \eqref{def_SVM_dual_kernel}. The majority of the concepts remain the same from Section 3.2 and thus we only describe the differences.

For initial clustering, instead of clustering in Euclidean space, we need to cluster in the transformed kernel space. This can be done by the kernel k-means algorithm \cite{scholkopf1998nonlinear}. In the plain k-means setting, the optimization problem is
\begin{center}
$\min \sum_{k \in K^t} \sum_{i \in C_k^t} \| x_i - \mu_k \|^2$,
\end{center}
where $\mu_k = \frac{\sum_{i \in C_k^t} x_i}{|C_k^t|}$ is the center of cluster $k \in K^t$. In the kernel k-means, the optimization problem is
\begin{center}
$\min \sum_{k \in K^t} \sum_{i \in C_k^t} \| \phi(x_i) - \mu_k \|^2$,
\end{center}
where $\mu_k = \frac{\sum_{i \in C_k^t} \phi(x_i)}{|C_k^t|}$ is the center of cluster $k \in K^t$. Note that, despite we use $\phi$ in the formulation, the kernel k-means algorithm does not explicitly use $\phi$ and only use the corresponding kernel function $\mathcal{K}$. As an alternative, we may cluster in the Euclidean space, if we follow the initial clustering procedure from Section \ref{section_compuatation}. In the experiments, we sampled a small number of original entries to obtain the initial hyperplane. If we use a kernel when obtaining the initial hyperplane, the distance $d \in \mathbb{R}^n$ can approximate the relative locations in the kernel space.

Next, let us define aggregated data as
\begin{center}
$\phi(x_k^t) = \frac{\sum_{i \in C_k^t} \phi(x_i)}{|C_k^t|}$ and $y_k^t = \frac{\sum_{i \in C_k^t} y_i}{|C_k^t|} \in \{-1,1\}$. 
\end{center}
Note that we will not explicitly use $x_k^t$. Further, even though we define $\phi(x_k^t)$, it is not explicitly calculated in our algorithm. However, $\phi(x_k^t)$ will be used for all derivations. Aggregated problem $F^t$ is defined similarly. 
\begin{equation}
\label{def_SVM2_kernel}
\begin{split}
& F^t = \min_{w^t, b^t, \xi^t} \frac{1}{2} \|w^t \|^2 + M \sum_{k \in K^t} |C_k^t|  \xi_k^t \\
& \qquad \quad \mbox{s.t.} \quad y_k^t [w^t \phi(x_k^t) + b^t] \geq 1 - \xi_k^t, \xi_k^t \geq 0, k \in K^t,
\end{split}
\end{equation}
By replacing $x_i$ and $x_k^t$ with $\phi(x_i)$ and $\phi(x_k^t)$, respectively, in all of the derivations and definitions in Section 3.2, it is trivial to see that all of the findings hold.

We next show how to perform the required computations without access to $x_k^t$ and $\phi(x_k^t)$. To achieve this goal, we work with the dual formulation and kernel function.

Let us consider the dual formulation \eqref{def_SVM_dual_kernel}. The corresponding aggregated dual problem is written as 
\begin{equation}
\label{def_SVM_dual_kernel_agg}
\begin{split}
& \max \sum_{k \in K^t} \alpha_k^t - \frac{1}{2} \sum_{k,q \in K^t} \mathcal{K}(x_k^t,x_q^t) \alpha_k^t \alpha_q^t y_k y_q \\
& \mbox{s.t.} \quad \sum_{k \in K^t}  \alpha_k^t y_k = 0,\\
& \qquad 0 \leq \alpha_k^t \leq |C_k^t| M, k \in K^t.
\end{split}
\end{equation}

Recall that we are only given $\mathcal{K}(x_i,x_j)$ for $i,j \in I$. We derive
\begin{center}
$\mathcal{K}(x_k^t,x_q^t) = \langle \phi(x_k^t), \phi(x_q^t) \rangle = \Big\langle \frac{\sum_{i \in C_k^t} \phi(x_i)}{|C_k^t|}, \frac{\sum_{j \in C_q^t} \phi(x_j)}{|C_q^t|} \Big\rangle = \frac{\sum_{i \in C_k^t} \sum_{j \in C_q^t} \langle \phi(x_i), \phi(x_j) \rangle }{|C_k^t||C_q^t|} = \frac{\sum_{i \in C_k^t} \sum_{j \in C_q^t} \mathcal{K}(x_i,x_j)}{|C_k^t||C_q^t|}$,
\end{center}
where the first equality is by the definition of $\phi(x_k^t)$. Hence, $\mathcal{K}(x_k^t,x_q^t)$ can be expressed in terms of $\mathcal{K}(x_i,x_j)$'s and cluster information. Let $\bar{\alpha}^t$ be an optimal solution to \eqref{def_SVM_dual_kernel_agg}. Then, for original observation $i \in I$, we can define a classifier
\begin{center}
$f(x_i) = \sum_{k \in K^t} \bar{\alpha}_k^t y_k \langle \phi(x_k^t), \phi(x_i) \rangle + \bar{b}^t = \sum_{k \in K^t} \frac{\bar{\alpha}_k^t y_k \sum_{j \in C_k^t} \mathcal{K}(x_j,x_i)}{|C_k^t|} + \bar{b}^t$.
\end{center}
Note that $f(x_i)$ is equivalent to $w^t \phi(x_i) + b^t$. Therefore, the declustering procedure and optimality condition in Section \ref{subsection_SVM} can be used. For example, $C_k^t$ is divided into $C_{k+}^t = \{ i \in C_k^t| 1- y_i f(x_i) > 0\}$ and $C_{k-}^t = \{ i \in C_k^t| 1- y_i f(x_i) \leq 0\}$.

\section{Results for the Original IBM Dataset}

In this section, we present the result of solving the original large size IBM classification data set with 1.3 million observations and 342 attributes by AID with \textsf{libsvm} and \textsf{liblinear}. We use penalty $M = 0.1$ and $r^0 = \frac{1.1m}{n} \approx 0.057\%$ for both algorithms. In Table \ref{tab:svm_result_ibm_original}, the iteration information of AID with \textsf{libsvm} and \textsf{liblinear} are presented. The initial clustering times are 900 seconds and 150 seconds for \textsf{libsvm} and \textsf{liblinear}, respectively. The execution times are approximately 20 minutes for both algorithms and AID terminates after 8 or 9 iterations with the optimality gap less than 0.1\% and aggregation rates $r^T$ of 2.93\% and 4.58\%. In early iterations $(t \leq 5)$, the values of $r^t$ are almost doubled in each iteration, which implies that almost all clusters violate the optimality conditions and are declustered. In the later iterations, it takes longer time to solve the aggregated problem in each iteration and the optimality gap is rapidly decreasing in $t$, because the clusters are finer but the number of clusters is larger. For both algorithms, we observe that the training classification rates become stable after several iterations although the optimality gaps are still large.

\begin{table}[htbp]
\scriptsize
  \centering
    \begin{tabular}{|r|rrrrrrr|rrrrrrr|}
\hline
          & \multicolumn{7}{c|}{AID with \textsf{libsvm} }                            & \multicolumn{7}{c|}{AID with \textsf{liblinear}} \\ \hline
       $t$  & $r^t$ & $F_t$     & $E_{best}$ & Opt  & Iter & Cum & Train & $r^t$ & $F_t$     & $E_{best}$ & Opt  & Iter & Cum & Train \\
          &       &       &       & gap   & time  & time  & rate  &       &       &       & gap   & time  & time  & rate \\ \hline
    0     & 0.06\% & 19  & 139,971 & 1000\%+ & 14    & 14    & 66.3\% & 0.06\% & 19  & 140,349 & 1000\%+ & 13    & 13    & 66.3\% \\
    1     & 0.11\% & 248 & 139,971 & 1000\%+ & 22    & 35    & 64.1\% & 0.11\% & 256 & 140,349 & 1000\%+ & 30    & 43    & 64.0\% \\
    2     & 0.22\% & 687 & 89,104 & 1000\%+ & 23    & 58    & 75.4\% & 0.22\% & 687 & 93,743 & 1000\%+ & 30    & 74    & 74.3\% \\
    3     & 0.40\% & 941 & 27,713 & 1000\%+ & 27    & 85    & 91.8\% & 0.40\% & 934 & 44,699 & 1000\%+ & 38    & 112   & 85.6\% \\
    4     & 0.68\% & 1,079 & 17,366 & 1000\%+ & 38    & 123   & 96.9\% & 0.68\% & 1,043 & 20,870 & 1000\%+ & 52    & 163   & 95.1\% \\
    5     & 1.12\% & 1,135 & 8,038  & 608\% & 60    & 183   & 99.4\% & 1.14\% & 1,125 & 9,389  & 734\% & 53    & 216   & 99.4\% \\
    6     & 1.88\% & 1,182 & 1,616 & 37\%  & 114   & 296   & 99.6\% & 1.94\% & 1,179 & 2,138 & 81\%  & 49    & 265   & 99.6\% \\
    7     & 2.73\% & 1,203 & 1,227 & 2.00\% & 234   & 531   & 99.6\% & 2.99\% & 1,205 & 1,218 & 1.04\% & 71    & 336   & 99.5\% \\
    8     & 2.93\% & 1,208 & 1,208 & 0.04\% & 697   & 1,227  & 99.5\% & 3.29\% & 1,209 & 1,214 & 0.42\% & 362   & 698   & 99.5\% \\
    9     & \multicolumn{7}{c|}{}                                  & 4.58\% & 1,210 & 1,210 & 0.03\% & 518   & 1,216  & 99.5\% \\ \hline
    \end{tabular}%
      \caption{Number of iterations of AID for the original IBM data ($n=1.3$ million and $m=342$)}
\label{tab:svm_result_ibm_original}%
\end{table}%

\end{document}